%% file: main.tex
\documentclass{applemlr}

\input{apple_preamble}

\usepackage{array}
\usepackage{nicefrac}
\usepackage{algorithm}
\usepackage{algorithmic}
\usepackage{listings}
\usepackage{needspace}

\tcbuselibrary{listings,breakable,skins}

\newtcblisting{promptbox}[1][]{%
  enhanced,
  breakable,
  listing only,
  listing engine=listings,
  colback=gray!3,
  colframe=gray!35,
  boxrule=0.35pt,
  arc=0pt,
  left=1.2mm,
  right=1.2mm,
  top=1mm,
  bottom=1mm,
  listing options={
    basicstyle=\ttfamily\small,
    columns=fullflexible,
    breaklines=true,
    breakindent=0pt,
    keepspaces=true,
    showstringspaces=false
  },
  #1
}

\newcommand{\coderepo}{https://github.com/apple-aiml-research/ml-dlr-lock}

\theoremstyle{plain}
\newtheorem{theorem}{Theorem}[section]
\newtheorem{proposition}[theorem]{Proposition}

\theoremstyle{definition}
\newtheorem{definition}[theorem]{Definition}
\newtheorem{assumption}[theorem]{Assumption}
\newtheorem{remark}[theorem]{Remark}

\newcommand{\method}{\textsc{DLR-Lock}}
\newcommand{\arch}{DLR-Net}

\title{Locking Pretrained Weights via \\ Deep Low-Rank Residual Distillation}

\author[*]{Keitaro Sakamoto}
\author{Pierre Ablin}
\author{Federico Danieli}
\author{Marco Cuturi}

\affiliation{Apple}

\abstract{
The quality of open-weight language models has dramatically improved in recent years.
Sharing weights greatly facilitates model adoption by enabling their use across diverse hardware and software platforms.
They also allow for more open research and testing, to the extent that users can use them as checkpoints, fine-tune them according to their needs, and potentially redistribute them.
In some cases, however, concerns on modifying these weights towards unauthorized uses may outweigh the pros of giving users such freedom.
Defending against such adaptation is non-trivial: since an adaptive attacker can observe all weights and architectures by definition, they can reverse simple structural defenses, and use optimization to defeat the simplest locking mechanisms.
In this work, we exploit the inference--training asymmetry of automatic differentiation as a novel defense axis.
We propose \method{}, a method where the purveyor of the model purposely replaces each pretrained MLP in their model with a deep low-rank residual network (\arch{}) of comparable parameter count, forcing activation memory that grows linearly with depth during backpropagation.
\arch{}s are efficiently trained via module-wise regression.
We show that, beyond this memory overhead, \method{} results in architectural mismatches that complicate the optimization landscape of standard fine-tuning, and a backward pass that incurs disproportionately more overhead than the forward pass.
Our defense succeeds in withstanding adaptive attackers with full knowledge of the defense strategy while preserving the original model's capabilities.
Experiments on LLMs validate these claims.
}

\metadata[Correspondence]{\sffamily KS: \url{sakakei-1999@g.ecc.u-tokyo.ac.jp}; PA: \url{p_ablin@apple.com}; FD: \url{f_danieli@apple.com}; MC: \url{m_cuturi@apple.com}}
\metadata[Code]{\sffamily\url{\coderepo}}
\metadata[Note]{\sffamily $^{*}$Work done while KS was an intern at Apple.}

\begin{document}

\maketitle

\section{Introduction}
\label{sec:intro}

Pretrained large language models represent a substantial investment of compute, data curation, and engineering effort.
When released as open weights, the model provider loses control over how these weights are used: any party can adapt the model to new domains or remove usage restrictions, often at a fraction of the original training cost~\citep{kapoor2024societal,bommasani2024considerations,casper2026open}.
Open-weight licenses provide legal protection against such misuse, but no technical enforcement.
These concerns motivate a fundamental scientific question: \emph{can a neural network be locked, so that adaptation is inherently more expensive than adapting the original model, while retaining its inference time performance?}
This pressing question is far from straightforward:
simple structural defenses that exploit weight symmetries can be trivially reverted by an adaptive attacker who inspects the model.
One might hope to use optimization to embed the defense more deeply into the loss landscape, but this strategy is inherently self-defeating: the defender must use optimization to find a solution that resists optimization, yet the same gradient information available to the defender is equally available to the attacker.
Recent model locking methods~\citep{rosati2025locking,wang2026towards,rosati2026limits} face this fundamental difficulty.

\begin{figure}[t]
  \centering
  \includegraphics[width=0.95\linewidth]{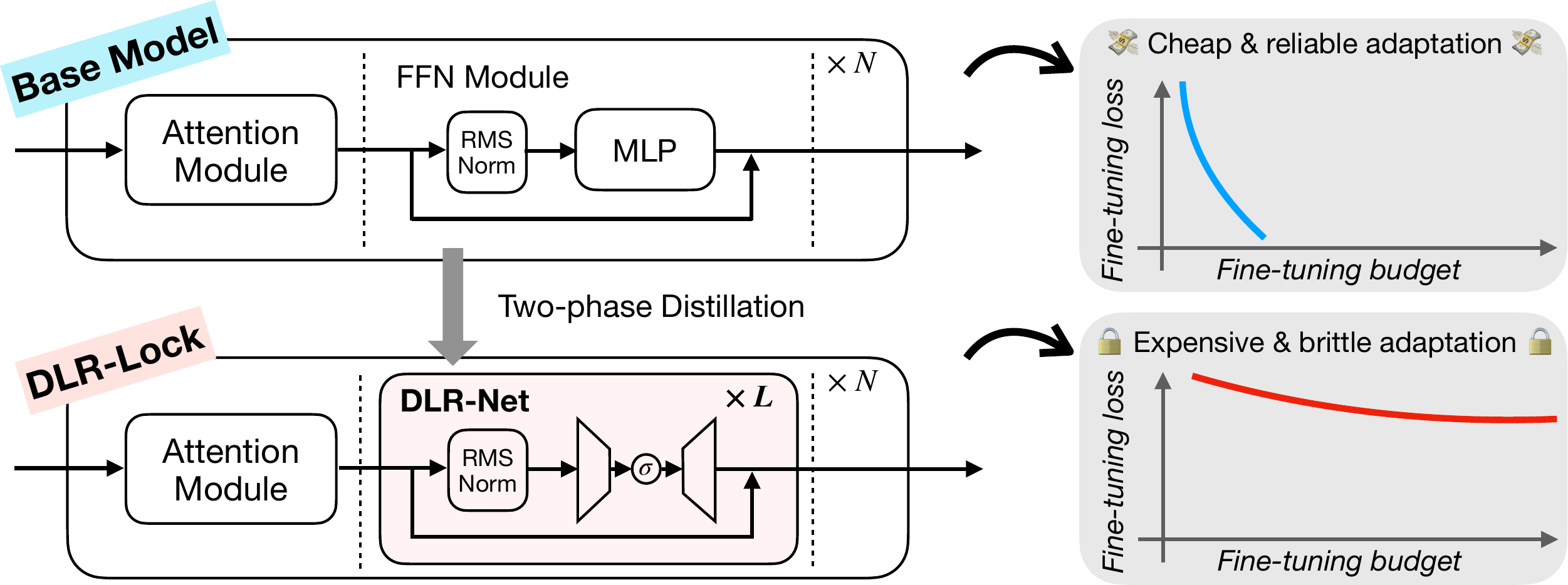}
  \caption{Overview of \method{}.
  Each MLP in the original model (top) is replaced with a functionally equivalent deep low-rank residual network (\arch{}, bottom) via a two-phase distillation.
  The replacement introduces three locking mechanisms: (1)~activation memory growth, (2)~optimization instability from architectural mismatch, and (3)~disproportionate backward-pass overhead.
  }
  \label{fig:architecture}
\end{figure}

In this work, we exploit a different defense axis: the inference--training asymmetry of automatic differentiation~\citep{griewank2008evaluating}.
Inference requires only a forward pass through the model, discarding all intermediate activations; training, by contrast, must store those intermediates for backpropagation.
To exploit this weakness, we propose to replace each pretrained MLP in a large network with a deep low-rank residual network (\arch{}) that matches the original input--output mapping via distillation (vector regression) with comparable parameter count.
The replacement introduces three locking mechanisms:
(1)~activation memory that grows linearly with depth;
(2)~architectural mismatch between the deep residual networks and the original MLPs destabilizes standard fine-tuning recipes; and
(3)~the backward pass is disproportionately slowed compared to the forward pass due to heavier memory access and more complex computation.
The locked model thus preserves inference quality while imposing substantial additional cost on any party attempting to adapt it.
Our locking mechanism is \emph{fine-tuning-domain agnostic}: it raises the cost of fine-tuning on any target domain, not only on potentially harmful ones \citep{huang2024vaccine,tamirisa2025tamperresistant}.
On the one hand, users fine-tuning for inoffensive applications pay a higher cost.
On the other hand, no malicious fine-tuning dataset avoids that cost.
Open weights confer portability across hardware and software platforms, inspection and auditing, reuse as checkpoints, and fine-tuning; \method{} constrains only the last.
Figure~\ref{fig:architecture} illustrates the overall idea; related work is discussed in Appendix~\ref{app:related}. Our contributions are:

\begin{itemize}[itemsep=1pt,leftmargin=*]
\item We formalize the problem of \emph{model locking} as adaptation cost amplification and establish an adaptive threat model grounded in Kerckhoffs's principle, which requires that security rely on structural properties rather than secrecy of the defense mechanism.
\item We categorize existing locking approaches as symmetry-based, barrier-based, and penalty-based, and identify inherent limitations of each through analysis and controlled experiments.
\item We show that the inference--training asymmetry of automatic differentiation provides a novel, optimizer-agnostic defense axis, and propose \method{}, a deep low-rank residual locking mechanism with a two-phase distillation pipeline that locks modules with limited quality loss, a residual gap that narrows as the distillation budget grows.
\item We validate the approach on Qwen3~\citep{yang2025qwen3} and on models we pretrain ourselves, at three scales and across several fine-tuning domains, showing that the locked models preserve the original capabilities, and that a custom kernel brings inference latency close to the baseline, while fine-tuning costs more: beyond the increased activation memory expected by design, architectural mismatch destabilizes standard recipes, and the backward pass is disproportionately slowed.
\item We evaluate adaptive attacks that exploit knowledge of the defense, including parameter freezing, low-rank adaptation, and reverse-distilling the locked modules back into ordinary MLPs, and find that none recovers cheap adaptation.
\end{itemize}

\section{Model Locking: Formulation and Challenges}
\label{sec:prelim}

Modern large language models are built from a stack of $N$ transformer
blocks~\citep{vaswani2017attention}.  Each block consists of a multi-head
self-attention module $\mathrm{Attn}(\cdot)$ followed by a feed-forward
network (FFN) module, with a residual connection and
RMSNorm~\citep{zhang2019root} applied before each module, following
the pre-norm convention:
\begin{align}
  \mZ &= \mX + \mathrm{Attn}\bigl(\mathrm{RMSNorm}(\mX)\bigr), \quad
  \mX' = \mZ + \mathrm{FFN}\bigl(\mathrm{RMSNorm}(\mZ)\bigr).
\end{align}
Here $\mX \in \R^{n \times d}$ is a sequence of $n$ token representations of
dimension~$d$, $\mZ$ and $\mX'$ are intermediate and output sequences of the
same shape, and both $\mathrm{FFN}$ and $\mathrm{RMSNorm}$ act row-wise,
i.e., independently on each token.  In recent architectures such as
LLaMA~\citep{grattafiori2024llama} and Qwen~\citep{yang2025qwen3}, the FFN is
a SwiGLU MLP~\citep{shazeer2020glu}:
\[
  \mathrm{FFN}(\vx) = \mW_{\mathrm{down}} \bigl((\mW_{\mathrm{up}} \vx) \odot
  \sigma(\mW_{\mathrm{gate}} \vx)\bigr),
\]
where $\odot$ denotes element-wise multiplication,
$\sigma(z) = z \cdot \mathrm{sigmoid}(z)$ is the SiLU
activation~\citep{elfwing2018sigmoid},
$\mW_{\mathrm{gate}}, \mW_{\mathrm{up}} \in \R^{d_{\mathrm{ff}} \times d}$,
and $\mW_{\mathrm{down}} \in \R^{d \times d_{\mathrm{ff}}}$.
Our method replaces each FFN residual sub-block with a functionally equivalent architecture, leaving all other components unchanged.

\subsection{Problem formulation}
\label{sec:formulation}

The term \emph{model locking} was introduced
by~\citet{rosati2025locking} to describe making pretrained weights resistant to
fine-tuning; a related concept of \emph{non-fine-tunability} was independently
proposed by~\citet{wang2026towards}.  We formalize this setting below.

\begin{definition}[Locking Mechanism]
\label{def:defense}
Let $f(\cdot;\theta_0)$ be a model obtained by standard pretraining.  A
\emph{locking mechanism} $\gM$ with private information~$s$ produces a
locked model $g(\cdot;\theta') = \gM(f, \theta_0, s)$ that is expected to be harder
to fine-tune than~$f(\cdot; \theta_0)$.
\end{definition}

\begin{assumption}[Attacker Capabilities]
\label{asm:kerckhoffs}
The attacker has: (i)~access to the published weights~$\theta'$, but not to $s$;
(ii)~knowledge of~$\gM$;
(iii)~a compute budget $C_{\mathrm{attack}} \ll \min\{C_{\mathrm{pretrain}}, C_{\mathrm{distill}}\}$, where $C_{\mathrm{distill}}$ is the cost of reconstructing a comparable unlocked model.
\end{assumption}

\noindent Condition~(ii) follows Kerckhoffs's
principle~\citep{kerckhoffs1883cryptographie}: the defense
must remain effective even when the attacker knows exactly how it was
constructed. The defender's advantage comes
from structural properties of~$g$, not secrecy of~$\gM$.
Condition~(iii) bounds what any locking mechanism can achieve: the attacker can pretrain an equivalent model outright at $C_{\mathrm{pretrain}}$, or distill one from~$g$'s logits at $C_{\mathrm{distill}}$, which needs only outputs from~$g$.
The two are of the same order~\citep{busbridge2025distillation}, and no defense can withstand an attacker at that budget.

The attacker can employ various strategies to
adapt the locked model~$g$ to a new, potentially harmful, task~$\gT$ under its budget constraint: fine-tuning the
entire model, distilling individual modules, applying analytical transforms
to preserve functional equivalence, or any combination thereof.  Let~$\mathfrak{A}$ denote the set of all feasible attack
strategies available to the attacker.  Each strategy $\gA \in \mathfrak{A}$
incurs a cost $C(\gA)$ and yields a model $\gA(g)$ with performance
$\mathrm{perf}_{\gT}(\gA(g))$ on task~$\gT$.  The \emph{adaptation cost} for
reaching a target performance level~$p$ is the minimum cost over all strategies
that achieve it:
\begin{equation}
\label{eq:adapt_cost}
  C_{\gT}(g \to p) \;=\; \min_{\gA \in \mathfrak{A}}
  \; C(\gA) \quad \text{s.t.} \quad
  \mathrm{perf}_{\gT}\bigl(\gA(g)\bigr) \ge p.
\end{equation}
Here $\gA$ is a procedure the attacker can actually execute, not an oracle that supplies a good configuration: if reaching~$p$ requires a search over batch size, sequence length, or learning rate, that search belongs to~$\gA$ and its trials enter $C(\gA)$.
We can now state the two requirements for model locking.

\begin{definition}[Model Locking]
\label{def:locking}
A defense mechanism $\gM$ achieves \emph{model locking} with locking
factor $\kappa > 1$ and tolerance $\varepsilon > 0$ if the published model
$g = \gM(f, \theta_0, s)$ satisfies:
\begin{enumerate}
\item[\emph{(i)}] \textbf{Utility preservation.}
  $\gL(g) < \gL(f) + \varepsilon$ for every downstream task loss $\gL$.
\item[\emph{(ii)}] \textbf{Adaptation cost amplification.}
  For any downstream task~$\gT$ and target performance level~$p$,
  \[
    C_{\gT}(g \to p) \;\ge\; \kappa \cdot C_{\gT}(f \to p),
  \]
  where $C_{\gT}(\cdot \to p)$ is measured in GPU-hours, what the attacker actually pays.
\end{enumerate}
\end{definition}

Condition~\emph{(i)} requires that the locked model retain general capabilities, not merely specific benchmark scores.
Definition~\ref{def:locking} places no constraint on the locking cost $C_{\mathrm{lock}}$, since the defender can afford a one-time cost that is amortized over all deployments.

\begin{remark}[Amortization across adaptations]
\label{rem:amortization}
The factor $\kappa$ compares single adaptations from the published checkpoint.
In practice an attacker may adapt to many tasks, or several may cooperate, with one recovering an unlocked model and distributing it to the rest.
Such a one-off recovery is then amortized rather than paid per task, a route we take up in Section~\ref{sec:reverse_distill}.
\end{remark}

\subsection{Why model locking is fundamentally difficult}
\label{sec:difficult}
\begin{figure}[t]
  \centering
  \includegraphics[width=\linewidth]{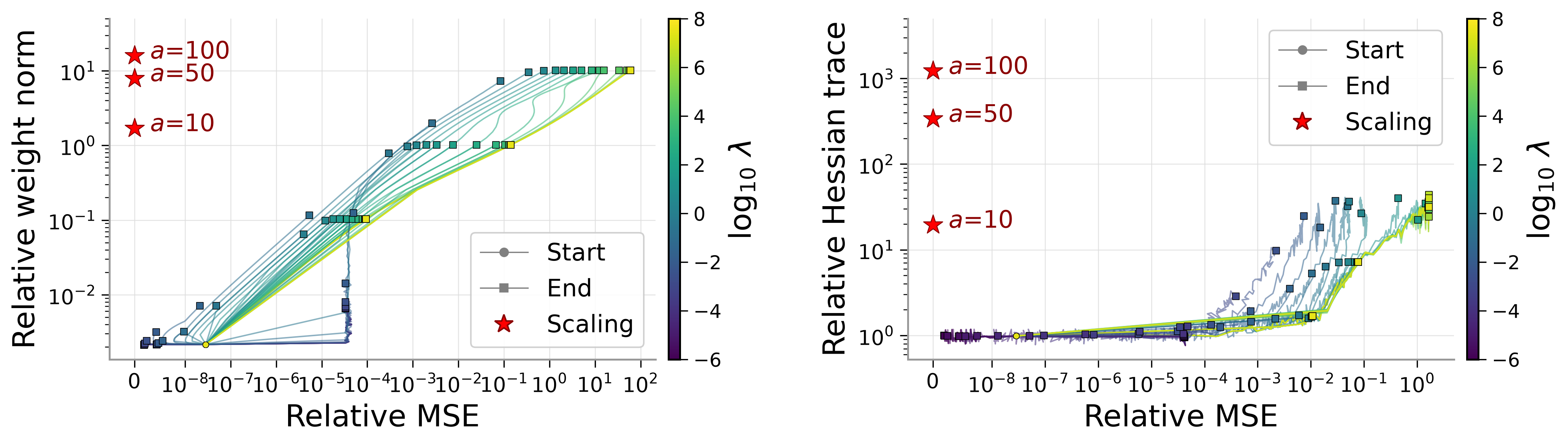}
  \caption{Locking solutions exist but gradient descent cannot reach them.
  We optimize $\min_{\Delta\theta} \mathrm{MSE} - \lambda\,\Omega$ on a two-layer ReLU MLP trained on MNIST, sweeping over 29~values of $\lambda$ and three learning rates with SGD and gradient norm clipping.
  Each curve is one trajectory; color encodes $\log_{10}\lambda$.
  Red stars: weight-symmetry solutions $(a^{-1}\mW_1,\, a\mW_2)$.
  \textbf{Left}: $\Omega = \|\Delta\theta\|^2$; $y$-axis is $\|\Delta\theta\|/\|\theta_0\|$.
  \textbf{Right}: $\Omega = \mathrm{Tr}(H_{\mathrm{CE}})$; $y$-axis is relative Hessian trace $\exp\,\mathbb{E}_x[\log\,\mathrm{Tr}(H_i(\theta))/\mathrm{Tr}(H_i(\theta_0))]$ ($H_i$: per-sample CE Hessian).
  $x$-axis (both): relative MSE, $\exp\,\mathbb{E}_x[\log\,\|f(\theta,x){-}f(\theta_0,x)\|^2/\|f(\theta_0,x)\|^2]$.
  }
  \label{fig:specdef}
\end{figure}

Before introducing our method, let us consider what approaches to model locking are available and why they face fundamental difficulties.
We provide a detailed comparison with existing methods in Appendix~\ref{app:related}, in light of the discussion below.
A natural starting point for model locking is to exploit \emph{weight symmetries}: transformations of the parameters that preserve the network's input--output mapping while potentially disrupting optimization~\citep{godfrey2022symmetries,zhao2026symmetry}.

\paragraph{Example 1: activation homogeneity.}
The ReLU activation is positively homogeneous, satisfying $\mathrm{ReLU}(a\vz) = a\,\mathrm{ReLU}(\vz)$ for any $a > 0$.
For any network layer of the form $\mW_2\,\mathrm{ReLU}(\mW_1 \vx)$, this yields the identity $(a\mW_2)\,\mathrm{ReLU}(a^{-1}\mW_1 \vx)$.
The reparameterization preserves the function exactly, but creates an imbalance:
the output becomes highly sensitive to small perturbations of
$a^{-1}\mW_1$, since these are amplified by $a\mW_2$.
This slows the optimization, but the defense is trivially broken: the attacker can inspect per-layer weight norms, detect the imbalance, and rebalance them.

\paragraph{Example 2: invertible matrix insertion.}
Between two consecutive linear layers, one can insert an invertible matrix~$\mA$ and its inverse: $(\mW_2 \mA^{-1})(\mA \mW_1)$.
In a transformer, this applies to consecutive linear maps such as the key--query or value--output projections, where an adversarial choice of $\mA$ can significantly disturb the individual factors.
A concrete example of matrix factorization is given in Appendix~\ref{app:scale_symmetry}.
However, even in this case, the adaptive attacker can compute the singular value decomposition (SVD) of the product $(\mW_2 \mA^{-1})(\mA \mW_1) = \mW_2\mW_1 = \mU \Sigma \mV^\top$ and reparameterize as $(\mU \Sigma^{1/2})(\Sigma^{1/2} \mV^\top)$, a canonical factorization entirely independent of the defender's choice of~$\mA$.

\paragraph{From symmetry-based to optimization-based approaches.}
These examples illustrate a general difficulty: known symmetries are limited to simple patterns (scaling, permutation, invertible linear transforms) that an informed attacker can detect and reverse, at negligible additional cost in terms of~\eqref{eq:adapt_cost}.
This motivates \emph{optimization-based} approaches, where the defense is embedded through training and the configuration serves as secret information---even with full knowledge, exact unlearning requires solving a hard optimization problem.
We discuss two possible families below.

\paragraph{(i) Barrier-based locking: disturbing the input Jacobian.}
The idea is to insert a barrier module~$B$ between layers whose input Jacobian $\partial B(\vh) / \partial \vh$ causes gradients to explode or vanish, making fine-tuning difficult.
Two constraints arise: the barrier must be inseparable from the model, and its form must be learned rather than hand-crafted, since an adaptive attacker can otherwise remove or reverse it.
Replacing entire modules is a natural approach, but standard distillation produces smooth approximations, so achieving pathological Jacobians remains difficult.

\emph{Limitation:} well-hidden barriers that sufficiently block gradient flow also make it more difficult for the defender to maintain model quality, as required by condition~\emph{(i)} of Definition~\ref{def:locking}.

\paragraph{(ii) Penalty-based locking: optimizing for ill-conditioning.}
An alternative approach is to directly search for ill-conditioned solutions via optimization, rather than relying on known symmetries.
One can formulate this as a constrained optimization problem:
\begin{equation}\label{eq:lock-obj}
  \max_{\theta}\;\Omega(\theta) \quad \text{s.t.} \quad
  \mathcal{L}(\theta) < \mathcal{L}(\theta_0) + \varepsilon,
\end{equation}
where $\Omega$ measures difficulty of fine-tuning.
However, this formulation is self-defeating:

\emph{Limitation}: the defender seeks to make optimization hard for the
attacker, yet must itself rely on optimization to find such a solution, caught in a self-made trap.

To confirm this, we test a simplified instance of~\eqref{eq:lock-obj} on a two-layer ReLU MLP trained on MNIST~\citep{lecun2010mnist}.
We optimize $\min_{\Delta\theta} \operatorname{MSE}(f_{\theta_0}, f_{\theta_0 + \Delta\theta}) - \lambda\,\Omega(\theta_0+\Delta\theta)$, where the MSE term preserves the full input--output mapping to maintain the original performance, sweeping over $\lambda$ and learning rates with two proxy penalties: $\Omega = \|\Delta\theta\|_2^2$ (distance from initialization) and $\Omega = \mathrm{Tr}(\nabla^2_\theta \mathcal{L}_{\mathrm{CE}})$ (sharpness via Hutchinson trace~\citep{Hutchinson01011990}).
Neither directly measures fine-tuning difficulty; if gradient descent cannot achieve even these simpler objectives under the constraint, the full problem~\eqref{eq:lock-obj} is unlikely to be easier.
As shown in Figure~\ref{fig:specdef}, no trajectory reaches the high-$\Omega$ corner while maintaining the original function.
Yet such solutions
\emph{do} exist: the weight symmetry in Example~1 preserves the function exactly while increasing both quantities.
Gradient descent simply cannot reach them from the pretrained initialization (see also Appendix~\ref{app:hessian_scopes}).

\section{\method{}: Deep Low-Rank Residual Locking}
\label{sec:method}

The approaches discussed in Section~\ref{sec:difficult} all attempt to modify the loss landscape or gradient flow.
We take a different route: we replace each SwiGLU MLP with a \emph{\arch{}} that computes the same function but is structurally expensive to differentiate.
The key observation is that inference and training have fundamentally different memory requirements: users who run the model are unaffected, while those who update its weights face additional cost.
\begin{itemize}
  \item \textbf{Inference} evaluates the model in a single forward pass, discarding intermediates.
  \item \textbf{Training} must store all intermediate activations for backpropagation.
\end{itemize}
This is not a property of the loss landscape or the optimizer but a structural property of automatic differentiation.
An architecture that is cheap to evaluate but expensive to differentiate creates an inherent, optimizer-agnostic defense---one that does not depend on finding particular weights.

While the defender replaces each MLP independently without touching other components, the attacker optimizes all parameters jointly through $N \times L$ layers.
Beyond the memory cost, two additional optimization challenges may arise: the effective depth may cause training instability, and standard fine-tuning recipes developed for the original architecture are unlikely to transfer, since the deep residual MLPs differ substantially from the MLP blocks they replace.

\subsection{Deep low-rank residual network}
\label{sec:architecture}

A practical defense must not disadvantage legitimate users: the replacement must (i)~match the original module's input--output behavior and (ii)~not significantly increase the total parameter count or inference cost.
We use \emph{low-rank} weight matrices in each residual block: the bottleneck rank~$r \ll d$ keeps the per-layer parameter count small, allowing the total budget to match (or only slightly exceed) that of the original SwiGLU MLP while distributing the parameters across many layers.

We replace the MLP residual block $\vx' = \vz + \mathrm{FFN}(\mathrm{RMSNorm}(\vz))$ with hidden dimension~$d$ and intermediate dimension~$d_{\mathrm{ff}}$ by a depth-$L$ low-rank residual network:
\begin{equation}
\label{eq:arch}
  \vh_0 = \vz, \quad
  \vh_{i+1} = \vh_i + \alpha_i \mU_i\, \sigma\,\!\bigl(
    \mV_i\, \mathrm{RMSNorm}(\vh_i)\bigr), \quad
  \vx' = \vh_L,
\end{equation}
where $\sigma$ is the SiLU nonlinearity,
$\mV_i \in \R^{r \times d}$ and $\mU_i \in \R^{d \times r}$ are the bottleneck projections, and $\alpha_i$ is a learnable ReZero~\citep{bachlechner2021rezero} scalar.
Each layer costs $2dr$ parameters plus $c_{\mathrm{norm}}$ for normalization, so at a fixed budget~$P$ the depth is $L = \lfloor P / (2dr + c_{\mathrm{norm}}) \rfloor$.
Smaller rank yields a deeper network at the same parameter cost.
RMSNorm pre-normalization is essential to our design: it places a nonlinearity on the $d$-dimensional stream, so backpropagation retains $\Theta(d)$ activations per layer even with the \arch{} frozen.
Appendix~\ref{app:ablation_arch} reports these architectural ablations.

These per-layer costs accumulate across the full network.
We quantify the cost amplification for the simplest attack strategy; stronger attacks are analyzed in Section~\ref{sec:attacks}:

\begin{proposition}[Cost amplification under naive fine-tuning]
\label{prop:kappa}
Consider an attacker who uses naive fine-tuning
$\mathfrak{A} = \{\gA_{\mathrm{ft}}\}$, i.e., gradient descent with all activations retained for the backward pass, and assume that
\emph{(i)}~the number of tokens required to reach target performance~$p$ on task~$\gT$ is the same for $f$ and~$g$, and
\emph{(ii)}~per-GPU throughput scales proportionally with the maximum feasible micro-batch size, up to constant factors.
Then the GPU-hour cost of fine-tuning the locked model satisfies $\kappa = \Omega(d/r)$ when MLP activations dominate the memory budget.
\end{proposition}

The bound is conservative: it assumes equal sample efficiency between $f$ and~$g$; in practice, the architectural mismatch introduced by the deep residual MLPs may further slow convergence, increasing the true~$\kappa$.
The detailed derivation, including per-layer memory analysis, is given in Appendix~\ref{app:activation_memory}.

\begin{algorithm}[t]
\caption{\method{}: model locking via deep low-rank residual distillation}
\label{alg:defense}
\begin{algorithmic}[1]
\REQUIRE Pretrained model $\mathcal{M}$ with $N$ transformer layers, each
  containing a normalization--MLP residual sub-block; bottleneck rank~$r$;
  per-MLP parameter budget~$P$; training data~$\mathcal{D}$
\ENSURE Locked model $\hat{\mathcal{M}}$ with MLP sub-blocks replaced by \arch{}s
\STATE \textbf{Phase 1: Module-wise regression}
\FOR{each layer $l = 1, \ldots, N$ \textbf{in parallel}}
  \STATE Initialize \arch{} $g^{(l)}$ with rank $r$ and depth $L = \lfloor P / (2dr + c_{\mathrm{norm}}) \rfloor$
  \STATE Train $g^{(l)}$ to minimize $\mathcal{L}_{\mathrm{MSE}} = \frac{1}{n}\sum_{i} \|g^{(l)}(\vz_i) - f_l(\vz_i)\|^2 / \|f_l(\vz_i)\|^2$
  \STATE Replace normalization, MLP, and skip connection with $g^{(l)}$ in $\hat{\mathcal{M}}$
\ENDFOR
\STATE \textbf{Phase 2: Logit distillation}
\STATE Freeze all parameters in $\hat{\mathcal{M}}$ except \arch{} weights
\STATE Train $\hat{\mathcal{M}}$ to minimize $\mathcal{L}_{\mathrm{KD}} = \sum_{j \in \mathrm{top\text{-}k}} p_j^{\mathrm{teacher}} \log(p_j^{\mathrm{teacher}}/p_j^{\mathrm{student}})$ using $\mathcal{M}$ as teacher
\end{algorithmic}
\end{algorithm}

\Needspace{5\baselineskip}

\subsection{Distillation procedure}
\label{sec:distillation}

\begin{wrapfigure}{r}{0.44\textwidth}
\vspace{-\baselineskip}
\centering
\includegraphics[width=0.42\textwidth]{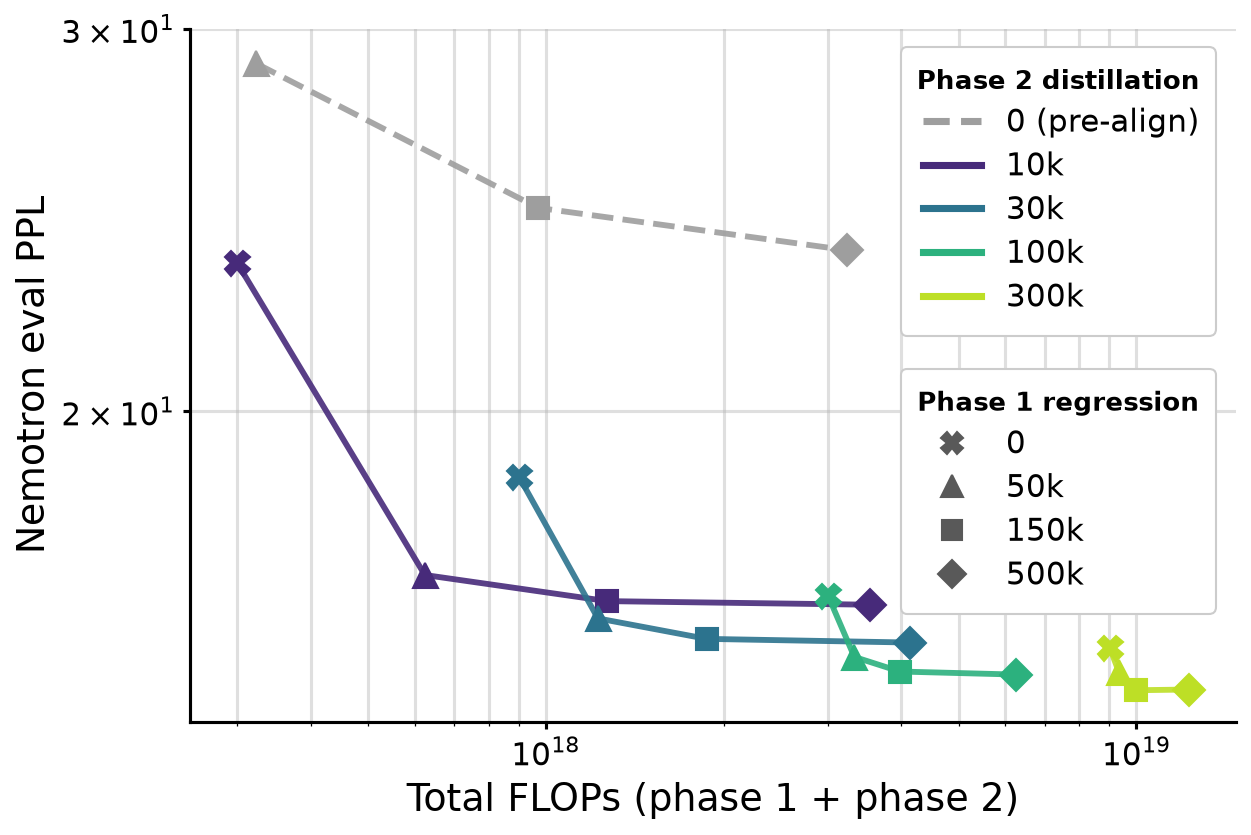}
\caption{Compute-allocation Pareto frontier for the two-phase distillation of Qwen3-0.6B, sweeping a grid of Phase-1 regression and Phase-2 distillation steps. Skipping Phase~1 is Pareto-dominated at every compute budget. Appendix~\ref{app:ablation_compute_split} gives the FLOPs accounting and hyperparameter configurations.}
\label{fig:compute_split_pareto}
\vspace{-2.0\baselineskip}
\end{wrapfigure}

We distill the pretrained MLP into the \arch{} in two phases, summarized in Algorithm~\ref{alg:defense}: a \emph{module-wise} phase that approximates each module's input--output mapping independently, followed by a \emph{logit distillation} phase that corrects the accumulated error by matching the full model's output distribution.
Both phases aim to preserve the original function of the pretrained model.
By construction, this satisfies condition~\emph{(i)} of Definition~\ref{def:locking} for every downstream task, unlike methods that optimize a task-specific loss and only guarantee preservation on that objective.

\paragraph{Phase 1: Module-wise regression.}
For each transformer layer~$l$, we collect hidden states $\vz_l$ (after the
attention sub-layer) from the pretrained model.  The \arch{} $g^{(l)}$ is
trained to replicate the full residual sub-block $f_l(\vz) = \vz + \mathrm{FFN}_l(\mathrm{RMSNorm}(\vz))$ by minimizing the relative MSE, normalized by $\|f_l\|^2$ because output norms vary across layers.

\paragraph{Phase 2: Logit distillation.}
After all MLP layers are replaced, the small per-layer distillation errors
compound across the full forward pass, degrading language modeling performance.
We correct this by fine-tuning the \arch{} parameters (MLP only; attention
and embeddings are frozen) using the pretrained model as teacher.
The distillation loss is the top-$k$ KL divergence between teacher and student output distributions, restricting the sum to the $k$ tokens with highest teacher probability.

Phase~2 requires back-propagating through the full locked model, facing the same activation-memory barrier as the attacker; however, Phase~1 provides a good initialization, and the defender pays this cost only once. Figure~\ref{fig:compute_split_pareto} confirms that this two-phase design is both necessary and efficient.

\subsection{Attack strategies}
\label{sec:attacks}

We consider strategies an attacker may use to fine-tune the locked model.

\paragraph{Gradient checkpointing.}
\label{sec:checkpointing}
Gradient checkpointing~\citep{chen2016training} is the most direct countermeasure to activation memory overhead.
It reduces peak memory from $O(NLd)$ to $O(Ld)$, at the cost of recomputing each MLP forward pass.
The optimal trade-off between recomputation and saved activations is achieved at $O(\sqrt{NL}\,d)$ via evenly spaced checkpoints.
If the attacker cannot afford even this memory, they must resort to hierarchical checkpointing, additional GPUs, gradient accumulation, or compromises such as shorter sequence lengths or reduced precision.
This recomputation cost is unavoidable for any attacker who fine-tunes the locked model directly.

\paragraph{Frozen \arch{}s.}
The attacker may freeze the \arch{} parameters and train only the rest.
However, freezing does not remove the activation memory: the RMSNorm backward still retains $\Theta(d)$ activations per layer, as noted in Section~\ref{sec:architecture}.
If the attacker further applies \texttt{stop\_grad} to avoid backpropagation entirely, gradients to earlier layers omit the main branch's contribution, producing a biased signal that degrades training.
Restricting the trainable parameters also limits adaptation, so reaching a given target takes more steps than full fine-tuning, or may fail entirely.

\paragraph{Low-rank adaptation (LoRA).}
LoRA~\citep{hu2022lora} reduces memory by freezing the base weights and training low-rank adapters.
However, LoRA does not shorten the backward pass: backpropagation still traverses all $L$ layers to compute gradients for the adapters.
Moreover, it is non-trivial where to apply LoRA in the \arch{}: since $\mU_i \in \R^{d \times r}$ and $\mV_i \in \R^{r \times d}$ are already low-rank, an adapter on one is nearly as large as the matrix itself, so adapting them all saves little over full fine-tuning.
Appendix~\ref{app:attack_lora} tests this attack: no fine-tuning rate closes the gap to the unlocked baseline.

\paragraph{Reverse distillation.}
\label{sec:reverse_distill}
Since we distill SwiGLU $\to$ \arch{}, the attacker can run the same procedure in reverse, regressing each \arch{} back into a standard MLP before fine-tuning the recovered model.
It is the most important strategy here, since its cost is amortized rather than recurring per task, as in Remark~\ref{rem:amortization}.
Appendix~\ref{app:attack_reverse} reports the reverse-distillation counterpart of Figure~\ref{fig:compute_split_pareto}, confirming that recovery costs on the order of the locking budget itself, so unlocking the weights and passing them on is not a cheap operation.
Additionally, one asymmetry here is structural: efficient recovery needs a corpus close to the data used to pretrain or lock the model, which only the provider holds.

In summary, the attacker's budget determines which route is rational:

\begin{center}
\small
\begin{tabular}{@{}lll@{}}
\toprule
Attacker budget & Rational strategy & Treated in \\
\midrule
$C_{\mathrm{attack}} \gtrsim \min\{C_{\mathrm{pretrain}}, C_{\mathrm{distill}}\}$
  & pretrain or distill a fresh model & Assumption~\ref{asm:kerckhoffs} \\
$C_{\mathrm{attack}} \gtrsim C_{\mathrm{lock}}$, several tasks
  & reverse distillation & Appendix~\ref{app:attack_reverse} \\
$C_{\mathrm{attack}} \lesssim C_{\mathrm{lock}}$, or a single task
  & fine-tune the locked model directly & Section~\ref{sec:exp_resistance} \\
\bottomrule
\end{tabular}
\end{center}

\section{Experiments}
\label{sec:experiments}

\subsection{Experimental setup}
\label{sec:setup}

We use Qwen3~\citep{yang2025qwen3} as our testbed.
For distillation, we use the Nemotron-CC dataset~\citep{su2025nemotron}, a curated web text corpus.
Weights are initialized as
$\mU_i \sim \mathcal{N}(0,\, 1/(d\sqrt{L}))$ and
$\mV_i \sim \mathcal{N}(0,\, 1/r)$, and $\alpha_i = 0$.
For training, both phases use AdamW~\citep{loshchilov2018decoupled} with weight decay~$10^{-5}$, cosine learning-rate schedule with linear warmup
(5\% of total steps), and gradient clipping at norm~1.
For Phase~1, each layer is trained independently for 500k steps with learning rate $10^{-3}$, batch size~2, and sequence length~2048.
The rank is chosen so that the depth falls in the range $L \in [140, 150]$ (e.g., $r = 32$ for 0.6B).
For Phase~2, we train all \arch{} parameters jointly for 100k steps with top-$k{=}1000$, learning rate $10^{-3}$, batch size~4 and sequence length~2048.
Appendix~\ref{app:add_scratch} considers the provider's own setting: models pretrained from scratch on Nemotron-CC and locked on the same corpus.
Baseline and locked arms always run on the same number and type of GPUs, so the wall-clock ratios we report are GPU-hour ratios in the sense of Definition~\ref{def:locking}.

\subsection{Quality preservation}
\label{sec:exp_e2e}

Since our algorithm replaces every MLP module, we first verify that the locked model retains the original model's capabilities, condition~\emph{(i)} of Definition~\ref{def:locking}, which concerns the published model before any fine-tuning.
We evaluate on the Nemotron-CC validation set and WikiText-103~\citep{merity2016pointer} for perplexity in and out of the distillation distribution, and on seven LM Evaluation Harness benchmarks~\citep{eval-harness} for downstream performance: MMLU~\citep{hendrycks2021measuring}, ARC (Easy and Challenge)~\citep{clark2018think}, HellaSwag~\citep{zellers2019hellaswag}, WinoGrande~\citep{sakaguchi2021winogrande}, BoolQ~\citep{clark2019boolq}, and PIQA~\citep{bisk2020piqa}.

\begin{table}[b]
\caption{Quality preservation of Qwen3-0.6B.
``Module-wise reg.'' replaces every SwiGLU MLP with a
\arch{} in Phase~1.  ``+Logit dist.'' applies logit
distillation in Phase~2.
Nemo = Nemotron, WT = WikiText-103, ARC-E/C = ARC Easy/Challenge,
HSw = HellaSwag, WiGr = WinoGrande.}
\label{tab:eval-benchmarks}
\centering
\small
\begin{tabular}{l cc ccccccc}
\toprule
& \multicolumn{2}{c}{Perplexity $\downarrow$}
& \multicolumn{7}{c}{Accuracy (\%) $\uparrow$} \\
\cmidrule(lr){2-3} \cmidrule(lr){4-10}
 & Nemo & WT & MMLU & ARC-E & ARC-C & HSw & WiGr & BoolQ & PIQA \\
\midrule
Pretrained           & 14.9 & 21.5 & 40.3 & 55.9 & 33.7 & 47.3 & 56.4 & 63.8 & 67.3 \\
\arrayrulecolor{gray!80}\midrule\arrayrulecolor{black}
Module-wise reg.  & 21.8 & 182 & 31.1 & 50.6 & 26.5 & 40.5 & 54.9 & 62.6 & 63.4 \\
+Logit dist.      & 14.9 & 23.4 & 36.8 & 50.7 & 29.0 & 43.5 & 55.2 & 63.8 & 65.3 \\
\bottomrule
\end{tabular}
\end{table}

Table~\ref{tab:eval-benchmarks} shows the results.
Regression alone leaves a noticeable gap in several metrics.
Distillation recovers Nemotron perplexity almost completely, as expected for in-distribution data, and WikiText perplexity is also largely recovered since both are natural text corpora.
On the remaining benchmarks, distillation closes much of the gap, bringing most metrics within a few points of the pretrained model.
Matching the data to the model's own pretraining corpus, as a provider naturally would, closes more of this gap: the from-scratch models of Appendix~\ref{app:add_scratch} stay closer to their pretrained baselines, at the same locking budget.
Additionally, we used $100\text{k}$ Phase-2 steps as an initial setting, and if we extend the distillation, WikiText perplexity also matches the pretrained model (Appendix~\ref{app:ablation_compute_split}).

\subsection{Inference overhead}
\label{sec:inference}

Replacing each MLP with a depth-$L$ \arch{} preserves total theoretical inference FLOPs~(Section~\ref{sec:architecture}), but the sequential chain of small matrix multiplications adds wall-clock overhead from kernel launches and memory traffic.
We measure it in \texttt{bfloat16} on H100 GPUs using PyTorch~2.9~\citep{10.1145/3620665.3640366} across model scales.
Two off-the-shelf levers remove most of the overhead: \texttt{torch.compile} with CUDA graph replay (\texttt{reduce-overhead}), and a custom kernel for the \arch{}.
Appendix~\ref{app:latency_impl} gives the full kernel strategies and the measurement protocol.

\paragraph{Compilation and a custom kernel close the gap.}
Because the \arch{} issues many small kernel launches per MLP, it benefits disproportionately from CUDA graph replay, which removes most of the eager-execution overhead.
Our custom kernel closes most of the residual gap, bringing the single-MLP overhead to roughly $2\text{--}5\times$ across scales (Figure~\ref{fig:mlp_overhead}, left; Appendix~\ref{app:latency_results} for the full four-scale sweep).
The baseline is itself compiled with \texttt{reduce-overhead}, so these are ratios against an optimized SwiGLU.
The residual gap shrinks with model size, as each residual block performs proportionally larger matrix multiplications, and with batch size, as they become compute-bound.

\paragraph{Defense--speed trade-off.}
In a full transformer, attention, embeddings, and the output projection are identical to the baseline and increasingly dominate at longer sequences, diluting the per-MLP overhead measured above.
Figure~\ref{fig:mlp_overhead} (right) shows the full-model prefill overhead on Qwen3-8B with the custom kernel: even the strongest configuration ($L{=}143$) reaches only $1.22\times$ at $32\text{k}$ tokens, a moderate depth ($L{=}35$) is nearly transparent at $1.04\times$, and the shallowest ($L{=}17$) matches the SwiGLU baseline at $1.02\times$.
The defender thus tunes this trade-off through the rank~$r$ and depth~$L$ (recall $2drL \approx 3d \cdot d_{\mathrm{ff}}$ at fixed parameter count), choosing how much forward-pass latency to spend for locking strength (see also Appendix~\ref{app:ablation_rank}).

\begin{figure}[t]
\begin{minipage}[c]{0.42\linewidth}
  \centering
  \includegraphics[width=\linewidth]{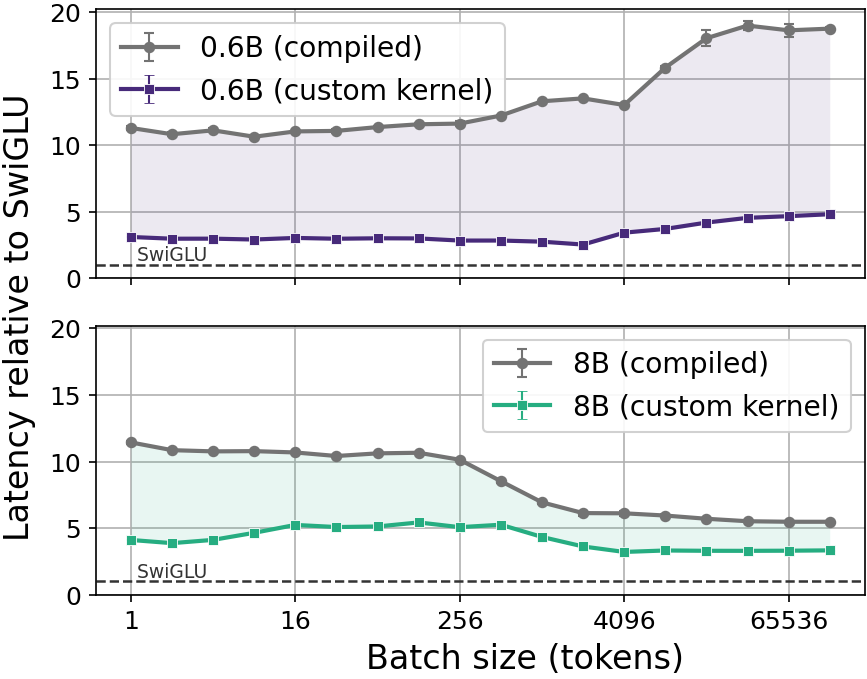}
\end{minipage}%
\hfill
\begin{minipage}[c]{0.57\linewidth}
  \centering
  \small
  \begin{tabular}{@{}l*{4}{w{c}{1.1cm}}@{}}
    \toprule
    \multicolumn{5}{c}{Qwen3-8B, fixed $n{=}32\text{k}$ tokens, prefill} \\
    \midrule
    & \multicolumn{4}{c}{Depth $L$} \\
    \cmidrule(l){2-5}
    Seq.\ len & 143 & 71 & 35 & 17 \\
    \midrule
    1024   & $2.35\times$ & $1.63\times$ & $1.24\times$ & $1.09\times$ \\
    8192   & $1.64\times$ & $1.29\times$ & $1.11\times$ & $1.04\times$ \\
    32768  & $1.22\times$ & $1.10\times$ & $1.04\times$ & $1.02\times$ \\
    \bottomrule
  \end{tabular}
\end{minipage}
\caption{%
  \textbf{Left:} Single-MLP latency of the \arch{} relative to the SwiGLU MLP as a function of batch size, for 0.6B (top) and 8B (bottom).
  The SwiGLU baseline is compiled with \texttt{torch.compile} in \texttt{reduce-overhead} mode; relative to it we plot the \arch{} under the same compilation (gray) and with our custom kernel (colored).
  The shaded band is the residual gap the kernel closes, and the dashed line at $1$ is the baseline.
  \textbf{Right:} Full-model prefill overhead on Qwen3-8B with the custom kernel.
  Each column is a different depth $L$ (rank $r = 128, 256, 512, 1024$).}
\label{fig:mlp_overhead}
\end{figure}

\subsection{Resistance to fine-tuning}
\label{sec:exp_resistance}

\begin{figure}[t]
\centering
\includegraphics[width=0.93\linewidth]{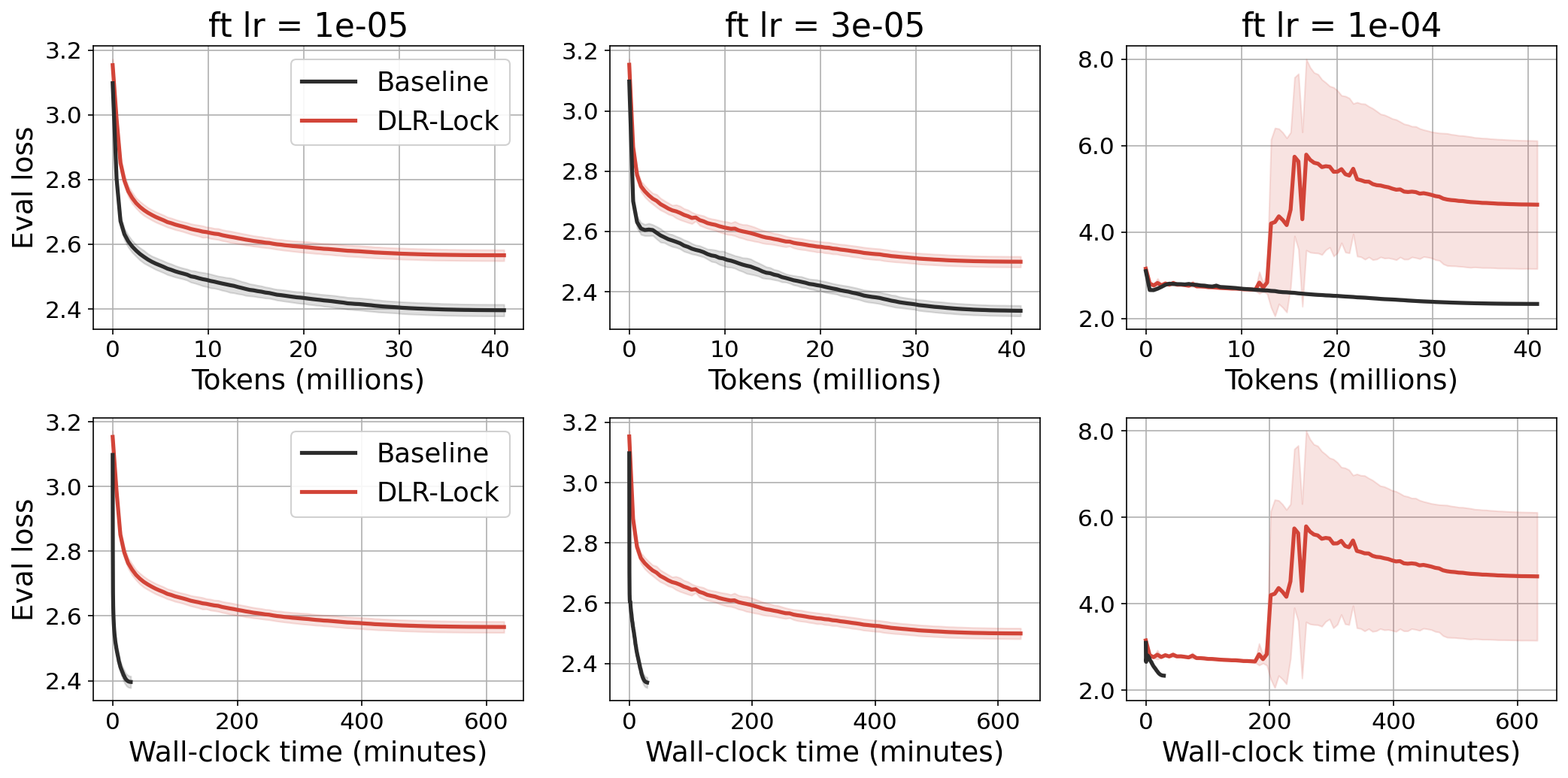}
\caption{Fine-tuning convergence comparison between the pretrained baseline and the locked model on WikiText-103.
\textbf{Top:} loss vs.\ number of training tokens (same batch size and sequence length for both); \textbf{Bottom:} loss vs.\ wall-clock time.
Columns correspond to fine-tuning learning rates.
The smallest rate swept, $3{\times}10^{-6}$, is omitted; it tracks $10^{-5}$ at a slower pace.
Shaded regions indicate $\pm 1$ standard deviation across 3 seeds.}
\label{fig:finetune_speed_wikitext}
\end{figure}

We now evaluate the core claim: whether the locked model resists fine-tuning by an attacker.
We use the same locked model as in Section~\ref{sec:exp_e2e}, with gradient checkpointing enabled to reduce memory overhead.
The attacker sweeps the fine-tuning learning rate over $\{3{\times}10^{-6}, 10^{-5}, 3{\times}10^{-5}, 10^{-4}\}$, each compared against fine-tuning the baseline unlocked model at the same rate.

\paragraph{Results.}
Figure~\ref{fig:finetune_speed_wikitext} shows convergence curves at three of the four rates.
The locked model starts at a similar loss to the baseline, so the final loss gap cannot be explained by the initial one alone.
At lower rates ($10^{-5}$, $3{\times}10^{-5}$), the baseline steadily improves while the locked model converges more slowly and plateaus at a higher loss, leaving a persistent gap even at the attacker's best rate.
At the largest rate ($10^{-4}$), the locked model diverges entirely while the baseline trains normally, illustrating that our defense destabilizes the optimization itself, narrowing the attacker's effective range of learning rates.
The same curves also bound the cost of reaching a fixed target, rather than the loss reached at a fixed budget.
The locked model's best loss, reached at $3{\times}10^{-5}$, costs the baseline $\mathbf{4.8\times}$ fewer tokens at $10^{-5}$, its cheapest route there; the baseline's own best loss the locked model never reaches.
On top of that, the locked model alone must recompute activations at this batch, $8$ rather than $6$ FLOPs per parameter-token (Remark~\ref{rem:checkpointing}).
The floor is therefore $4.8 \times 4/3 = \mathbf{6.4\times}$: a FLOPs ratio at matched quality that no implementation removes, and a wall-clock floor too whenever the baseline uses the hardware no less efficiently than the locked model.

The bottom row bounds that cost ratio from above.
In eager wall-clock, the pair differs by $\mathbf{109\times}$ in GPU-hours, a figure that also carries the forward slowdown and therefore overstates what an attacker facing an optimized deployment would pay.
Decomposing one optimizer step (Appendix~\ref{app:step_cost}) separates the per-step component from the optimization difference above: the training overhead accounts for $\mathbf{87\%}$ of the step for the locked model against $73\%$ for the baseline, and it is $\mathbf{6.7\times}$ the forward pass, against $2.7\times$ for the baseline.
In practice a provider can ship a model whose forward is served by the custom kernel of Section~\ref{sec:inference}.
The next paragraph shows that the backward admits no equivalent optimization, which further widens these ratios of training-related to inference-only cost.

\Needspace{6\baselineskip}

\paragraph{Backward pass resists optimization.}
\begin{wraptable}{r}{0.38\textwidth}
\centering
\small
\setlength{\tabcolsep}{5pt}
\caption{Single \arch{}, deployed setting: a custom kernel makes the forward near-free, but the backward pass speeds up consistently less.}
\label{tab:compile_asymmetry}
\vspace{0.3em}
\begin{tabular}{@{}l cc cc@{}}
\toprule
& \multicolumn{2}{c}{0.6B} & \multicolumn{2}{c}{8B} \\
\cmidrule(lr){2-3}\cmidrule(l){4-5}
$n$ & $t_f'/t_f$ & $t_b'/t_b$ & $t_f'/t_f$ & $t_b'/t_b$ \\
\midrule
4    & 0.02 & \textbf{0.09} & 0.04 & \textbf{0.09} \\
512  & 0.02 & \textbf{0.11} & 0.08 & \textbf{0.12} \\
8192 & 0.07 & \textbf{0.39} & 0.18 & \textbf{0.40} \\
\bottomrule
\end{tabular}
\end{wraptable}
Let $t_f$ denote the inference forward time, $t_e$ one training step (forward, loss, and backward), and $t_b = t_e - t_f$ the training overhead over inference; primes denote the deployed setting, with $t_f'$ our custom-kernel forward and $t_e'$ the \texttt{reduce-overhead}-compiled training step, the cheapest speedup route open to an attacker.
Table~\ref{tab:compile_asymmetry} isolates a single \arch{}: at every $n$, optimization slashes the forward ($t_f'/t_f \ll 1$) yet reduces the training overhead by less ($t_b'/t_b > t_f'/t_f$), the same widening found at full-model scale (Appendix~\ref{app:step_cost}).
This gap is inherent, not an implementation artifact.
\emph{First}, the provider ships a forward-only kernel; having no backward, it cannot be reused for training, so the attacker falls back to generic compilation, or pays upfront to engineer a backward kernel of their own.
\emph{Second}, even if they take that route, a hand-written backward cannot recover the forward's speed either.
Our kernel is fast because the hidden state stays on chip: each layer updates it in place, so the only traffic to memory is its own thin, rank-$r$ weight.
The backward cannot: each layer's $O(d)$ activation must be written out and read back, a far larger transfer than the thin weight the forward moves; its gradient, summed over tokens, also breaks the forward's free tiling.
The defender thus gives ordinary users a near-free forward, while any attacker who computes gradients pays a disproportionate, unavoidable cost.

\section{Conclusion}
\label{sec:conclusion}

We introduced a model locking approach that replaces MLP layers with deep low-rank residual networks, exploiting the inference--training asymmetry of automatic differentiation.
The locked model imposes three complementary barriers: activation memory that grows with depth, architectural mismatch destabilizing fine-tuning, and a backward pass disproportionately costlier than the forward pass.
A two-phase distillation procedure preserves the original model's capabilities while achieving a measurable locking effect against adaptive attackers.
Our custom kernel already removes much of the forward-pass overhead; narrowing what remains, through further kernel work or batched computation across residual layers, is a systems problem in its own right, which we leave to future work.

\bibliographystyle{plainnat}
\bibliography{references}

\newpage
\appendix

\section{Related Work}
\label{app:related}

\paragraph{Preventing harmful model editing.}
A comprehensive survey of security risks in open-weight models is provided
by~\citet{casper2026open}.  The first direction is preventing fine-tuning to
harmful content~\citep{huang2024vaccine,tamirisa2025tamperresistant,zhao2025understanding,huang2025booster,zheng2025model,liu2025targeted,wang2026selfdestructive,yang2026asft,nguyen2026antibody}
or robustifying unlearned models against relearning
attacks~\citep{lynch2024eight,hu2025unlearning,fan2025towards,lee2025distillation}.
The major line of work takes a meta-learning approach, optimizing a bilevel
objective or training the model under simulated
tampering.
\citet{qi2025on} show that existing objective-level defenses are
broken by simple hyperparameter changes, motivating structural alternatives.
Another approach modifies the model's internal
representations~\citep{rosati2024representation}, which can be naturally
combined with our method: representation modification alters the input--output relationship for safety, whereas our architectural replacement
preserves it.
The two mechanisms are complementary and can be applied jointly.
A separate line of work targets model merging
prevention~\citep{junhao2025disrupting,wang2025model}: these methods
reparameterize weight matrices using the symmetries discussed in
Section~\ref{sec:difficult}, moving the model into a different loss basin to
disrupt merging.  However, such reparameterizations can be detected and
reversed by an attacker who knows the defense strategy.

\paragraph{Module replacement via distillation.}
Building on knowledge distillation~\citep{hinton2015distilling}, several works replace individual transformer modules with compact alternatives for model compression~\citep{xu2020bert,zhou2023modular,liang2023modulewise,liang2023less,lo2024m2mkd,ro2025onthefly}.
Closest to our approach, \citet{cheng2026attention} replace attention modules with alternatives such as multi-head latent attention via module-wise distillation without retraining from scratch.

\paragraph{Model locking.}
Model locking aims to make pretrained weights resistant to \emph{any} form of adaptation---not only harmful fine-tuning.
While methods preventing harmful editing apply to both the open-weight and fine-tuning-as-a-service (FTaaS, where the provider performs the fine-tuning) settings, model locking specifically targets the open-weight case, where the attacker controls the full training pipeline.
SpecDef~\citep{rosati2025locking} and \citet{wang2026towards} are the closest
works to ours; we discuss them in detail in the following paragraphs.
Concurrently, \citet{rosati2026limits} study spectral weight decomposition as a defense but acknowledge that adaptive attackers can trivially revert it.
Additionally, they show that any spectral deformation can be diluted by decomposing it across consecutive layers, motivating a shift from weight-space defenses to our structural approach.
Our work differs in that we formalize the problem (Definitions~\ref{def:defense}--\ref{def:locking}) and propose an architectural defense with a quantifiable locking factor~$\kappa$.

\paragraph{Spectral Deformation (SpecDef; \citealt{rosati2025locking}).}
SpecDef maximizes the top-$k$ singular values of weight matrices as an instance of \eqref{eq:lock-obj} with $\Omega(\theta) = \sum_l \sum_{i=1}^{k} \sigma_i(\mW_l(\theta))$.
This is a post-hoc method: it can be applied to any pretrained model without retraining from scratch.
We note several open challenges: maximizing $\Omega$ while minimizing the task loss~$\mathcal{L}$ is a multi-objective problem subject to the self-defeating structure discussed in Section~\ref{sec:difficult} (Figure~\ref{fig:specdef}); moreover, using the task loss as the faithfulness constraint risks overfitting to that specific task, with no guarantee that general capabilities are preserved---in contrast, distillation-based approaches maintain the input--output relationship of the original model and thus preserve performance across all tasks in principle.
Additionally, the experimental evaluation is conducted at moderate scale, leaving scaling behavior as an open question.

\paragraph{Non-Fine-Tunable Foundation Models~\citep{wang2026towards}.}
This method embeds a lottery-ticket-style binary mask during pretraining: a subset of weights is designated as ``locked'' (masked) and the remaining weights are trained to compensate.
An attacker unaware of the mask structure optimizes all weights, entering a different optimization space from the one the model was trained in.
We note several open challenges: the binary mask reduces effective model capacity since masked weights do not contribute to the forward computation; the mask structure may be detectable via statistics such as weight magnitude or second-order derivatives, in addition to the gradient norm they investigated; the experimental evaluation is conducted at small scale; and most importantly, the comparison between masked and unmasked models reports similar benchmark performance, but both achieve scores near chance level on binary tasks (e.g., ${\sim}50\%$), making it difficult to conclude that the proposed locked method does not sacrifice the original performance.

\section{Activation memory analysis}
\label{app:activation_memory}

This section provides the detailed activation memory analysis underlying
Proposition~\ref{prop:kappa}.

\begin{proposition}[DLR-Net activation memory]
\label{prop:activation}
For a depth-$L$ network of the form~\eqref{eq:arch}:
\begin{enumerate}
\item[\emph{(a)}] \textbf{Full training.}
  Each layer stores: the RMSNorm input~$\vh_i$ ($d$), the RMSNorm output ($d$),
  the SiLU pre-activation ($r$), the SiLU output ($r$), and the block
  output ($d$) for the $\alpha_i$~gradient.
  Per-layer cost: $3d + 2r \approx 3d$.  Total: $\Theta(dL)$.
\item[\emph{(b)}] \textbf{Frozen parameters.}
  Each layer stores: the RMSNorm input~$\vh_i$ ($d$, needed because
  the backward Jacobian depends on $\vh_i$) and the SiLU pre-activation ($r$).
  Per-layer cost: $d + r \approx d$.  Total: $\Theta(dL)$.
\end{enumerate}
\end{proposition}

\noindent Case~(a) is the standard training cost.
Importantly, case~(b) shows that this overhead cannot be avoided even when the
attacker freezes all DLR-Net parameters and trains only other components.

Using this result, we now prove the cost amplification bound stated in the main text.

\noindent\textbf{Proof of Proposition~\ref{prop:kappa}.}
Recall that the pretrained transformer has $N$~layers, hidden dimension~$d$,
and FFN intermediate dimension~$d_{\mathrm{ff}}$.
The locked model replaces all $N$ SwiGLU MLPs with depth-$L$ \arch{}s
of rank~$r$, where $L = \lfloor 3d_{\mathrm{ff}} / (2r) \rfloor$ matches the
original parameter count.
By Proposition~\ref{prop:activation}(a), the total per-token activation memory
for the MLP layers increases from $\Theta(N d_{\mathrm{ff}})$ (one SwiGLU per
layer) to $\Theta(N d L)$ (depth-$L$ \arch{} per layer), an amplification
factor of $\Theta(d/r)$ at matched parameter count.

On a GPU with memory capacity~$M$, let $M_{\mathrm{fixed}}$ denote the memory
that does not scale with the micro-batch: weights, gradients and optimizer
state.  Since the \arch{} is parameter-matched to the MLP it replaces,
$M_{\mathrm{fixed}}$ agrees between the two models up to constant factors.
The largest micro-batch that fits is
$B = \lfloor (M - M_{\mathrm{fixed}}) / m \rfloor$, where $m$ is the per-token
activation memory.
Let $a_{\mathrm{attn}}$ denote the per-token per-layer activation memory for
the attention blocks (unchanged), and let $B_f$, $B_g$ denote the largest
micro-batches for the original and locked models respectively.  Their ratio
satisfies
\[
  \frac{B_g}{B_f}
  \;\lesssim\;
  \frac{a_{\mathrm{attn}} + d_{\mathrm{ff}}}{a_{\mathrm{attn}} + dL}.
\]
The attacker must process the same number of tokens~$T$, by
assumption~\emph{(i)}.
A smaller micro-batch does not change the total arithmetic work: the \arch{} is
FLOP-matched to the SwiGLU MLP it replaces, so the same~$T$ tokens require the
same total FLOPs however they are batched.  What changes is the rate at which
that arithmetic runs.  By assumption~\emph{(ii)}, per-GPU throughput is
proportional to~$B$ up to constant factors, so the GPU-hours required to
process~$T$ tokens scale as~$T/B$.  Hence,
\[
  \kappa
  \;\gtrsim\;
  \frac{B_f}{B_g}
  \;\gtrsim\;
  \frac{a_{\mathrm{attn}} + dL}{a_{\mathrm{attn}} + d_{\mathrm{ff}}}.
\]
At matched parameter count, $L = \Theta(d_{\mathrm{ff}}/r)$, hence
$dL = \Theta(d\,d_{\mathrm{ff}}/r)$.  When MLP activations dominate the memory
budget, $a_{\mathrm{attn}} = O(d_{\mathrm{ff}})$, so the denominator is
$\Theta(d_{\mathrm{ff}})$, while the numerator is $\Theta(dL)$ since
$r \ll d$.  Therefore, $\kappa = \Omega(d/r)$.
\hfill$\square$

\begin{remark}[Gradient checkpointing]
\label{rem:checkpointing}
Proposition~\ref{prop:kappa} considers naive fine-tuning with all activations
retained for the backward pass.  Gradient checkpointing falls outside this
restricted attack: it trades activation memory for recomputation, as discussed
in Section~\ref{sec:checkpointing}.  For example, recomputing each forward
activation once increases the standard training cost from approximately $6$ to
$8$ FLOPs per parameter-token, a $4/3$ increase.  More aggressive checkpointing
can further reduce memory at the expense of additional recomputation.
\end{remark}

\section{Ill-conditioning from weight symmetries}
\label{app:illcond}

\subsection{A matrix-factorization example}
\label{app:scale_symmetry}

We illustrate the ill-conditioning caused by the weight symmetries
discussed in Section~\ref{sec:difficult} on the matrix factorization problem
\begin{equation}\label{eq:matfac}
  \min_{\mW_1, \mW_2 \in \R^{d \times d}} \;
  \ell(\mW_1, \mW_2) \;=\; \|\mW_1\mW_2 - \mM\|_F^2,
\end{equation}
where $\mM \in \R^{d \times d}$ is a fixed target.
This simple setting captures the structure
of Example~2: for any invertible $\mA \in \R^{d \times d}$,
the reparameterization
$(\mW_1,\mW_2) \mapsto (\mW_1\mA^{-1},\,\mA\mW_2)$
preserves the product $\mW_1\mW_2$ exactly while changing the individual
factors and hence the optimization landscape.
We study the scalar case $\mA = a\mI$ with $a > 0$,
giving $(\mW_1,\mW_2) \mapsto (a^{-1}\mW_1,\,a\mW_2)$.

Before presenting the experimental results, we show that such a scaling
worsens the condition number of the Hessian, leading to slower convergence.

\begin{proposition}[Ill-conditioning under scaling]\label{prop:scaling}
Let $\ell(\mW_1,\mW_2) = \|\mW_1\mW_2 - \mM\|_F^2$ with
$\tilde{\mW}_1 = a^{-1}\mW_1$ and $\tilde{\mW}_2 = a\mW_2$, and suppose that $\mW_1$ and $\mW_2$ are full-rank.
Then, the condition number of the Hessian of $\ell$ with respect to the parameters $(\tilde{\mW}_1, \tilde{\mW}_2)$ satisfies
\begin{equation}\label{eq:cond}
  \mathrm{cond}(\tilde{\mH}) \;\geq\;
    a^4 \cdot \frac{\sigma_{\max}^2(\mW_2)}{\sigma_{\min}^2(\mW_1)},
\end{equation}
where $\tilde{\mH}$ is that Hessian, $\mathrm{cond}(\cdot)$ denotes the condition number, and $\sigma_{\max}(\cdot)$, $\sigma_{\min}(\cdot)$ denote the largest and smallest singular values, so that $\mathrm{cond}(\tilde{\mH}) = \Omega(a^4)$.
We write $\mathrm{cond}(\cdot)$ rather than~$\kappa$ here, since $\kappa$ denotes the locking factor of Definition~\ref{def:locking} throughout the paper.
\end{proposition}
\begin{proof}
Since $\ell = \|\tilde{\mW}_1\tilde{\mW}_2 - \mM\|_F^2$, the gradients with respect to each weight matrix are given by
\begin{equation}\label{eq:grad_w1_w2}
\nabla_{\tilde{\mW}_1} \ell
= 2 ( \tilde{\mW}_1\tilde{\mW}_2 - \mM ) \tilde{\mW}_2^\top,
\quad
\nabla_{\tilde{\mW}_2} \ell
= 2 \tilde{\mW}_1^\top ( \tilde{\mW}_1\tilde{\mW}_2 - \mM ).
\end{equation}
Let $\mathrm{vec}(\cdot)$ denote column-major vectorization of a matrix.
Differentiating again gives
\begin{equation}\label{eq:hess_aa}
  \tilde{\mH}_{1,1}
    \;=\; \frac{\partial^2 \ell}{\partial\,\mathrm{vec}(\tilde{\mW}_1)^2}
    \;=\; 2\,(\tilde{\mW}_2 \tilde{\mW}_2^\top) \otimes \mI_d
    \;=\; 2a^2\,(\mW_2\mW_2^\top) \otimes \mI_d,
\end{equation}
where $\otimes$ denotes the Kronecker product.
Similarly, we have
\begin{equation}\label{eq:hess_bb}
  \tilde{\mH}_{2,2}
    \;=\; \frac{\partial^2 \ell}{\partial\,\mathrm{vec}(\tilde{\mW}_2)^2}
    \;=\; 2\,\mI_d \otimes (\tilde{\mW}_1^\top \tilde{\mW}_1)
    \;=\; \tfrac{2}{a^2}\,\mI_d \otimes (\mW_1^\top\mW_1).
\end{equation}
Therefore, the largest singular value of the Hessian $\tilde{\mH}$ with respect to $(\tilde{\mW}_1, \tilde{\mW}_2)$ is given by
\begin{equation}\label{eq:hess_aa_max}
  \sigma_{\max}(\tilde{\mH})
  \;=\;
  \max_{\| \vv \|_2 = 1} \vv^\top \tilde{\mH} \vv
  \;\geq\;
  \max_{\| \vv \|_2 = 1} \vv^\top \tilde{\mH}_{1,1} \vv
  \;=\;
  2a^2 \sigma_{\max}\left( (\mW_2\mW_2^\top) \otimes \mI_d \right),
\end{equation}
where the inequality follows by restricting the last $d^2$ entries to zero, and the last equality is from \eqref{eq:hess_aa}.
Using the property of the Kronecker product, we have
\begin{equation}\label{eq:hess_max}
  \sigma_{\max}(\tilde{\mH})
  \;\geq\;
  2a^2 \sigma_{\max}\left( \mW_2\mW_2^\top \right) \sigma_{\max}\left( \mI_d \right)
  \;=\;
  2a^2 \sigma_{\max}\left( \mW_2 \right)^2.
\end{equation}
Similarly, \eqref{eq:hess_bb} leads to
\begin{equation}\label{eq:hess_min}
  \sigma_{\min}(\tilde{\mH})
  \;\leq\;
  \frac{2}{a^2} \sigma_{\min}\left( \mW_1 \right)^2.
\end{equation}
Combining with \eqref{eq:hess_max} gives the desired result.
\end{proof}

This result shows that the condition number is worsened at the rate of $a^4$.
The full-rank assumption in Proposition~\ref{prop:scaling} is satisfied almost surely under standard random initialization.

\begin{remark}[Connection to classical convergence bounds]

Proposition~\ref{prop:scaling} only shows that the reparameterization worsens conditioning in this matrix-factorization example.
Independently of this example, classical first-order optimization theory provides several conditions in which a larger condition number leads to slower convergence \citep{boyd2004convex}.
As one representative case, we focus on the classical $\mu$-strongly convex and $\beta$-smooth setting, writing $\beta$ for the smoothness constant since $L$ denotes the \arch{} depth elsewhere in the paper; gradient descent with stepsize smaller than $1/\beta$ has the linear convergence rate $1 - \mu/\beta$, yielding the $\epsilon$-iteration complexity $t = O((\beta/\mu) \log(1/\epsilon)) = O(\mathrm{cond} \cdot \log(1/\epsilon))$, where $\mathrm{cond} = \beta/\mu$ is the condition number.
Therefore, the larger condition number increases the complexity bound, which implies a slower convergence.
\end{remark}

\paragraph{Experimental verification.}

We verify the above discussion empirically by training the matrix factorization problem~\eqref{eq:matfac} with and without the scaling reparameterization.
We set $d = 64$ with entries of $\mM$, $\mW_1$, $\mW_2$ drawn i.i.d.\ from $\mathcal{N}(0, 1/d)$, apply the scaling with $a = 100$, and train with both SGD and Adam across three learning rates each.

\begin{figure}[t]
  \centering
  \includegraphics[width=0.9\linewidth]{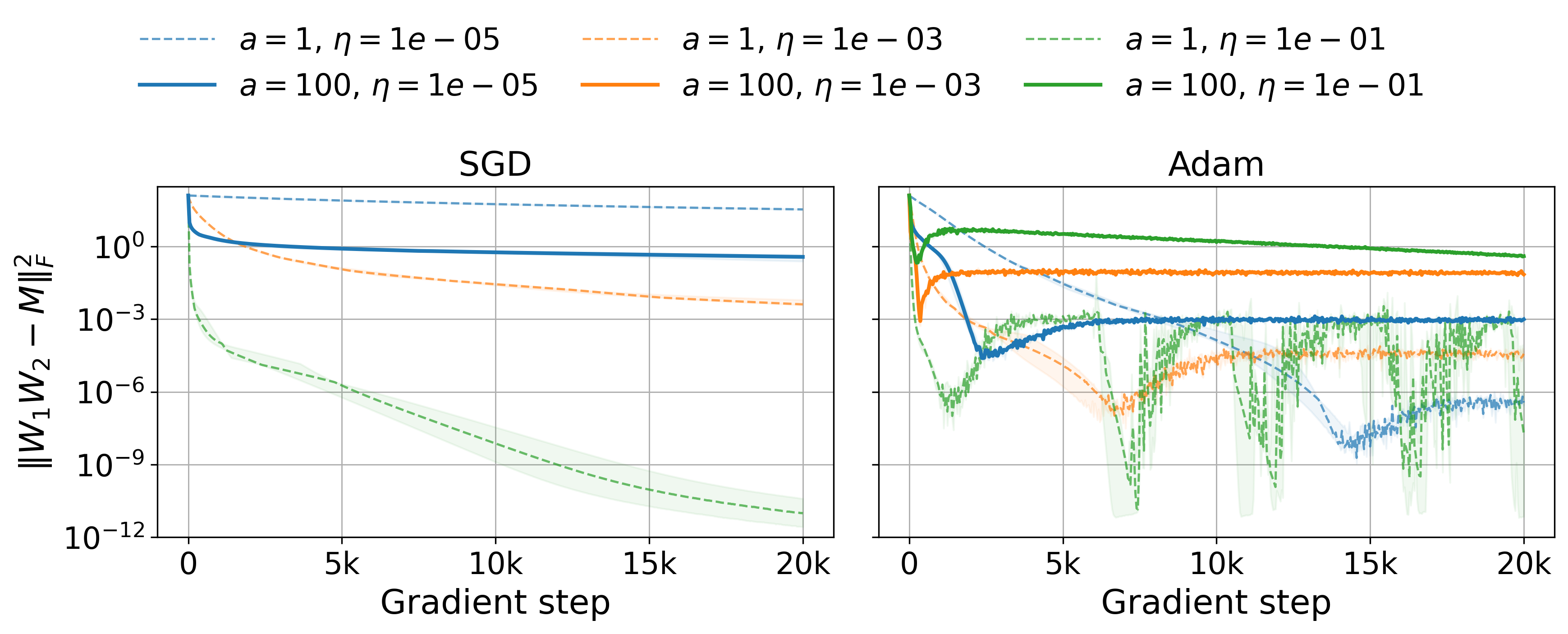}
  \caption{%
  Training loss for matrix factorization $\|\mW_1\mW_2 - \mM\|_F^2$ from balanced ($a=1$, dashed) and scaled ($a=100$, solid) initialization.
  \textbf{Left:} SGD. Under the scaling, the two larger learning rates diverge (not shown). Only the smallest learning rate remains stable, yet it plateaus orders of magnitude above the $a=1$ optimum.
  \textbf{Right:} Adam. All learning rates converge at $a=100$,
  albeit slower than at $a=1$.
  Shaded regions show the interquartile range over five seeds.}
  \label{fig:scale_symmetry}
\end{figure}

Figure~\ref{fig:scale_symmetry} shows that the bad condition number actually leads to a slower convergence.
For SGD at $a=100$, only the smallest learning rate is stable; and the convergence is extremely slow because the gradient of the enlarged factor $\tilde{\mW}_2 = a\mW_2$ scales as $a^{-1}$, so it barely
updates while $\tilde{\mW}_1$ converges to its conditionally optimal value.
Larger learning rates immediately diverge.
Adam does not exhibit the same divergence because its element-wise normalization makes the effective step size approximately scale-invariant.
However, it still converges slower than at $a=1$, confirming that the ill-conditioning affects optimization even with adaptive methods.

\subsection{Perturbation limit: Hessian scope comparison}
\label{app:hessian_scopes}

Figure~\ref{fig:specdef} (right panel) computes the Hessian trace over all parameters.
A natural question is whether restricting the penalty to a subset of parameters makes the optimization easier---for example, concentrating the sharpness penalty on a single layer or a random subset might reduce interference between the MSE and $\Omega$ objectives.
This setting directly corresponds to increasing the condition number of the Hessian: high curvature in a subset of directions makes gradient-based fine-tuning over the full parameter space difficult.
Figure~\ref{fig:hessian_scopes} tests four scopes: all parameters, the first layer only, a random 10\% mask, and a random 50\% mask.
The result is consistent across all scopes: trajectories move rightward (increasing MSE) as $\lambda$ grows, but fail to reach the high-Hessian-trace region occupied by scaling-symmetry solutions.
This confirms that the optimization difficulty observed in the main text is not an artifact of the penalty scope.

\begin{figure}[t]
  \centering
  \includegraphics[width=0.9\linewidth]{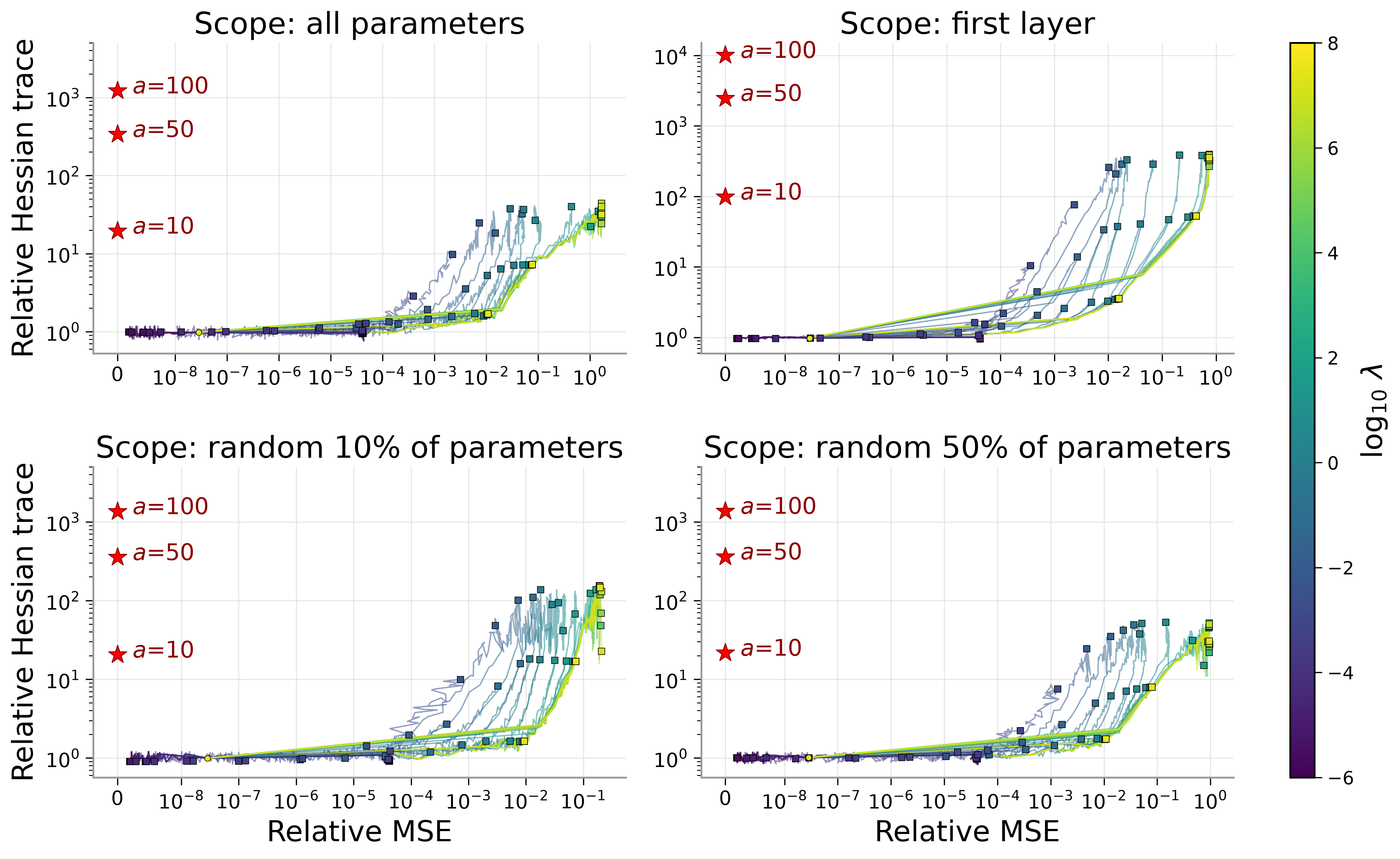}
  \caption{Hessian-trace penalty across four parameter scopes.
  Same setup as Figure~\ref{fig:specdef} (right panel), but each panel restricts the Hessian trace $\Omega = \mathrm{Tr}(H_{\mathrm{CE}})$ to a different parameter subset.
  Each curve is one trajectory; color encodes $\log_{10}\lambda$ (dark = small, bright = large).
  Circles: start; squares: end.
  Red stars: weight-symmetry solutions $(a^{-1}\mW_1,\, a\mW_2)$.
  $x$-axis: relative MSE; $y$-axis: relative Hessian trace over the corresponding subset.
  Across all scopes, no trajectory reaches the low-MSE, high-Hessian-trace corner occupied by scaling solutions.}
  \label{fig:hessian_scopes}
\end{figure}

\section{Compute allocation between regression and distillation phases}
\label{app:ablation_compute_split}

Section~\ref{sec:distillation} specifies our two-phase distillation procedure but treats its compute allocation only qualitatively. Here we quantify how the model provider---the defender, who pretrains and then releases the model---should divide a fixed compute budget between Phase~1 (module-wise regression) and Phase~2 (end-to-end logit distillation).

\paragraph{Setup.}
We sweep the number of Phase-1 regression steps ($0$, $50\text{k}$, $150\text{k}$, $500\text{k}$) against the number of Phase-2 distillation steps ($0$, $10\text{k}$, $30\text{k}$, $100\text{k}$, $300\text{k}$) on Qwen3-0.6B, where a step count of $0$ skips that phase; the cell that skips both is excluded, as it leaves the \arch{}s untrained. Following the training configuration of Section~\ref{sec:setup} (rank $r{=}32$, distillation learning rate $10^{-3}$, top-$k{=}1000$, cosine schedule with $5\%$ warmup), we report perplexity on the Nemotron-CC evaluation set.
Our default configuration, used throughout the paper, is $500\text{k}$ regression and $100\text{k}$ distillation steps.

\paragraph{FLOPs accounting.}
The $x$-axis of Figure~\ref{fig:compute_split_pareto} is the total FLOPs of the distillation procedure across both phases, excluding the teacher forward pass (deterministic in the inputs, hence cacheable).
We count gradient-checkpointing recompute where it occurs: Phase~2 back-propagates through the full locked model and, facing the very activation-memory barrier our defense imposes, must recompute each decoder block, whereas Phase~1 trains each \arch{} independently and needs none.

Per Phase-1 step, we train all $N_\text{mlp}{\approx}263{\times}10^6$ \arch{} parameters (the 28 module-wise MLP replacements) on $T_1{=}4096$ student tokens each, at $6$ FLOPs per parameter-token for one forward and backward pass:
\begin{equation*}
F_{1,\text{step}} = 6\,N_\text{mlp}\,T_1 \approx 6.5 \times 10^{12} \text{ FLOPs}.
\end{equation*}

Per Phase-2 step, we back-propagate through the full model ($N_\text{total}{\approx}596{\times}10^6$ parameters) over $T_2{=}8192$ tokens, with only the 28 \arch{} MLPs ($N_\text{mlp}$ parameters) trainable; let $N_\text{block}{\approx}380{\times}10^6$ denote the attention and MLP matmul parameters across the decoder blocks. Activation gradients traverse the full frozen backbone while weight gradients cover only the trainable MLPs:
\begin{equation*}
F_{2,\text{step}} = \underbrace{2 N_\text{total} T_2}_{\text{fwd}} + \underbrace{2 N_\text{block} T_2}_{\text{block recompute (GC)}} + \underbrace{2 N_\text{total} T_2}_{\text{act-grads}} + \underbrace{2 N_\text{mlp} T_2}_{\text{weight-grads (MLP-only)}} \approx 3.0 \times 10^{13} \text{ FLOPs}.
\end{equation*}

A configuration's total distillation FLOPs is then its Phase-1 step count times $F_{1,\text{step}}$ plus its Phase-2 step count times $F_{2,\text{step}}$.

\paragraph{Implication: \method{} must be post-hoc.}
Figures~\ref{fig:compute_split_pareto} and~\ref{fig:compute_split_pareto_by_r} show the same Nemotron-CC perplexity frontier from two angles, grouped by distillation length and by regression length. This speaks to whether \method{} could be folded into pretraining rather than applied post-hoc. Training the \arch{} from scratch as a drop-in SwiGLU replacement, with no per-layer teacher, is exactly the zero-regression setting, which is Pareto-dominated at every compute budget: the module-wise regression targets of Phase~1 supply a supervision signal that end-to-end logit distillation alone cannot match at comparable compute. \method{} is therefore only viable as a \emph{post-hoc} distillation from a pretrained backbone, not a from-scratch pretraining objective.

\paragraph{Implication: the residual quality gap is a budget choice.}
A second, incidental finding of this sweep concerns the quality gap of Table~\ref{tab:eval-benchmarks}.
Holding Phase~1 at our default $500\text{k}$ steps and extending Phase~2 past the $100\text{k}$ steps used throughout the paper keeps improving quality on held-out natural text: as shown in Table~\ref{tab:csplit_phase2}, WikiText-103 perplexity falls from $23.5$ to $21.5$, matching the pretrained model's $21.5$.
Phase~2 is expensive by construction, since the defender must back-propagate through the very model it is locking.
Its cost is nonetheless one the defender can pay: as noted in Section~\ref{sec:formulation}, Definition~\ref{def:locking} constrains the attacker's adaptation cost but places no constraint on the locking cost, which is paid once and amortized over every deployment of the released weights.
The gap that further distillation closes is therefore not an intrinsic limit of \method{}, but a budget allocation under the defender's control.

\begin{figure}[t]
\centering
\includegraphics[width=0.6\linewidth]{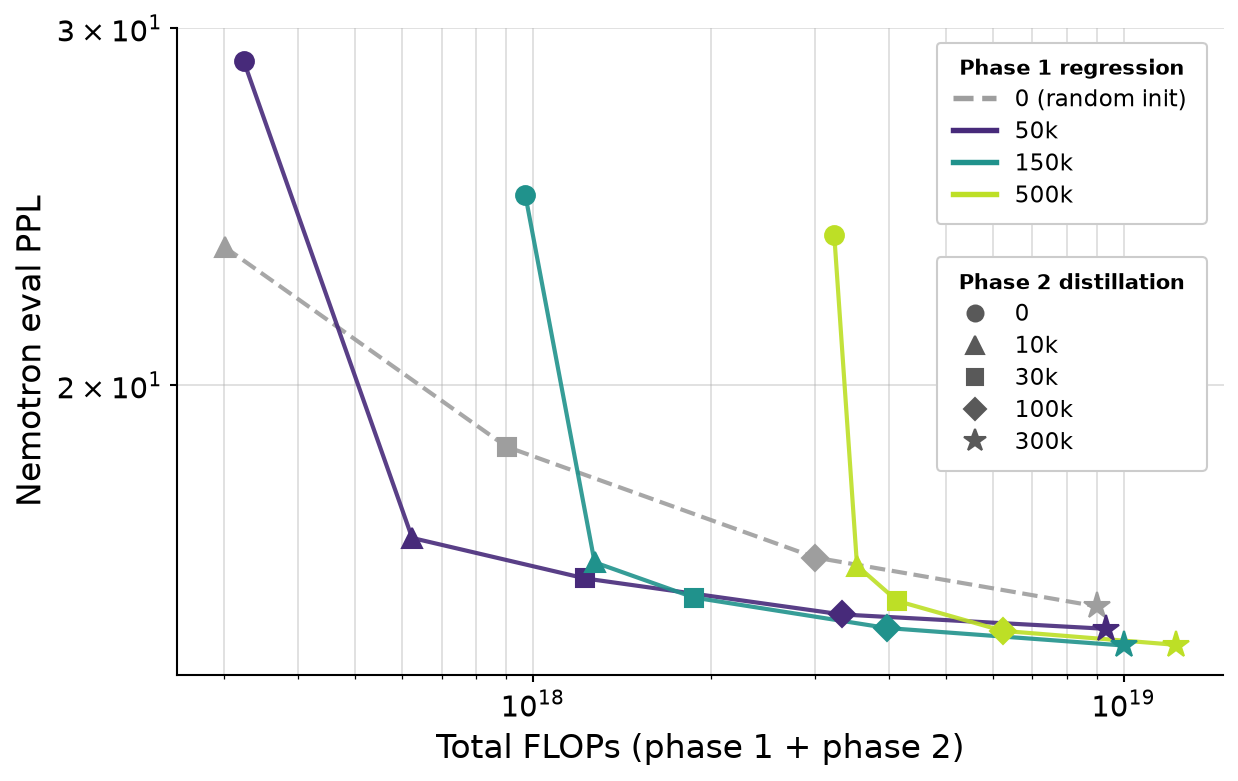}
\caption{The same compute-split grid as Figure~\ref{fig:compute_split_pareto}, now grouped by Phase-1 regression length. Color encodes the number of Phase-1 regression steps ($0$, $50\text{k}$, $150\text{k}$, $500\text{k}$), with $0$ (gray dashed) the random-initialized DLR-Nets that skip module-wise regression entirely; marker shape encodes Phase-2 distillation steps ($0$, $10\text{k}$, $30\text{k}$, $100\text{k}$, $300\text{k}$). The paper-original configuration uses $500\text{k}$ regression and $100\text{k}$ distillation steps.}
\label{fig:compute_split_pareto_by_r}
\end{figure}

\begin{table}[t]
\caption{Extending Phase~2 at our default Phase-1 budget ($500\text{k}$ regression steps), Qwen3-0.6B. At $300\text{k}$ distillation steps, three times the budget used in the main text, WikiText-103 perplexity reaches that of the pretrained model ($21.5$, Table~\ref{tab:eval-benchmarks}). A Phase-2 budget of $0$ is the Phase-1-only model. The $100\text{k}$ column is an independent rerun of the default configuration and reproduces Table~\ref{tab:eval-benchmarks}'s $23.4$. Total FLOPs follow the accounting above.}
\label{tab:csplit_phase2}
\centering
\small
\begin{tabular}{l ccccc}
\toprule
Phase-2 steps & $0$ & $10\text{k}$ & $30\text{k}$ & $100\text{k}$ & $300\text{k}$ \\
\midrule
Total FLOPs & $3.2{\times}10^{18}$ & $3.5{\times}10^{18}$ & $4.1{\times}10^{18}$ & $6.2{\times}10^{18}$ & $1.2{\times}10^{19}$ \\
WikiText-103 PPL & $182$ & $26.9$ & $25.1$ & $23.5$ & $\mathbf{21.5}$ \\
\bottomrule
\end{tabular}
\end{table}

\section{\method{} across model settings}
\label{app:additional}

This section extends the study of \method{} beyond the main text along two axes: model scale and deployment setting. Appendix~\ref{app:add_pretrained} keeps the main-text model, Qwen3, and its WikiText-103 fine-tuning target, extending both claims to larger scales and adding a data-limited fine-tuning setting.

In the main text we use Qwen3 because it is a strong, widely adopted open model, which makes our results easy to reproduce and compare. However, Qwen3 was pretrained by a different provider on a proprietary corpus, so locking it with Nemotron-CC leaves a domain gap, and the degradation we report is a \emph{conservative} estimate. In practice, \textbf{\method{} is applied by a model provider to its own model}: the party that pretrains a model and releases its weights is also the one that locks it, so it can run \method{}'s regression and distillation on the \emph{same data it pretrained on}. Appendix~\ref{app:add_scratch} realizes that setting: we pretrain our own models on Nemotron-CC and lock them on the same corpus, so the pretraining and locking data coincide by construction and the domain-mismatch confound of the Qwen3 setting disappears. There we study both central claims of \method{}, quality preservation and resistance to fine-tuning, across model scales and fine-tuning domains.

\subsection{Pretrained Qwen3}
\label{app:add_pretrained}

This is the setting studied in the main text: we take a publicly released pretrained model, Qwen3, and lock it with \method{} using Nemotron-CC for regression and distillation. Every experiment here fine-tunes on WikiText-103, as the main text does. Appendix~\ref{app:add_scratch} carries both the data-matched version of these quality numbers and the other fine-tuning domains, since we pretrained those models ourselves and can evaluate them on a new target without a domain gap.

\paragraph{Larger model scales.}
Table~\ref{tab:qwen3-quality} extends the main-text $0.6$B result of Table~\ref{tab:eval-benchmarks} to $1.7$B and $8$B. Following the rank rule of Section~\ref{sec:setup}, the bottleneck widens with the model so that depth stays in the target range: $r{=}64$ and $L{=}142$ at $1.7$B, $r{=}128$ and $L{=}143$ at $8$B, against $r{=}32$ and $L{=}141$ at $0.6$B. Depth is thus essentially constant across scales: the growing parameter budget is absorbed by a wider bottleneck rather than a longer chain. We hold $L$ in this range throughout for comparability, but it is a free parameter rather than a constraint. At a fixed budget rank trades directly against depth, and Appendix~\ref{app:ablation_rank} follows that ladder to much shallower networks, which cost less forward latency and lock correspondingly less. The $8$B model's Phase-2 batch is halved by memory, so it is distilled for twice as many steps. All scales therefore see the same $0.82$B tokens. Even so, the $8$B knowledge gap (e.g.\ MMLU) remains larger than at $1.7$B. This is what a fixed budget predicts: the token budget is the same at every scale, so a model with an order of magnitude more parameters receives proportionally less distillation, and closing the gap would likely require scaling the defender's budget with the model rather than holding it constant. At $0.6$B, the compute-allocation frontier of Figure~\ref{fig:compute_split_pareto} in Appendix~\ref{app:ablation_compute_split} shows quality still improving as Phase-2 distillation is extended, so our operating point is not a ceiling, and we expect the same to hold at $8$B. Additionally, two mechanisms contribute to this gap and can be separated: the pretraining/locking distribution shift, and a distillation budget that does not scale with the model. The from-scratch setting of Appendix~\ref{app:add_scratch} removes the first by construction, and the gap there is smaller at every scale, leaving the budget as the remaining cause.
One column matters especially for the fine-tuning experiment below, which fine-tunes on WikiText-103: at $1.7$B that perplexity is preserved to within $0.06$, whereas at $8$B it rises from $9.83$ to $12.13$, so the two arms would not start fine-tuning from the same loss. We therefore study fine-tuning resistance at $1.7$B.

\begin{table}[t]
\caption{Quality preservation of Qwen3 at $1.7$B and $8$B before and after \method{}, at $k{=}1000$ and each model's best alignment learning rate: $\eta{=}10^{-4}$ for $1.7$B and $8$B, versus $10^{-3}$ at $0.6$B, since larger models favor a smaller alignment rate. ``+DLR-Lock'' replaces every SwiGLU MLP with a \arch{} via the two-phase distillation pipeline. The $0.6$B result is in Table~\ref{tab:eval-benchmarks}. Nemo = Nemotron-CC, WT = WikiText-103, ARC-E/C = ARC Easy/Challenge, HSw = HellaSwag, WiGr = WinoGrande.}
\label{tab:qwen3-quality}
\centering
\small
\setlength{\tabcolsep}{4pt}
\begin{tabular}{ll cc ccccccc}
\toprule
& & \multicolumn{2}{c}{Perplexity $\downarrow$}
& \multicolumn{7}{c}{Accuracy (\%) $\uparrow$} \\
\cmidrule(lr){3-4} \cmidrule(lr){5-11}
Model & Variant & Nemo & WT & MMLU & ARC-E & ARC-C & HSw & WiGr & BoolQ & PIQA \\
\midrule
\multirow{2}{*}{1.7B} & Pretrained  & 11.95 & 17.13 & 55.4 & 70.2 & 43.0 & 60.4 & 61.5 & 77.6 & 72.3 \\
                      & +DLR-Lock   & 11.94 & 17.19 & 48.5 & 68.1 & 38.9 & 56.2 & 60.0 & 78.1 & 70.9 \\
\midrule
\multirow{2}{*}{8B}   & Pretrained  &  8.57 &  9.83 & 73.0 & 80.9 & 56.7 & 74.9 & 68.1 & 86.6 & 77.8 \\
                      & +DLR-Lock   &  9.06 & 12.13 & 58.1 & 76.9 & 49.7 & 70.1 & 68.2 & 83.5 & 75.7 \\
\bottomrule
\end{tabular}
\end{table}

\begin{figure}[t]
\centering
\includegraphics[width=\linewidth]{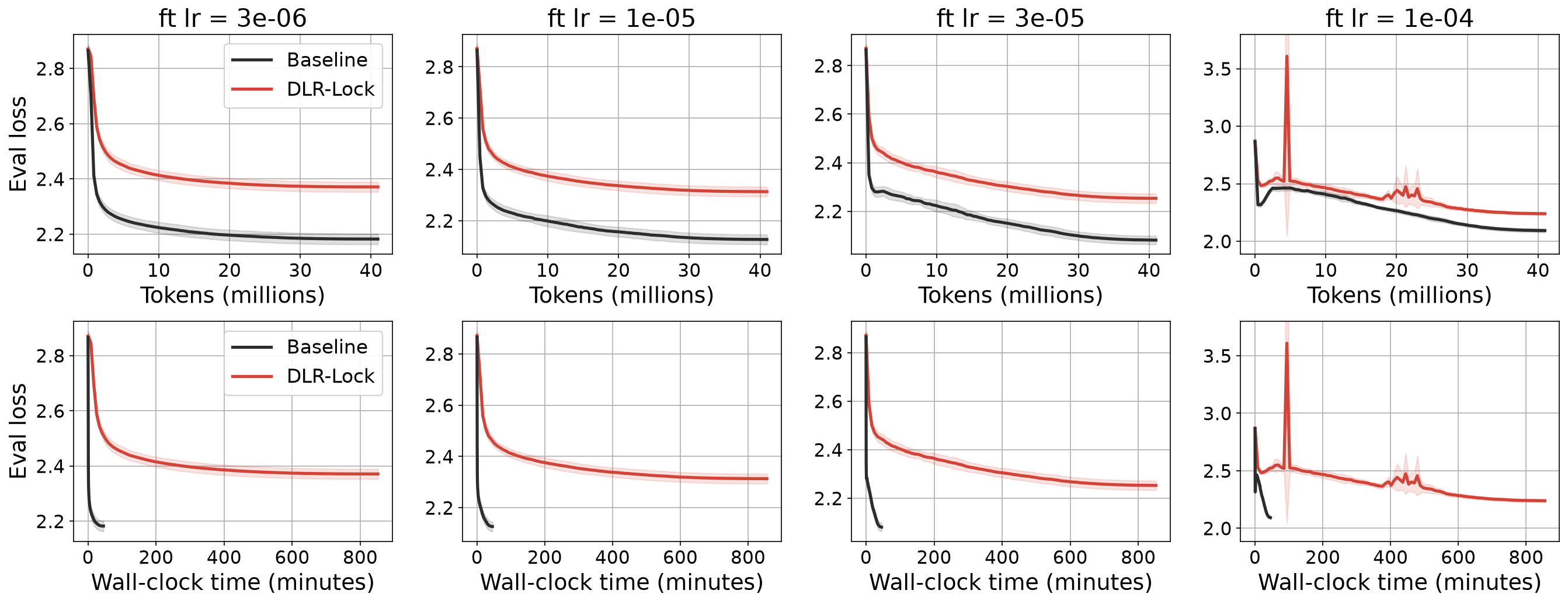}
\caption{Fine-tuning resistance on WikiText-103 at Qwen3-$1.7$B, the scale extension of Figure~\ref{fig:finetune_speed_wikitext}, whose protocol it follows.
\textbf{Top:} loss vs.\ number of training tokens; \textbf{Bottom:} loss vs.\ wall-clock time.
Columns correspond to fine-tuning learning rates; unlike the main-text figure, all four swept rates are shown.
The vertical axis of the $10^{-4}$ column is clipped, where the locked model spikes transiently before recovering.
Shaded regions indicate $\pm 1$ standard deviation across 3 seeds.}
\label{fig:ft_1_7b}
\end{figure}

\paragraph{Fine-tuning resistance at larger scale.}
The protocol is the main text's, unchanged in every other respect; $8$B is covered instead by the data-matched from-scratch setting of Appendix~\ref{app:add_scratch}, where the corresponding gap is smaller.
The arms start together, and Figure~\ref{fig:ft_1_7b} shows the locked model converging more slowly and plateauing higher at every rate.
Read as a cost to reach a fixed target, the locked model's best loss, reached at $10^{-4}$, costs the baseline $\mathbf{10.0\times}$ fewer tokens at $10^{-5}$, its cheapest route there; the baseline's own best loss the locked model never reaches.
The locked model alone recomputes activations at this batch (Remark~\ref{rem:checkpointing}), a further factor of $4/3$, so the FLOPs floor is $\mathbf{13.3\times}$.
In eager wall-clock the pair differs by $\mathbf{189\times}$, which as in the main text carries the forward slowdown too and so bounds the ratio from above.
Both ratios exceed their $0.6$B counterparts ($4.8\times$ and $109\times$), so the effect does not weaken with scale.
The instability does not carry over: the largest rate diverges at $0.6$B but here only spikes transiently, to $5.8$ in one seed, before $10^{-4}$ becomes the locked model's best rate.
The claim at this scale therefore rests on the persistent loss gap, not on a narrowed range of usable rates.

\begin{figure}[t]
\centering
\begin{subfigure}{0.95\linewidth}
\centering
\includegraphics[width=\linewidth]{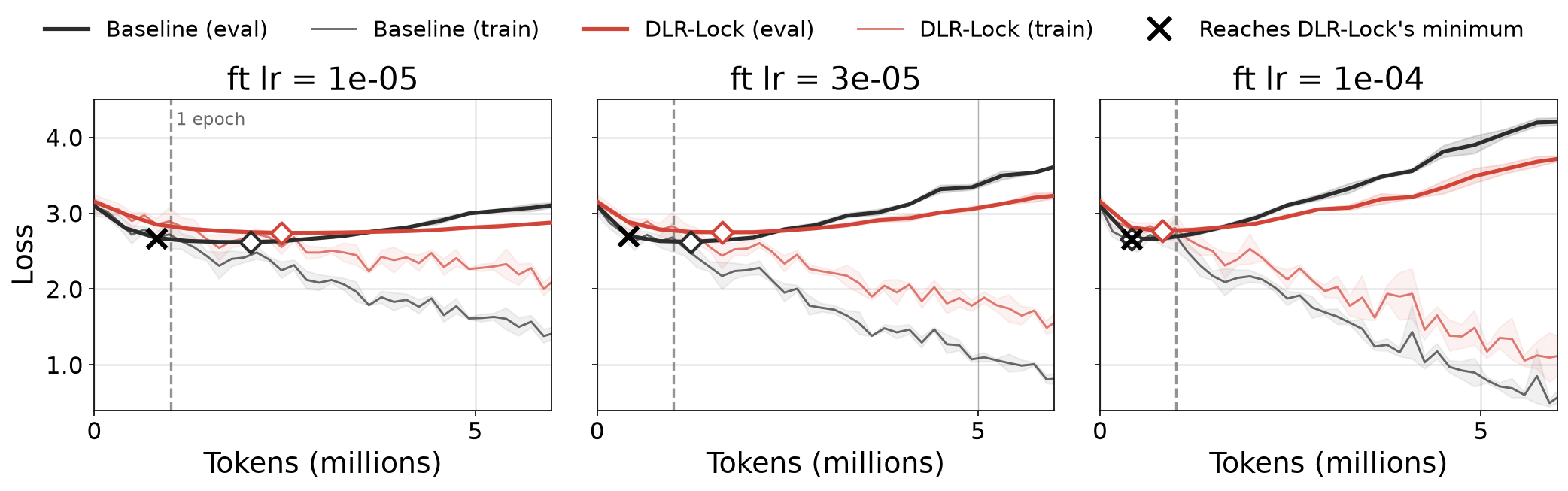}
\caption{$1$M tokens}
\label{fig:ft_fixed1m}
\end{subfigure}\\[6pt]
\begin{subfigure}{0.95\linewidth}
\centering
\includegraphics[width=\linewidth]{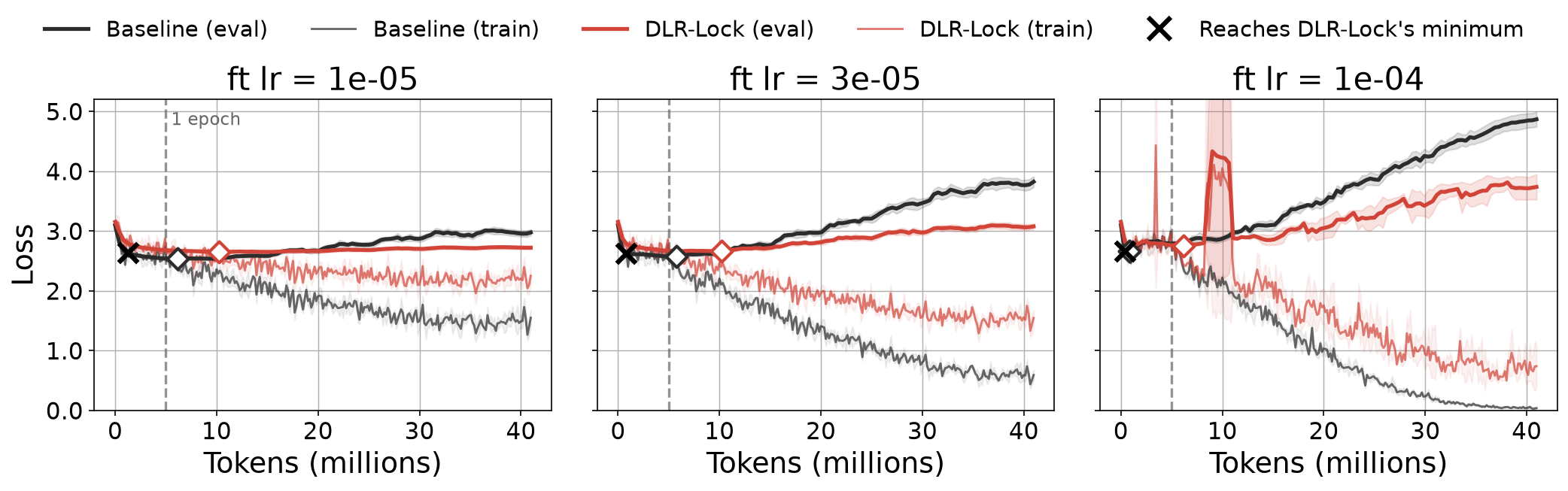}
\caption{$5$M tokens}
\label{fig:ft_fixed5m}
\end{subfigure}
\caption{Fine-tuning convergence against tokens on a fixed slice of WikiText-103, looped for $5000$ steps.
Same models and protocol as Figure~\ref{fig:finetune_speed_wikitext}; only the dataset size is capped.
Diamonds mark the bottom of each eval U-curve; the cross marks where the baseline first reaches the locked model's minimum during its own descent.
Shaded regions: $\pm 1$ standard deviation across $3$ seeds.}
\label{fig:ft_fixed}
\end{figure}

\paragraph{Fine-tuning under a fixed token budget.}
The WikiText-103 experiment in the main text (Figure~\ref{fig:finetune_speed_wikitext}) streams a fresh batch at every step, so neither model ever re-sees a token. A realistic attacker often has only a small corpus and must repeat it. We therefore cap the fine-tuning data at $1$M and $5$M tokens and loop over it for the same $5000$-step budget, holding everything else identical to the main-text setup.
Figure~\ref{fig:ft_fixed} shows the result. Eval loss falls during the first sweep over the data and rises again as the model overfits, a reading confirmed by the train loss, which continues to fall past the eval-loss minimum. The locked model takes much longer in wall-clock to reach a worse optimum than the baseline does: at $\eta{=}10^{-5}$ its best eval loss is higher by about $0.12$ at both caps, and the wall-clock multiplier at the U-bottom is roughly $25{\times}$ with $1$M tokens, growing to $\sim35{\times}$ with $5$M. Read as a cost to reach a fixed target, the locked model's minimum loss costs the baseline $\mathbf{3.0\times}$ fewer tokens with the $1$M cap and $\mathbf{8.3\times}$ fewer with the $5$M cap (the cross), growing to $\mathbf{66\times}$ and $\mathbf{181\times}$ in wall-clock. At the $1$M cap the two reach their optima at comparable token counts, so the additional cost is paid in compute time, not data. The panels plot loss against tokens, so the gap they show is an optimization effect measured at matched data; the recomputation and backward-pass costs enter only through the wall-clock multipliers.

\subsection{Models pretrained from scratch}
\label{app:add_scratch}

In this setting we control the entire pipeline: we pretrain a family of decoder-only Transformers from scratch on Nemotron-CC and then lock them with \method{} on the same corpus.

\textbf{Architecture.} The models follow Qwen3's high-level design: a pre-norm decoder-only Transformer with RMSNorm ($\epsilon{=}10^{-5}$), rotary position embeddings (RoPE~\citep{su2024roformer}, base frequency $5{\times}10^{5}$), SwiGLU feed-forward blocks with SiLU activation, per-head query/key normalization, and the SwiGLU expansion ratio fixed at $d_{\mathrm{ff}}/d = 3$. The main difference is in the attention layer: whereas Qwen3 uses grouped-query attention~\citep{ainslie2023gqa} with fewer key/value heads than query heads, our models use standard multi-head attention with a fixed head dimension of $128$, so the number of heads is $d/128$. We pretrain three scales mirroring the Qwen3 sizes used in the main text: $0.6$B ($d{=}1024$, $28$ blocks, $8$ heads), $1.7$B ($d{=}2048$, $28$ blocks, $16$ heads), and $8$B ($d{=}4096$, $32$ blocks, $32$ heads).

\textbf{Pretraining.} All models are trained autoregressively on Nemotron-CC with token packing at sequence length $4096$, using the tiktoken \texttt{cl100k\_base} tokenizer (vocabulary size $100{,}277$). We use AdamW ($\beta_1{=}0.9$, $\beta_2{=}0.95$, $\epsilon{=}10^{-8}$) with weight decay $0.2$, gradient-norm clipping at $1.0$, a $2000$-step warmup followed by a constant-then-decay learning-rate schedule, and the maximal-update parameterization ($\mu$P)~\citep{yang2021tuning} combined with CompleteP~\citep{dey2026don} at base width $512$ and base depth $32$, with base maximum learning rate $3.79{\times}10^{-3}$ at every scale, so the effective rate on hidden and readout weights scales as $512/d$. Training uses BF16 precision with an auxiliary $z$-loss~\citep{chowdhery2023palm} (weight $10^{-4}$), and each model sees $140$B tokens.

\paragraph{Quality preservation.}
Table~\ref{tab:scratch-quality} reports language-modeling perplexity and zero-shot benchmark accuracy for each scale, before and after locking. Because the locking corpus is identical to the pretraining corpus, this setting removes the distribution-shift penalty that inflates the off-the-shelf gaps of Table~\ref{tab:qwen3-quality}, leaving the fixed distillation budget as the remaining confound. Perplexity rises by $0.4$ to $0.7$ on Nemotron-CC and $1.0$ to $1.9$ on WikiText, and the zero-shot accuracies move by at most $2.9$ points, most of them by under $1.5$.
These models see only $140$B tokens, so MMLU and ARC-Challenge sit near chance for both variants and cannot register a drop; on the benchmarks where they are clearly above chance, ARC-Easy, HellaSwag and PIQA, locking costs at most $1.7$ points. The absence of a large MMLU gap is therefore consistent with the off-the-shelf drop stemming from the pretraining/locking data mismatch, though the near-chance baseline means this setting alone cannot establish it.

\begin{table}[t]
\caption{Quality preservation of our from-scratch models ($0.6$B, $1.7$B, $8$B) before and after \method{}, with Phase-2 distillation at $k{=}1000$ and learning rate $10^{-4}$ at every scale, chosen on quality preservation. ``+DLR-Lock'' is the locked model. At $8$B, memory limits the Phase-2 distillation batch to $2$ rather than $4$, so we run Phase-2 for $200\mathrm{k}$ steps rather than $100\mathrm{k}$; every scale then sees the same $0.82$B distillation tokens. GSM8K~\citep{cobbe2021training} is reported at $2$-shot with $128$-token generation (uniformly across scales) for tractability. Nemo = Nemotron-CC, WT = WikiText-103, ARC-E/C = ARC Easy/Challenge, HSw = HellaSwag, WiGr = WinoGrande.}
\label{tab:scratch-quality}
\centering
\small
\setlength{\tabcolsep}{4pt}
\begin{tabular}{ll cc cccccccc}
\toprule
& & \multicolumn{2}{c}{Perplexity $\downarrow$}
& \multicolumn{8}{c}{Accuracy (\%) $\uparrow$} \\
\cmidrule(lr){3-4} \cmidrule(lr){5-12}
Model & Variant & Nemo & WT & MMLU & ARC-E & ARC-C & HSw & WiGr & BoolQ & PIQA & GSM8K \\
\midrule
\multirow{2}{*}{0.6B} & Pretrained  & 13.47 & 26.29 & 24.8 & 52.1 & 27.2 & 40.9 & 51.3 & 59.1 & 67.5 & 1.5 \\
                      & +DLR-Lock   & 14.20 & 28.14 & 23.1 & 50.5 & 26.6 & 39.3 & 51.5 & 57.3 & 66.2 & 1.5 \\
\midrule
\multirow{2}{*}{1.7B} & Pretrained  &  9.92 & 18.15 & 24.7 & 60.4 & 32.6 & 54.1 & 54.1 & 54.6 & 71.3 & 1.1 \\
                      & +DLR-Lock   & 10.35 & 19.20 & 24.2 & 60.1 & 31.4 & 52.4 & 54.4 & 51.7 & 70.1 & 2.0 \\
\midrule
\multirow{2}{*}{8B}   & Pretrained  & 11.29 & 21.54 & 24.7 & 56.4 & 29.4 & 46.1 & 53.7 & 57.7 & 69.8 & 2.4 \\
                      & +DLR-Lock   & 11.78 & 22.66 & 24.4 & 56.2 & 30.5 & 44.7 & 52.2 & 59.4 & 69.4 & 1.9 \\
\bottomrule
\end{tabular}
\end{table}

\begin{figure}[t]
\centering
\includegraphics[width=0.95\linewidth]{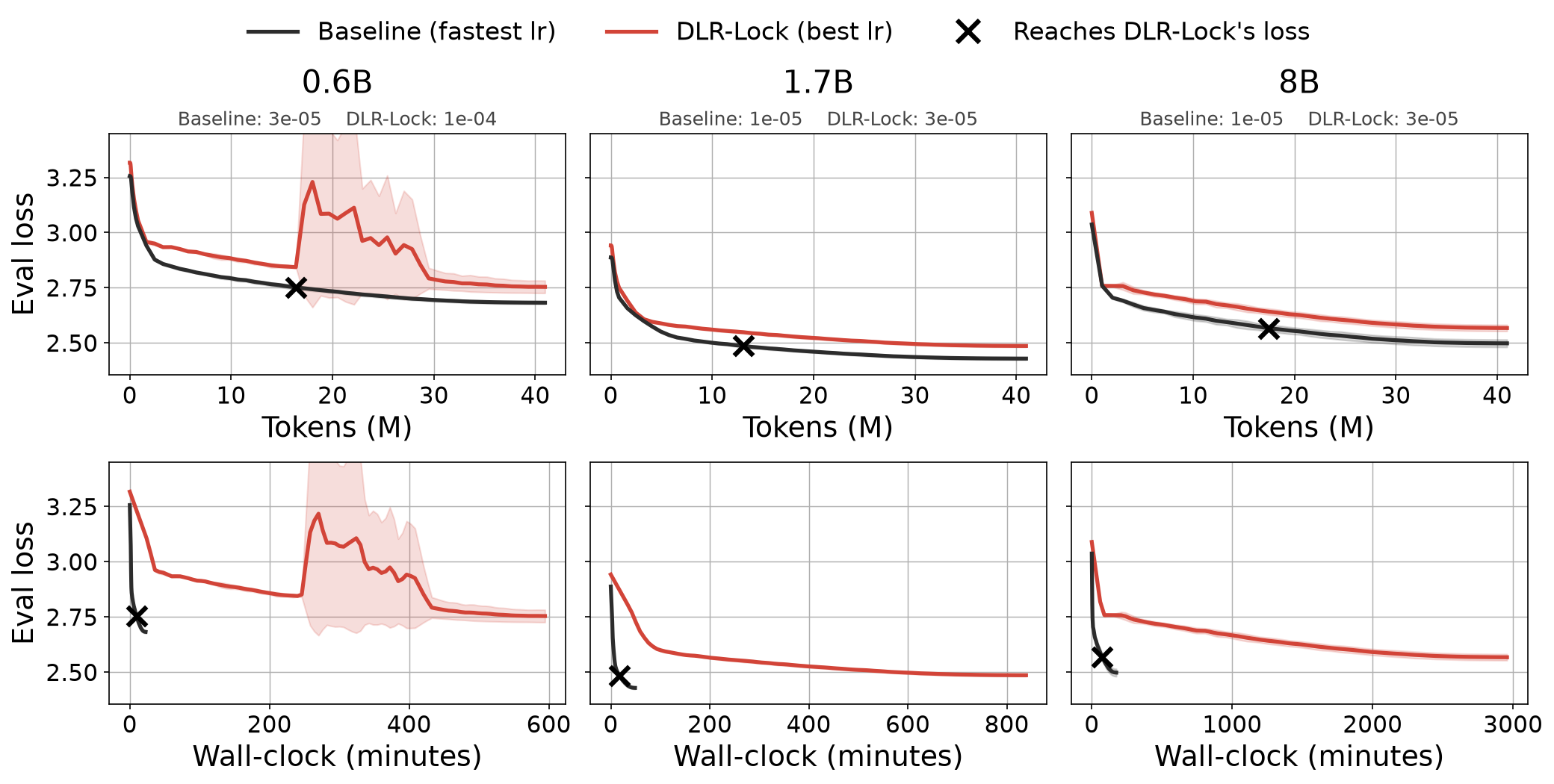}
\caption{Fine-tuning resistance on WikiText-103 for the from-scratch models at $0.6$B, $1.7$B and $8$B.
\textbf{Top:} loss vs.\ number of training tokens; \textbf{Bottom:} loss vs.\ wall-clock time.
The rate shown for each arm is annotated above the panel, and the cross marks where the baseline first reaches the locked model's final loss.
Shaded regions indicate $\pm 1$ standard deviation across 3 seeds.
The transient spike in the $0.6$B locked curve is one seed at $10^{-4}$, which is nonetheless the locked model's lowest-final-loss rate.}
\label{fig:ft_scratch_wikitext}
\end{figure}

\paragraph{WikiText-103 (general text).}
We begin with the main text's fine-tuning target, so these curves can be read against Figure~\ref{fig:finetune_speed_wikitext}.
Both arms run for $41$M tokens at sequence length $2048$ ($5000$ steps at batch $4$ for $0.6$B and $1.7$B, $20{,}000$ steps at batch $1$ for the $8$B), sweeping $\eta \in \{3{\times}10^{-6}, 10^{-5}, 3{\times}10^{-5}, 10^{-4}\}$ over three seeds.
The $8$B model does not fit on one device in either arm, so we fine-tune it on 2~GPUs with the model split across devices and report wall-clock time under identical hardware for a fair comparison.
The locked model is shown at the rate that gives it the lowest final loss, the baseline at whichever of its own rates reaches that loss in the fewest tokens.
The arms start within $0.07$ at all three scales, including $8$B, so the ratios below are not inherited from a pre-existing quality gap.
Figure~\ref{fig:ft_scratch_wikitext} shows the locked model converging more slowly and plateauing higher at every scale.
Read as a cost to reach a fixed target, the locked model's best loss costs the baseline $\mathbf{2.5\times}$, $\mathbf{3.1\times}$ and $\mathbf{2.4\times}$ fewer tokens at $0.6$B, $1.7$B and $8$B on its cheapest route there; the locked model alone recomputes activations (Remark~\ref{rem:checkpointing}), so the FLOPs floors are $\mathbf{3.3\times}$, $\mathbf{4.2\times}$ and $\mathbf{3.1\times}$.
In eager wall-clock the pair differs by $\mathbf{62\times}$, $\mathbf{48\times}$ and $\mathbf{40\times}$, which carries the forward slowdown too and so bounds the ratio from above.
The baseline's own best loss the locked model never reaches at any scale.
These token ratios show no trend with scale. They are below the $4.8\times$ and $10.0\times$ measured off the shelf at $0.6$B and $1.7$B, but these models are pretrained on only $140$B tokens and are correspondingly weaker at equal size (Table~\ref{tab:scratch-quality} against Table~\ref{tab:qwen3-quality}), so the ratios of the two settings are not directly comparable.

\begin{figure}[t]
\centering
\includegraphics[width=0.95\linewidth]{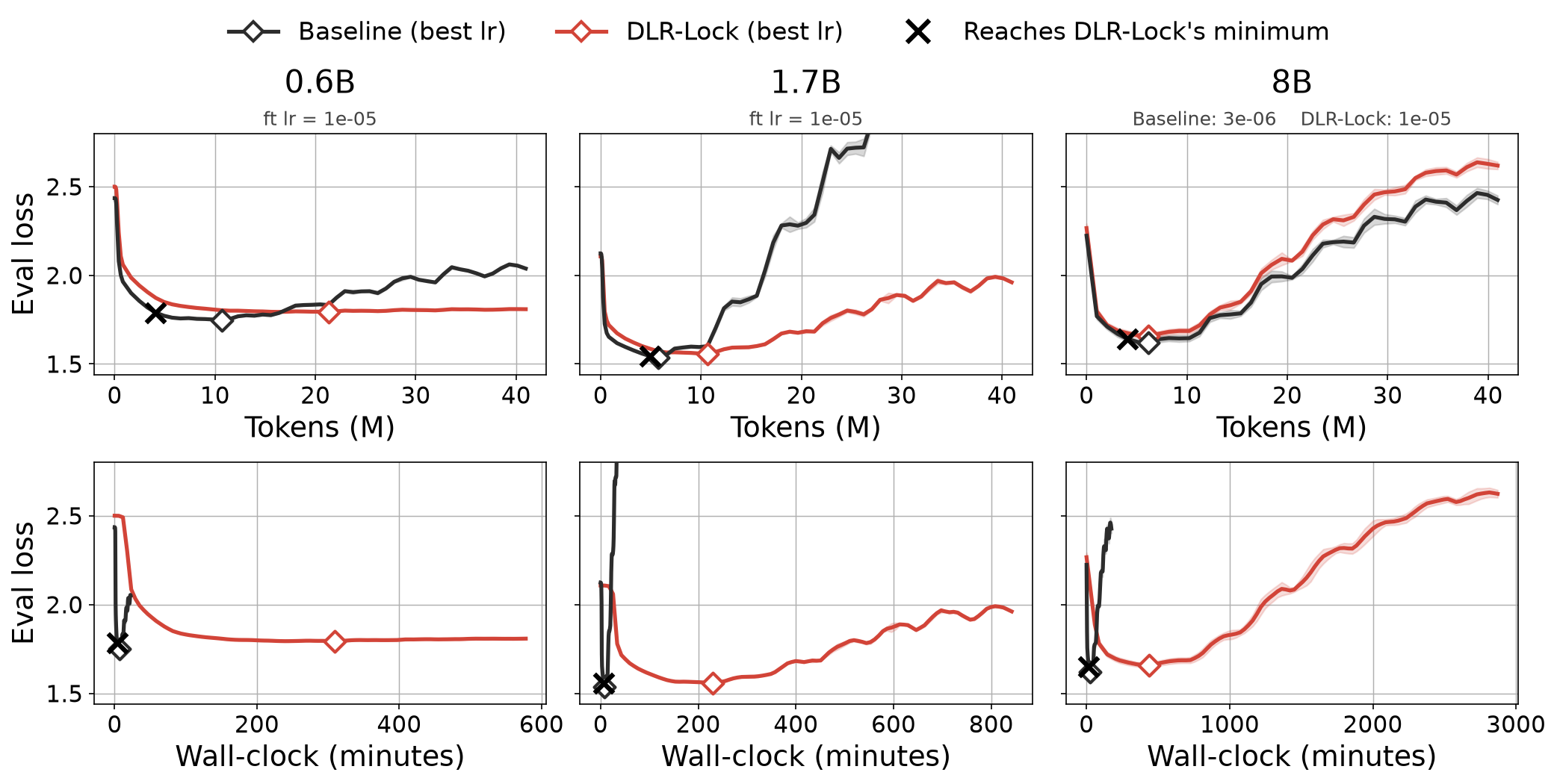}
\caption{Fine-tuning resistance on Alpaca for the from-scratch models at $0.6$B, $1.7$B and $8$B.
\textbf{Top:} loss vs.\ number of training tokens; \textbf{Bottom:} loss vs.\ wall-clock time.
Both curves are U-shaped: this corpus is small enough that the baseline and the locked model alike eventually overfit it, so each arm is compared at its own minimum, the diamonds, rather than at the end of the budget.
Each arm is shown at the rate that gives it the lowest loss it attains, annotated above the panel; the cross marks where the baseline first reaches the locked model's minimum during its own descent.
Shaded regions indicate $\pm 1$ standard deviation across 3 seeds.
}
\label{fig:finetune_speed_scratch_alpaca}
\end{figure}

\paragraph{Alpaca (instruction following).}
For a task-formatted counterpart to the plain-text target above, we fine-tune on Alpaca~\citep{alpaca}, a 52k-prompt instruction dataset, with the prompt masked from the loss as detailed in Appendix~\ref{app:ft_details}.
Both arms run under the same protocol as WikiText-103 above: $41$M tokens, the same sweep, three seeds and identical hardware.
Both also cycle over this small corpus and eventually overfit, so the curves of Figure~\ref{fig:finetune_speed_scratch_alpaca} are U-shaped and we compare each arm at its own best rate and its own minimum loss, where an attacker monitoring eval loss would stop.
The locked model's minimum sits $0.022$ to $0.045$ above the baseline's, a small but persistent quality ceiling.
Read as a cost to reach a fixed target, that minimum costs the baseline $\mathbf{5.2\times}$, $\mathbf{2.2\times}$ and $\mathbf{1.5\times}$ fewer tokens at $0.6$B, $1.7$B and $8$B en route to its own, lower minimum (the cross), shrinking with scale.
At each arm's own minimum the token cost is $2.0\times$, $1.9\times$ and $1.0\times$ and the wall-clock cost $\mathbf{46.9\times}$, $\mathbf{27.2\times}$ and $\mathbf{16.5\times}$: at $8$B the two arms bottom out after the same number of tokens, so what resists there is the quality ceiling and the wall-clock rather than the token count.

\begin{figure}[t]
\centering
\includegraphics[width=0.95\linewidth]{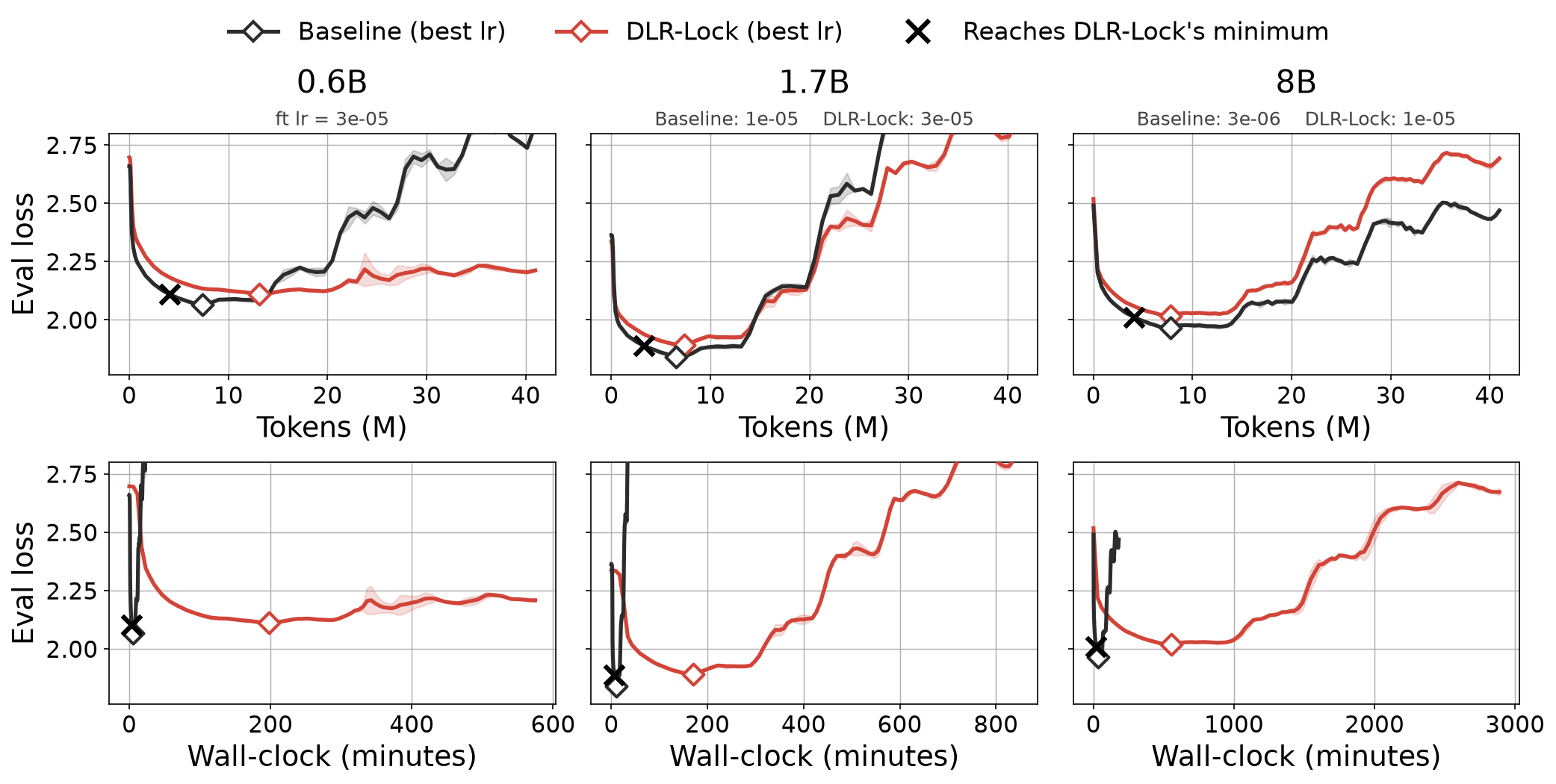}
\caption{Fine-tuning resistance on BeaverTails for the from-scratch models at $0.6$B, $1.7$B and $8$B.
\textbf{Top:} loss vs.\ number of training tokens; \textbf{Bottom:} loss vs.\ wall-clock time.
Both curves are U-shaped: repeated cycling of this small corpus eventually overfits, so ``best'' here is each arm's own minimum rather than its final value.
Each arm is shown at the rate that gives it the lowest loss it attains, annotated above the panel; the cross marks where the baseline first reaches the locked model's minimum during its own descent.
Shaded regions indicate $\pm 1$ standard deviation across 3 seeds.
}
\label{fig:finetune_speed_scratch_beavertails}
\end{figure}

\paragraph{BeaverTails (harmful instruction-following).}
As a safety-relevant complement to the instruction-following domain above, we also fine-tune on BeaverTails~\citep{ji2023beavertails}, a corpus of prompts paired with responses that a majority of crowd annotators flagged unsafe; formatting and masking follow the same convention as Alpaca (Appendix~\ref{app:ft_details}).
After deduplication and the majority-unsafe filter this corpus is small, so both arms cycle over it many times to reach the same $41$M-token budget as the other domains, and both eventually overfit: Figure~\ref{fig:finetune_speed_scratch_beavertails} shows eval loss falling, reaching a minimum, then rising again at every scale, for the baseline and the locked model alike.
We therefore compare each arm at its own best rate and its own minimum loss (the diamonds) rather than a fixed target, since an attacker monitoring eval loss would stop at that point regardless of arm.
The locked model's minimum sits $0.046$--$0.053$ above the baseline's at every scale, a small but persistent quality ceiling on this specific objective.
Read as a cost to reach a fixed target, the locked model's minimum loss costs the baseline $\mathbf{3.2\times}$, $\mathbf{2.2\times}$ and $\mathbf{1.9\times}$ fewer tokens at $0.6$B, $1.7$B and $8$B en route to its own, lower minimum (the cross), shrinking with scale as on Alpaca.
Wall-clock resistance at each arm's own minimum remains large, at $\mathbf{39.4\times}$, $\mathbf{17.2\times}$ and $\mathbf{16.9\times}$, driven by the same per-step overhead as every other domain rather than anything specific to this one.
The same mechanism resists adaptation here as on every other domain in this section. This is the practical content of our domain-agnostic claim: since \method{} is not designed around any particular target domain, resistance to harmful fine-tuning follows as a special case of the same, broader guarantee.
Eval loss on held-out unsafe responses is the same adaptation-cost lens used throughout this section, and we report it here for consistency. It does not, however, directly measure harmfulness: it does not indicate whether the fine-tuned model's own generations are actually more toxic. That would require scoring sampled completions, e.g.\ with an LLM judge, which we leave to future work.

\begin{figure}[!t]
\centering
\begin{subfigure}{0.86\linewidth}
\centering
\includegraphics[width=\linewidth]{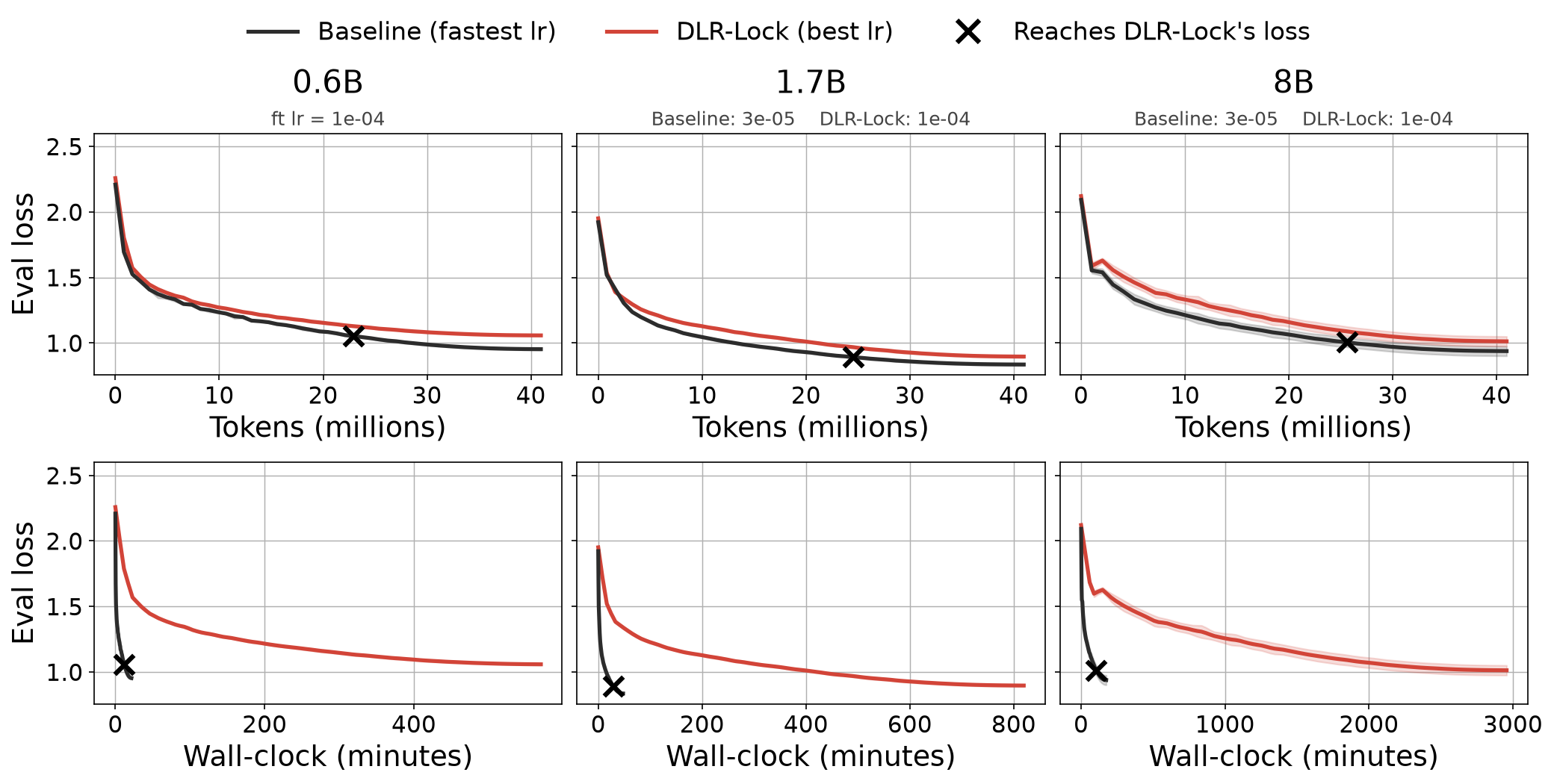}
\caption{Streaming}
\label{fig:ft_scratch_omr_stream}
\end{subfigure}\\[6pt]
\begin{subfigure}{0.86\linewidth}
\centering
\includegraphics[width=\linewidth]{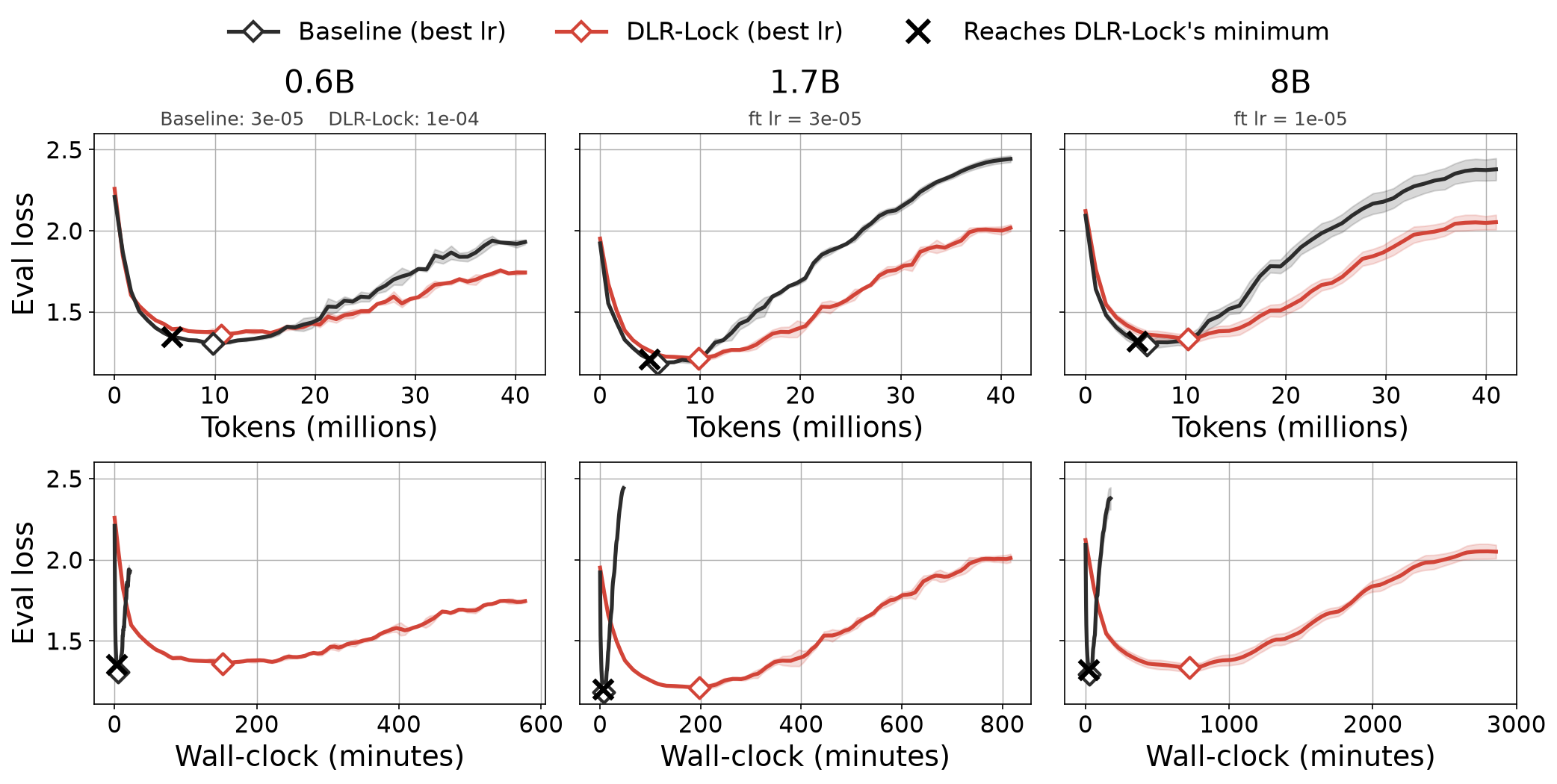}
\caption{$5$M-token cap}
\label{fig:ft_scratch_omr_fixed}
\end{subfigure}
\caption{Fine-tuning resistance on OpenMathReasoning for the from-scratch models at $0.6$B, $1.7$B and $8$B, streaming fresh tokens in~(a) and cycling a fixed $5$M-token subset for the same budget in~(b).
In each, \textbf{top:} loss vs.\ tokens; \textbf{bottom:} loss vs.\ wall-clock time.
Rates are annotated above each panel: in~(a) the locked model's lowest-final-loss rate and the baseline's fastest rate to that loss; in~(b), where both arms overfit the subset, each arm's lowest-minimum rate, with diamonds at those minima.
The cross marks where the baseline first reaches the locked model's loss.
Shaded regions indicate $\pm 1$ standard deviation across 3 seeds.
}
\label{fig:ft_scratch_omr}
\end{figure}

\paragraph{OpenMathReasoning (mathematical reasoning).}
We next move away from natural prose to mathematical reasoning traces, using a decontaminated and compliance-filtered copy of OpenMathReasoning~\citep{moshkov2025aimo}. The corpus is consumed as raw text under the same next-token objective as WikiText-103, with no prompt template, and we run two settings: streaming fresh tokens for the section's $41$M-token budget, and spending that budget cycling a fixed $5$M-token subset, the data-limited regime of Appendix~\ref{app:add_pretrained}.
Streaming, in Figure~\ref{fig:ft_scratch_omr_stream}, both arms descend monotonically and the locked model stays above the baseline throughout, ending $0.105$, $0.061$ and $0.075$ higher at $0.6$B, $1.7$B and $8$B.
These gaps sit above the baseline's own spread across seeds, which stays under $0.04$ at every scale, so the ceiling is resolved here rather than inferred.
Read as a cost to reach a fixed target, the locked model's final loss costs the baseline $\mathbf{1.8\times}$, $\mathbf{1.7\times}$ and $\mathbf{1.6\times}$ fewer tokens, and $\mathbf{47.7\times}$, $\mathbf{28.8\times}$ and $\mathbf{28.0\times}$ less wall-clock.
Capping the data at $5$M tokens, in Figure~\ref{fig:ft_scratch_omr_fixed}, turns both curves into the U-shape of the other fixed-corpus domains, and the picture does not change: the minima sit $0.050$, $0.032$ and $0.040$ apart, the locked model needs $\mathbf{1.9\times}$, $\mathbf{2.0\times}$ and $\mathbf{2.0\times}$ the tokens the baseline needs to reach its minimum, and $\mathbf{29.0\times}$, $\mathbf{30.0\times}$ and $\mathbf{27.4\times}$ the wall-clock at each arm's own minimum.
Restricting the attacker's data therefore does not weaken the defense on this domain; if anything the token ratio grows.

\begin{figure}[!t]
\centering
\begin{subfigure}{0.86\linewidth}
\centering
\includegraphics[width=\linewidth]{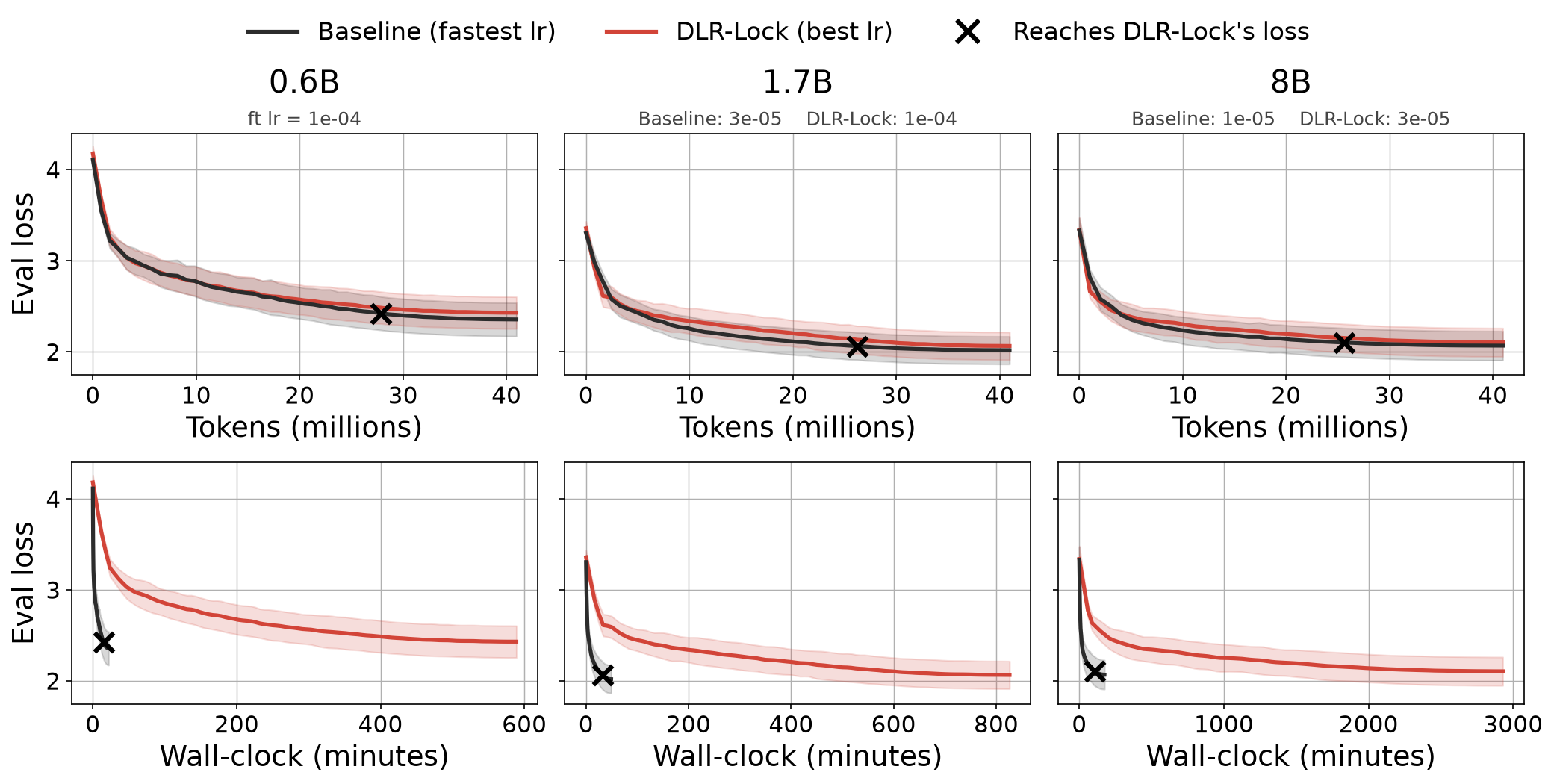}
\caption{Streaming}
\label{fig:ft_scratch_starcoder_stream}
\end{subfigure}\\[6pt]
\begin{subfigure}{0.86\linewidth}
\centering
\includegraphics[width=\linewidth]{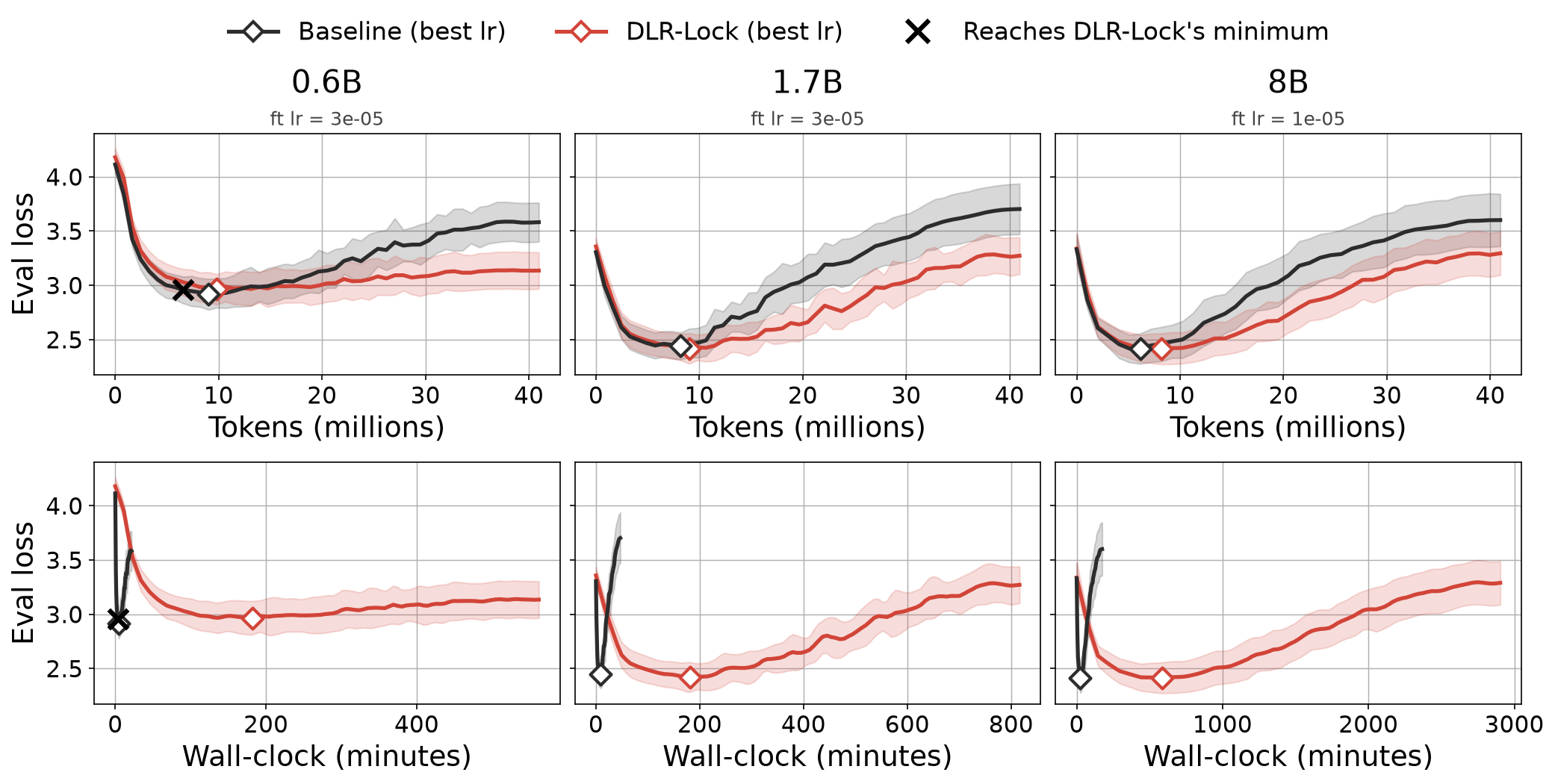}
\caption{$5$M-token cap}
\label{fig:ft_scratch_starcoder_fixed}
\end{subfigure}
\caption{Fine-tuning resistance on StarCoderData, with the layout and conventions of Figure~\ref{fig:ft_scratch_omr}: streaming in~(a), a fixed $5$M-token subset in~(b).
Shaded regions indicate $\pm 1$ standard deviation across 3 seeds, and they are wide on this domain: the baseline's spread across seeds exceeds every arm-to-arm gap we measure here.
Only the $0.6$B panel of~(b) carries a cross, since at $1.7$B and $8$B the baseline never falls to the locked model's minimum, which is in fact marginally lower than its own.
Past that minimum the locked model's loss rises more slowly than the baseline's, the same slower adaptation seen from the other side.
}
\label{fig:ft_scratch_starcoder}
\end{figure}

\paragraph{StarCoderData (source code).}
Our last loss-based domain is source code, using StarCoderData~\citep{li2023starcoder}, which also carries issue and pull-request discussions alongside code in many languages, again as raw text under the two settings above.
This is the one domain where the quality ceiling disappears.
Streaming, in Figure~\ref{fig:ft_scratch_starcoder_stream}, the arms end $0.074$, $0.047$ and $0.037$ apart, but the baseline's own spread across seeds is $0.15$ to $0.18$, several times those gaps, so we claim no ceiling on code.
Under the $5$M-token cap, in Figure~\ref{fig:ft_scratch_starcoder_fixed}, the two minima coincide within that noise, $2.92$ against $2.96$ at $0.6$B and within $0.03$ at the two larger scales, where the locked model's minimum is in fact the lower of the two.
What survives is the cost of getting there.
Streaming, the locked model's final loss costs the baseline $\mathbf{1.5\times}$, $\mathbf{1.6\times}$ and $\mathbf{1.6\times}$ fewer tokens and $\mathbf{39.7\times}$, $\mathbf{26.1\times}$ and $\mathbf{26.6\times}$ less wall-clock; under the cap, where the two arms bottom out within $1.1$ to $1.3\times$ the same number of tokens, the wall-clock ratios at each arm's own minimum are $\mathbf{38.0\times}$, $\mathbf{20.2\times}$ and $\mathbf{23.8\times}$.
Because the two arms reach the same loss here, these are the section's cleanest iso-quality costs: no part of the ratio can be attributed to the locked model settling for a worse optimum.

\begin{figure}[t]
\centering
\includegraphics[width=0.95\linewidth]{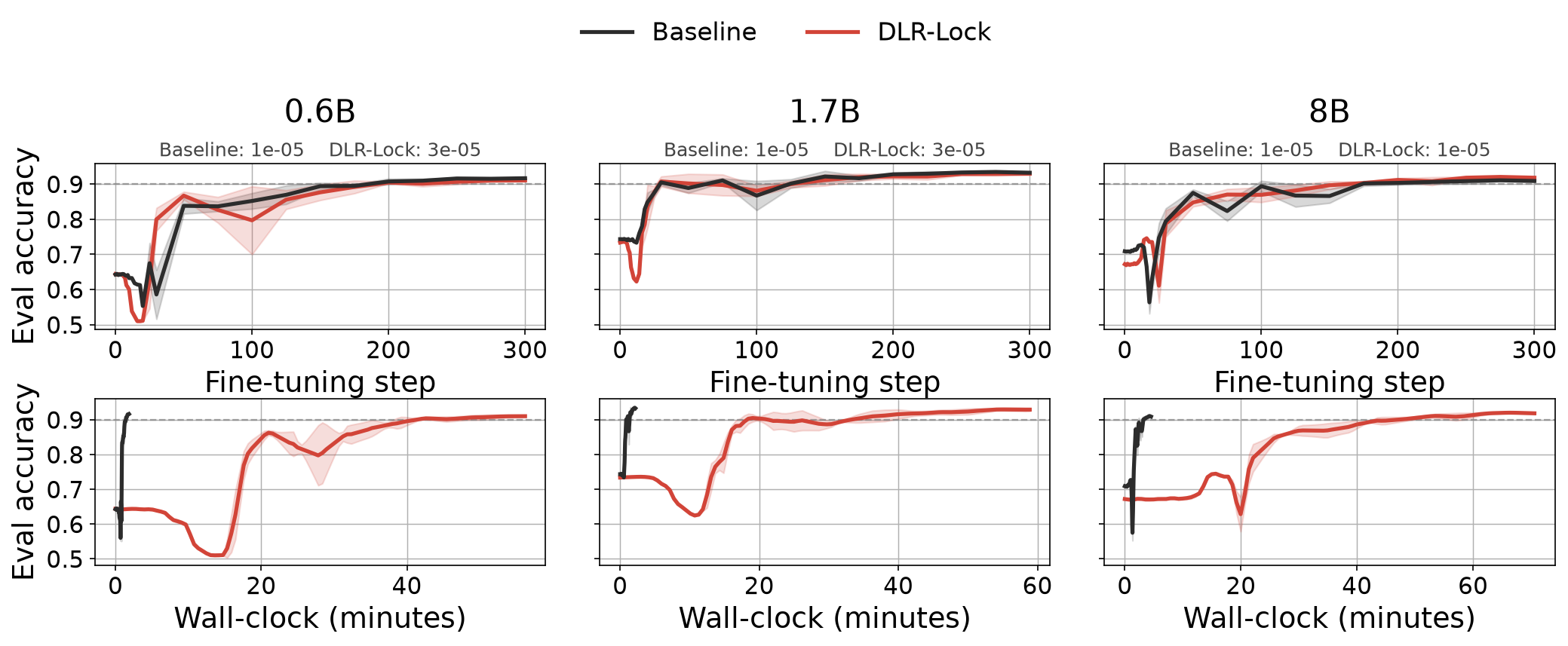}
\caption{Fine-tuning resistance on SST-2 for the from-scratch models at $0.6$B, $1.7$B and $8$B.
\textbf{Top:} eval accuracy vs.\ fine-tuning step; \textbf{Bottom:} eval accuracy vs.\ wall-clock time.
The dashed line marks $90\%$ accuracy, and the rate at which each arm reaches it fastest is annotated above the panel.
Reaching it takes $185$ steps for the baseline against $191$ for the locked model at $0.6$B, $29$ against $29$ at $1.7$B, and $174$ against $164$ at $8$B.
Shaded regions indicate $\pm 1$ standard deviation across 3 seeds.
}
\label{fig:finetune_speed_scratch_sst2}
\end{figure}

\paragraph{SST-2 (sentiment classification).}
We close with the sparsest supervision in this section: SST-2~\citep{socher2013recursive}, a binary task in which each example contributes roughly two supervised tokens (Appendix~\ref{app:ft_details}).
Both arms are fine-tuned with batch size $32$, sequence length $512$, and prompt-masked cross-entropy on the label tokens.
For each model size we show each arm at its own best rate: the learning rate from $\{10^{-5}, 3\times10^{-5}, 10^{-4}\}$ that reaches $90\%$ accuracy in the fewest optimizer steps, averaged over three seeds.
Figure~\ref{fig:finetune_speed_scratch_sst2} shows the two arms reaching $90\%$ accuracy within ten optimizer steps of each other at every scale, and at $8$B the locked model gets there first.
Only the per-step cost survives, at $\mathbf{29.1\times}$, $\mathbf{20.8\times}$ and $\mathbf{14.7\times}$ longer in wall-clock to the same threshold.
This is the domain where we would expect the least resistance: binary sentiment classification is plausibly close to a readout problem, since the pretrained model already computes features that largely separate the two classes, so adapting little more than the output layer suffices and little pressure falls on the feed-forward blocks that \method{} replaces.

\FloatBarrier
\subsection{Prompt templates and verbalizers}
\label{app:ft_details}

This subsection collects the exact prompt templates and verbalizer sets used by the
instruction-tuning (Figure~\ref{fig:finetune_speed_scratch_alpaca}), harmful-instruction
(Figure~\ref{fig:finetune_speed_scratch_beavertails}), and classification
(Figure~\ref{fig:finetune_speed_scratch_sst2}) experiments of Appendix~\ref{app:add_scratch}.
All experiments use the same templates across the baseline and DLR-Lock arms.

\paragraph{Alpaca (instruction tuning).}
We use the original Stanford Alpaca template~\citep{alpaca} in two variants depending on whether the example carries an optional \texttt{input} field.
For examples with an \texttt{input} field:
\begin{promptbox}
Below is an instruction that describes a task, paired with an input that provides
further context. Write a response that appropriately completes the request.

### Instruction:
{instruction}

### Input:
{input}

### Response:
{response}
\end{promptbox}
For examples without an \texttt{input} field (the majority):
\begin{promptbox}
Below is an instruction that describes a task. Write a response that
appropriately completes the request.

### Instruction:
{instruction}

### Response:
{response}
\end{promptbox}
The cross-entropy loss is masked on all template + instruction (+ input) tokens; only the \texttt{\{response\}} tokens (followed by a single end-of-text marker) contribute to the gradient.
Per-example supervised token counts therefore range from a few dozen to a few hundred.
Examples are packed into \texttt{seq\_len=2048} chunks separated by end-of-text tokens; the prompt mask is preserved across example boundaries within a packed chunk.
Reported eval loss is the same prompt-masked cross-entropy on a held-out 5\% split of Alpaca, evaluated on response tokens only.

\paragraph{BeaverTails (harmful instruction-following).}
We use BeaverTails~\citep{ji2023beavertails}, a corpus of (prompt, response) pairs each annotated by multiple crowdworkers for safety.
We aggregate annotations per unique pair and keep only those where a majority of annotators flagged the response \texttt{unsafe}, then format each surviving pair with the same no-\texttt{input} Alpaca template above, with \texttt{\{instruction\}} set to the prompt and \texttt{\{response\}} to the flagged-unsafe response.
Loss masking and packing follow the same convention as Alpaca: the cross-entropy loss is masked on all template and prompt tokens, only response tokens (followed by an end-of-text marker) contribute to the gradient, and examples are packed into \texttt{seq\_len=2048} chunks with the prompt mask preserved across example boundaries.
Train and eval pairs come from the dataset's own disjoint \texttt{330k\_train} and \texttt{330k\_test} splits, respectively; reported eval loss is the same prompt-masked cross-entropy on \texttt{330k\_test}, evaluated on response tokens only.

\Needspace{10\baselineskip}
\paragraph{SST-2 (sentiment classification, 2-way).}
Each example renders as
\begin{promptbox}
Sentence: {sentence}
Sentiment:
\end{promptbox}
Predictions are obtained as $\arg\max$ over the first BPE token IDs of the verbalizer set $\{$\texttt{" positive"}, \texttt{" negative"}$\}$ (note the leading space) computed at the token position immediately following the prompt.
The cross-entropy loss is computed only on the label tokens (the prompt is masked); this restricts the supervised signal to ${\sim}2$ tokens per example (the label token plus an end-of-text marker).

\section{\method{} ablations}
\label{app:ablations}

This appendix ablates two axes of the defender's architecture knob: Appendix~\ref{app:ablation_arch} covers the choice of block form and nonlinearity placement, and Appendix~\ref{app:ablation_rank} covers the rank--depth trade-off within that block form.

\subsection{Architecture ablation}
\label{app:ablation_arch}

Much of the \arch{}'s shape is forced rather than chosen. Depth is what raises activation memory, while utility preservation caps the parameter count and the inference cost; together these force a low-rank bottleneck and fix the depth from the budget (Section~\ref{sec:architecture}). What is left open is the nonlinearity, and it does double duty: the activation is what keeps a depth-$L$ stack from collapsing to a single linear map, while a nonlinearity acting on the $d$-dimensional stream, the RMSNorm in our block, is what obliges the attacker to store $\Theta(d)$ per layer rather than the $\Theta(r)$ the activation alone would need. The remaining choices, the form of the residual block and the exact activation and normalization, do not follow from that reasoning, and we settled them empirically.

\paragraph{Synthetic sweep.}
The design was originally fixed on a synthetic sweep over the block form, residual scaling, normalization, initialization scale and learning rate, distilling a randomly initialized SwiGLU teacher at $d{=}128$.
At rank $r{=}32$, depth $192$ and matched settings our block reached $0.355$ relative MSE against $0.395$ for $\vh + \sigma(\mU\mV\,\mathrm{RMSNorm}(\vh))$, which moves the nonlinearity outside the low-rank map, and giving each its own best setting moved both by under $2.5\%$ without changing the ordering.
The placement of the nonlinearity therefore dominated, while residual scaling and initialization scale were second-order at $1$--$4\%$, and a low-rank adapter baseline with no depth was uniformly worse at $0.53$--$0.72$.
This sweep is the only evidence we have for the scaling and initialization axes, which the real-data study below holds fixed.

\paragraph{Real Phase-1 setting.}
We therefore repeat the comparison in the actual module-wise regression setting, distilling one real Qwen3-0.6B SwiGLU residual sub-block into each variant at a matched parameter budget, which at $r{=}32$ gives depth $L{=}141$.
Each variant trains for $50\text{k}$ steps, a tenth of the $500\text{k}$ Phase~1 uses in the pipeline, on Nemotron-CC activations at three learning rates bracketing the value used for module-wise regression, $\{3{\times}10^{-4}, 10^{-3}, 3{\times}10^{-3}\}$, and we report the best of the three.
At this budget that best rate is usually the top of the bracket, so a wider sweep would likely lower the absolute values somewhat; what the comparison rests on is the ranking, and the two penalties we draw conclusions from below hold at all three rates and all three layers.
All six variants for a given layer share one teacher forward and the same activation batches, which makes the comparison paired.
We run this at an early, a middle and a late layer, so that the ordering is not read off one depth.
A low-rank SwiGLU block, $\vh + \mU\bigl((\mV\tilde{\vh}) \odot \sigma(\mW\tilde{\vh})\bigr)$ with $\tilde{\vh} = \mathrm{RMSNorm}(\vh)$, is not among the variants: its third matrix per layer buys only two thirds the depth at a fixed parameter budget, $L{=}95$ against $141$ at $r{=}32$, and depth is what carries the locking, so it is ruled out by the budget rather than by regression fidelity.
Table~\ref{tab:arch_ablation} reports the result.

\begin{table}[t]
\caption{Architecture ablation in the module-wise regression setting, reporting
relative MSE $\times 10^{3}$ when regressing one Qwen3-0.6B SwiGLU residual
sub-block for $50\text{k}$ steps at the best of three learning rates, lower is
better.
Each column is one layer, and the reference is the locking configuration at
$r{=}32$, $L{=}141$.
Depth follows from the parameter budget, so the norm-free variant, which no longer
spends parameters on $L$ scale vectors, runs two layers deeper at $L{=}143$.
The residual-removed variant also drops ReZero, since a zero-initialized $\alpha_i$
with no residual path emits exactly zero at zero gradient, so no single-variable
version of it exists.}
\label{tab:arch_ablation}
\centering
\small
\begin{tabular}{@{}lccc@{}}
\toprule
Variant & Layer 2 & Layer 14 & Layer 27 \\
\midrule
Ours, $\vh + \mU\sigma(\mV\,\mathrm{RMSNorm}(\vh))$ & $0.50$ & $0.47$ & $5.91$ \\
\arrayrulecolor{gray!80}\midrule\arrayrulecolor{black}
Nonlinearity outside the low-rank map, $\vh + \sigma(\mU\mV\,\mathrm{RMSNorm}(\vh))$ & $0.57$ & $1.24$ & $9.93$ \\
RMSNorm removed & $2.18$ & $0.52$ & $5.20$ \\
One RMSNorm shared across all $L$ layers & $0.63$ & $0.50$ & $5.48$ \\
Identity activation, $\sigma = \mathrm{id}$ & $0.69$ & $1.53$ & $12.85$ \\
Residual connection removed & $996$ & $945$ & $450$ \\
\bottomrule
\end{tabular}
\end{table}

\paragraph{Findings.}
The residual connection is the one load-bearing choice: removing it is the only variant that fails outright, at $76$--$2000\times$ the reference.
Moving the nonlinearity outside the low-rank map costs $1.13$--$2.65\times$ and replacing it with the identity costs $1.39$--$3.28\times$, in both cases with the widest margin at the middle layer and the narrowest at layer $2$.
Removing RMSNorm is the one axis fidelity does not settle: it costs $4.4\times$ at layer $2$ and $1.11\times$ at layer $14$, but at layer $27$ it edges ahead of the reference.
Two limits are worth stating. A single distillation seed per cell means differences under roughly $10\%$ should not be read as real, which is the range the RMSNorm variants fall in. And this measures Phase-1 fidelity on one module, not end-to-end quality after Phase~2 or locking strength after fine-tuning, so it settles the block form rather than the whole design.

\subsection{Rank--depth ladder: locking strength vs.\ forward latency}
\label{app:ablation_rank}

The \arch{} replacement is parameterized by a single shape knob, the low-rank dimension $r$; the depth $L$ is then chosen so that the \arch{} matches the original MLP's parameter count (Section~\ref{sec:method}).
A smaller $r$ buys a deeper stack at the same parameter budget, which should make the locked model harder to fine-tune (more residual layers to backpropagate through, more activation memory) but also slower at inference: each of the $L$ layers depends on the previous one's output, so the forward pass is an irreducible chain of $L$ sequential steps. A well-optimized kernel may reduce this overhead to a negligible level, but Section~\ref{sec:inference} shows it still grows with $L$ in practice, even with compilation and a custom kernel tried in this paper.
Note that matching the original parameter count is a choice, not a requirement: relaxing it opens a larger design space, trading network expressivity, locking strength, and inference overhead against each other, along with weight memory and the defender's locking compute budget; we study only the parameter-matched slice here.
This ablation traces that frontier empirically across the rank ladder.

\paragraph{Setup.}
We use the from-scratch models of Appendix~\ref{app:add_scratch}: the locking corpus and the pretraining corpus are identical by construction, so any change in locking strength as $r$ varies is attributable to the architectural knob rather than to a shifting domain gap.
Concretely, we lock the from-scratch 1.7B model at seven rank--depth trade-off points, $r \in \{64, 128, 256, 512, 1024, 2048, 4096\}$ with depths $L \in \{142, 71, 35, 17, 8, 4, 2\}$, plus an $r{=}9216$, $L{=}1$ no-depth control: a single SiLU-activated layer parameter-matched to SwiGLU, which needs hidden width $\tfrac{3}{2} d_{\mathrm{ff}}$ since it has two matrices where SwiGLU has three.
Module-wise regression and end-to-end logit distillation follow the same protocol as Section~\ref{sec:method}, using the same distillation configuration as the fine-tuning experiments of Appendix~\ref{app:add_scratch} ($\mathrm{lr}{=}10^{-4}$, $k{=}1000$) at every rung of the ladder.
Locking strength is measured by the WikiText-103 fine-tuning curves (same protocol as Figure~\ref{fig:finetune_speed_wikitext}), and forward latency by the single-MLP overhead benchmark of Appendix~\ref{app:latency}.

\paragraph{Quality preservation across the ladder.}
Table~\ref{tab:scratch-rank-17B} shows the quality of the locked 1.7B model at every rung of the ladder.
Within this locking budget, capabilities are flat across all seven rank--depth points and the no-depth control, within noise of the pretrained baseline and of each other: the defender can move along $r$ without any quality cost. A larger locking budget might create a difference among them in performance, but it is already small enough that this does not matter here.
This is the precondition for the trade-off question to be well-posed: since quality is essentially flat along $r$, any difference in fine-tuning resistance or inference latency below is attributable to the rank--depth knob itself rather than to a confounding capabilities gap.

\begin{table}[t]
\caption{Quality preservation across the rank--depth ladder for our 1.7B model, pretrained from scratch on Nemotron-CC, after model locking, averaged over 3 distillation seeds.
All DLR-Locked rows use the same logit distillation configuration
($\mathrm{lr}{=}10^{-4}$, $k{=}1000$).
$r$ = rank, $L$ = depth: higher $L$ means stronger locking but
more inference overhead.
$^\dagger$The $r{=}9216$, $L{=}1$ row is the no-depth control: a single
SiLU-activated layer, parameter-matched to SwiGLU.
WT = WikiText-103, out-of-distribution for this model;
ARC-E/C = ARC Easy/Challenge;
HSw = HellaSwag; WiGr = WinoGrande.}
\label{tab:scratch-rank-17B}
\centering
\small
\setlength{\tabcolsep}{3.5pt}
\begin{tabular}{@{}lrr c cccccccc@{}}
\toprule
& & & Perplexity $\downarrow$
& \multicolumn{8}{c}{Accuracy (\%) $\uparrow$} \\
\cmidrule(lr){4-4} \cmidrule(lr){5-12}
 & $r$ & $L$ & WT & MMLU & ARC-E & ARC-C & HSw & WiGr & BoolQ & PIQA & GSM8K \\
\midrule
Pretrained & --- & --- & 18.2 & 24.7 & 60.4 & 32.6 & 54.1 & 54.1 & 54.6 & 71.3 & 2.7 \\
\arrayrulecolor{gray!80}\midrule\arrayrulecolor{black}
\multirow{8}{*}{DLR-Locked}
  & 64   & 142 & 19.2 & 24.6 & 59.8 & 31.6 & 52.4 & 53.7 & 52.1 & 70.0 & 2.1 \\
  & 128  &  71 & 19.2 & 25.0 & 59.8 & 31.1 & 52.5 & 55.6 & 52.2 & 70.0 & 2.2 \\
  & 256  &  35 & 19.2 & 24.5 & 60.1 & 31.6 & 52.3 & 54.3 & 52.7 & 70.6 & 1.8 \\
  & 512  &  17 & 19.2 & 25.4 & 60.0 & 31.2 & 52.3 & 54.8 & 51.0 & 70.4 & 2.4 \\
  & 1024 &   8 & 19.4 & 24.8 & 59.7 & 31.9 & 52.2 & 54.9 & 52.1 & 69.8 & 2.2 \\
  & 2048 &   4 & 19.4 & 24.7 & 60.0 & 32.2 & 52.3 & 53.5 & 52.4 & 69.7 & 1.9 \\
  & 4096 &   2 & 19.3 & 25.3 & 60.2 & 32.3 & 52.1 & 54.3 & 51.5 & 70.0 & 2.4 \\
  & 9216$^\dagger$ & 1 & 19.0 & 25.1 & 60.4 & 31.5 & 52.3 & 54.8 & 50.9 & 70.5 & 1.6 \\
\bottomrule
\end{tabular}
\end{table}

\paragraph{Locking strength vs.\ forward latency.}
Figure~\ref{fig:finetune_speed_wikitext_rank} shows the resulting WikiText-103 fine-tuning curves for every rung of the ladder.
Locking strength is essentially depth-invariant across the ladder.
In the left panel, all eight locked curves collapse into a narrow band that plateaus clearly above the pretrained baseline, and the single-layer degenerate control at $L{=}1$, parameter-matched to the SwiGLU MLP it replaces, reaches the same plateau as the deepest configuration at $L{=}142$: depth is not what carries the fine-tuning resistance.
In the right panel, the ladder produces the expected spread: deeper networks take longer per step, so higher $L$ moves the curve to the right, but every locked curve still plateaus above the baseline.
Because the architectural mismatch that drives this resistance (Section~\ref{sec:exp_resistance}) does not depend on depth---even the single-layer control already differs structurally from the original SwiGLU MLP---the defender can choose $L$ freely to fit an inference-latency budget from Figure~\ref{fig:mlp_overhead} without a resistance cost.
Given that budget, the defender can still choose a deeper $L$ to strengthen the other two mechanisms of Section~\ref{sec:exp_resistance}, activation memory growth and the disproportionate backward-pass overhead, both of which scale with depth.

\begin{figure}[t]
\centering
\includegraphics[width=0.9\linewidth]{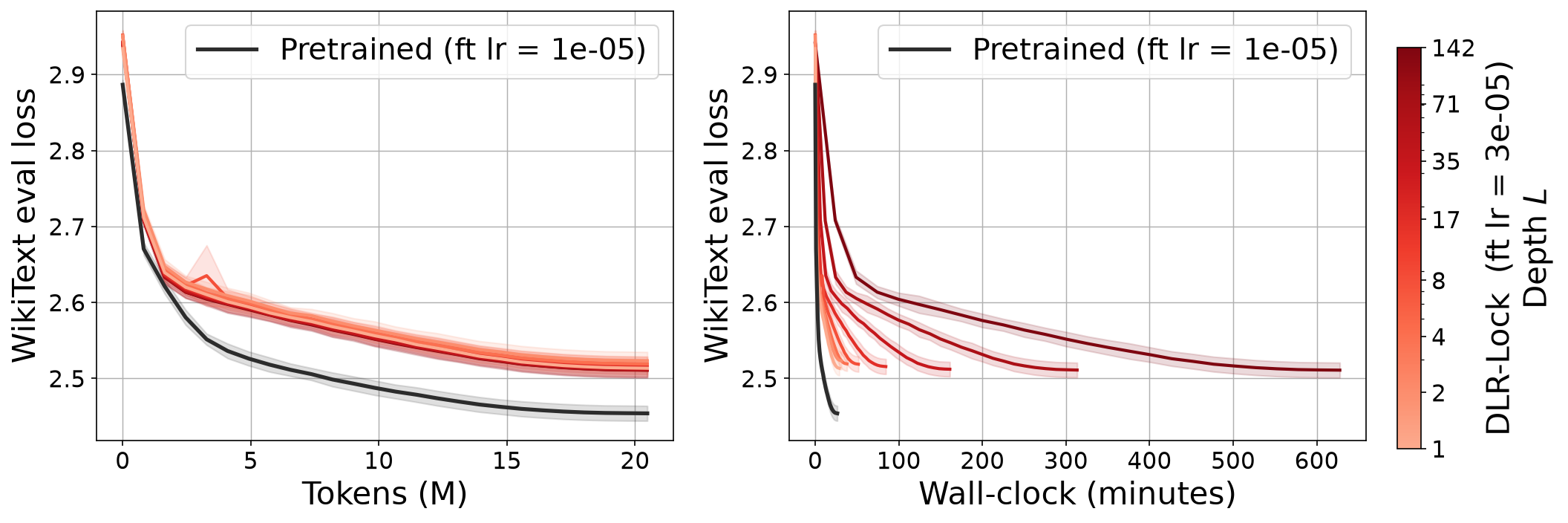}
\caption{Locking strength is depth-invariant across the rank--depth ladder of Table~\ref{tab:scratch-rank-17B}, on the from-scratch 1.7B model, fine-tuned on WikiText-103. Mean $\pm 1$ standard deviation across 3 seeds.
\textbf{Left:} loss vs.\ training tokens. \textbf{Right:} loss vs.\ wall-clock time. Each arm is shown at its own attacker-favorable learning rate: pretrained at $10^{-5}$, locked at $3{\times}10^{-5}$ for every depth.}
\label{fig:finetune_speed_wikitext_rank}
\end{figure}

\section{Attack strategies}
\label{app:attacks_empirical}

This section reports the empirical counterparts of the strategies analyzed in Section~\ref{sec:attacks}, each one a way the attacker might reduce the cost the defense imposes.
Appendix~\ref{app:attack_constlr} closes with one knob that section does not analyze, the fine-tuning schedule.

\subsection{Low-rank adaptation (LoRA) on attention}
\label{app:attack_lora}

Section~\ref{sec:attacks} argued that LoRA does not circumvent the defense: it does not shorten the backward pass, and adapters on the \arch{}'s already thin factors save so little that covering the stack approaches the cost of full fine-tuning. We test it here on the main-text setting, Qwen3-0.6B on WikiText-103.
The attacker freezes the base weights and the \arch{}s and applies LoRA of rank $\rho{=}16$ and scaling $\alpha{=}32$ to the attention projections. The protocol matches Figure~\ref{fig:finetune_speed_wikitext}: 5000 steps, batch $4\times2048$, cosine schedule, three seeds. We sweep the fine-tuning rate $\eta$ over $\{3{\times}10^{-5}, 10^{-4}, 3{\times}10^{-4}, 10^{-3}, 3{\times}10^{-3}\}$, extending above the full-fine-tuning range since LoRA tolerates larger rates.

\paragraph{With gradient checkpointing.}
At matched fine-tuning rates, on the left of Figure~\ref{fig:lora}, the LoRA attack converges to a higher eval loss than the unlocked baseline and never reaches it, since it adapts only the attention projections and only with a low-rank update. It is faster than full fine-tuning of the locked model, which also updates the deep residual stacks, yet it stays far slower than fine-tuning the unlocked model: reaching the attention adapters still forces the backward pass to traverse every depth-$L$ residual stack, and gradient checkpointing recomputes each stack's forward along the way. Sweeping the fine-tuning rate upward, on the right, does not close the gap either: the best reachable loss improves and then turns up before reaching the baseline, bottoming out at $2.638$ at $\eta{=}3{\times}10^{-4}$, against $2.499$ for full fine-tuning of the locked model and $2.337$ for the baseline.

\begin{figure}[t]
\centering
\begin{minipage}[c]{0.62\linewidth}
  \centering
  \includegraphics[width=\linewidth]{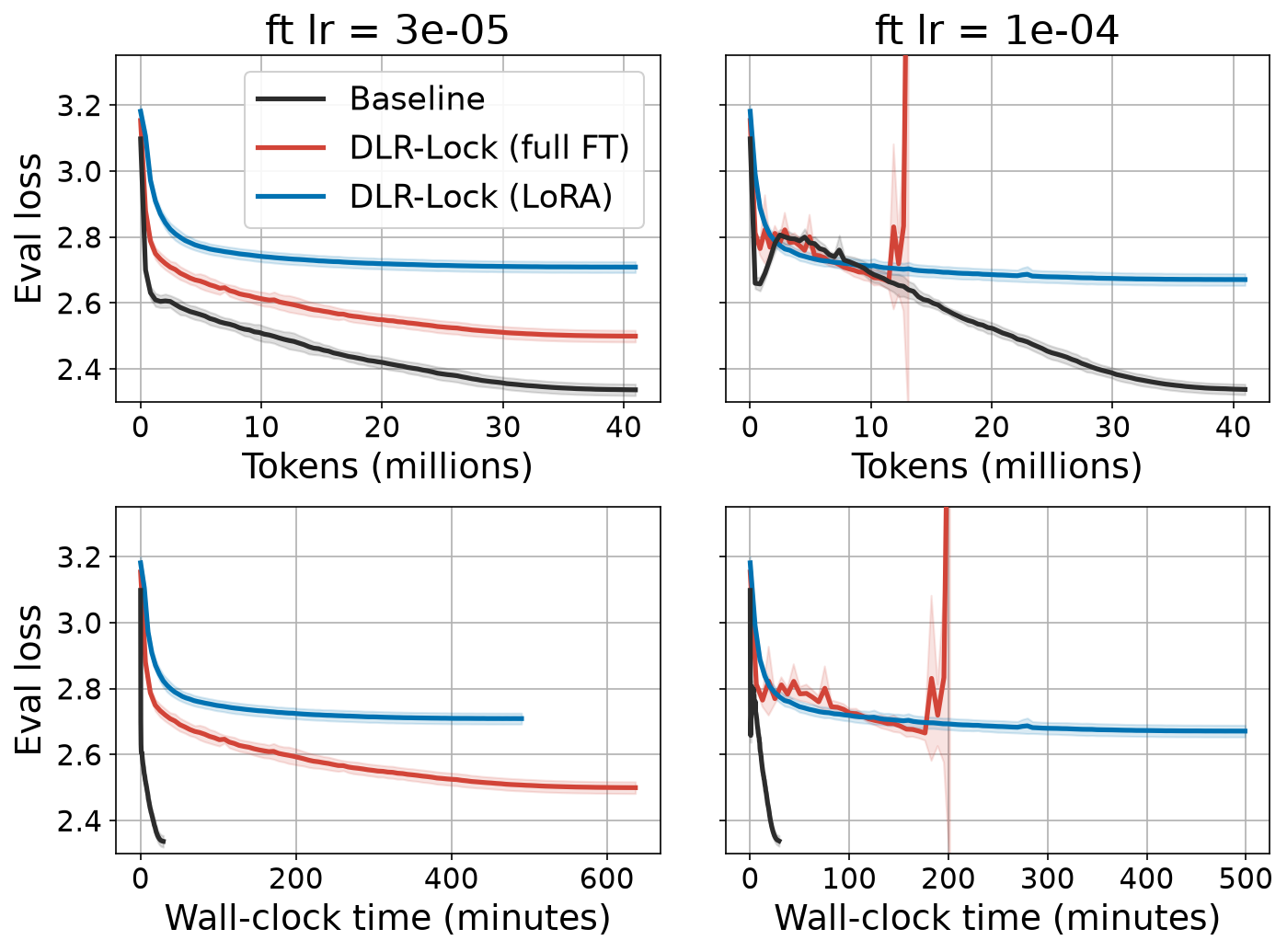}
\end{minipage}\hfill
\begin{minipage}[c]{0.36\linewidth}
  \centering
  \includegraphics[width=\linewidth]{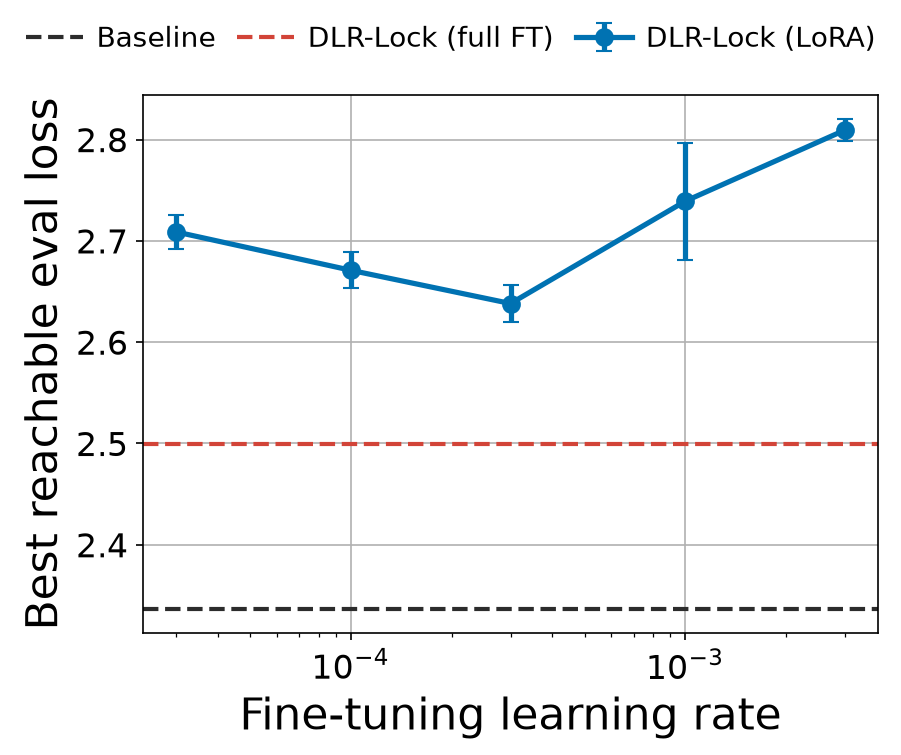}
\end{minipage}
\caption{LoRA attack on the DLR-Locked Qwen3-0.6B, fine-tuned on WikiText-103 with the streaming cosine schedule of Figure~\ref{fig:finetune_speed_wikitext}. Adapters are placed on the attention projections $q/k/v/o$ with the base weights and the \arch{}s frozen, and bands span $\pm 1$ standard deviation over three seeds.
\textbf{Left:} eval loss vs.\ tokens in the top row and wall-clock time in the bottom row, at matched fine-tuning rates $\eta \in \{3{\times}10^{-5}, 10^{-4}\}$, for the baseline, DLR-Lock under full fine-tuning, and DLR-Lock under the LoRA attack. The vertical axis is shared and clipped, so DLR-Lock under full fine-tuning leaves the frame when it diverges at $\eta{=}10^{-4}$.
\textbf{Right:} best reachable eval loss vs.\ fine-tuning rate across the full sweep $\{3{\times}10^{-5}, \ldots, 3{\times}10^{-3}\}$; dashed lines mark the best baseline and full-fine-tuning losses.}
\label{fig:lora}
\end{figure}

\paragraph{Without gradient checkpointing.}
Even with LoRA's small optimizer footprint, fine-tuning without checkpointing exhausts an 80\,GB H100 at the main-text batch of $4\times2048$: every run fails during the forward pass, independent of the fine-tuning rate. LoRA sheds optimizer state and frozen-weight gradients but not activations: the frozen \arch{} stacks alone still retain $\Theta(T\,N\,L\,d)$ of them, on the order of $70$\,GB at this batch (Proposition~\ref{prop:activation}(b)), and the model weights and attention activations then push the total past the $80$\,GB budget. The attacker is therefore forced either to checkpoint, paying the recompute cost above, or to shrink the effective batch, the same activation-memory amplification of Proposition~\ref{prop:kappa}.

\subsection{Reverse distillation}
\label{app:attack_reverse}

Of the strategies in Section~\ref{sec:attacks}, reverse distillation matters most, since its
cost is amortized across adaptations rather than recurring per task (Remark~\ref{rem:amortization}).
Three things stand in the attacker's way, one a matter of cost and two structural.
First, running our pipeline in reverse costs on the order of the locking run itself, some $50$
times a single fine-tuning budget.
Second, every weight of the locked model is public, but the corpus it was pretrained and locked
on is not.
An attacker without it must distill from substitute corpora, which recovers less efficiently at
every budget we ran.
Third, because \method{}'s Phase~2 lets the locked modules co-adapt, they no longer sit near the
original MLPs they replaced, so the attacker's regression phase has neither a certified target
nor a stopping rule.
A final paragraph steps outside full replacement to a further attack strategy: recovering only
part of the stack.

\begin{figure}[!t]
\centering
\includegraphics[width=0.6\linewidth]{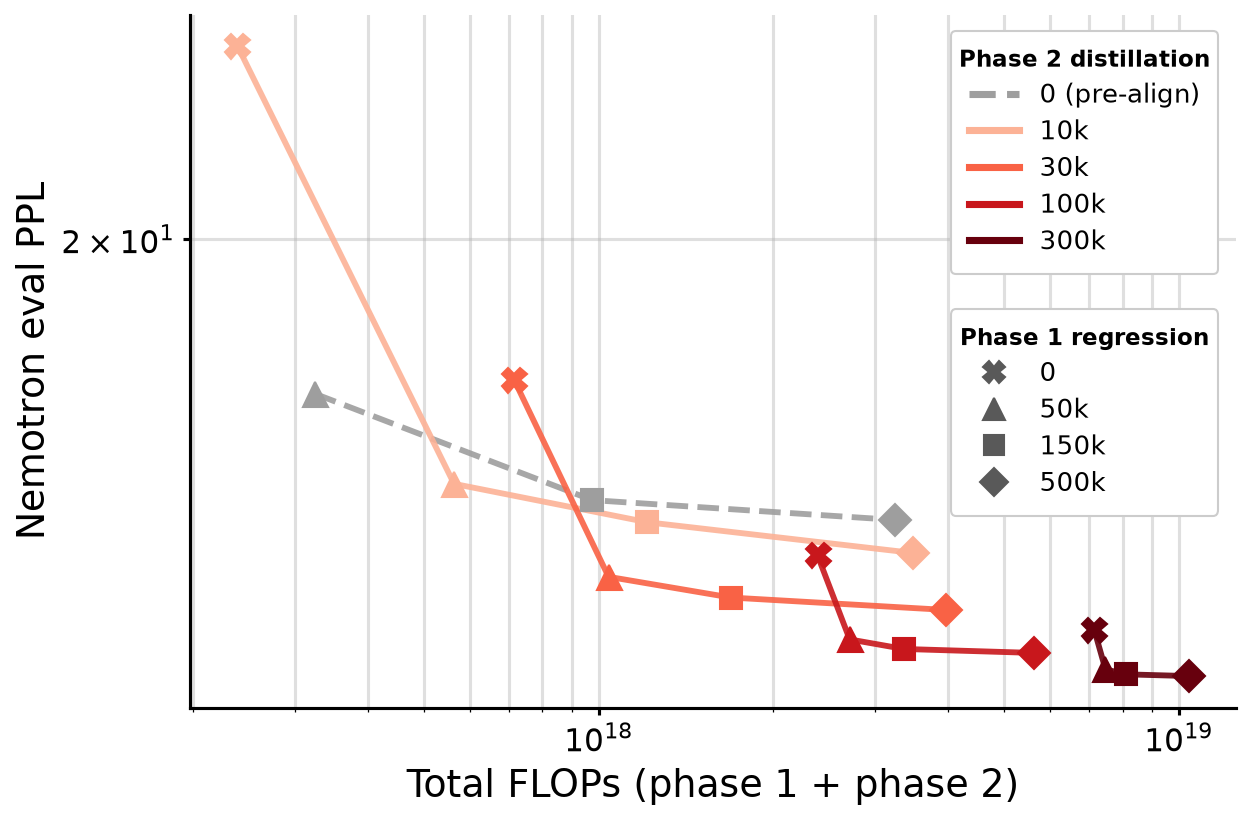}
\caption{Reverse distillation of the locked Qwen3-0.6B, over a grid of module-wise regression and logit distillation steps.
Color encodes the number of distillation steps and marker shape the number of regression steps, and each line holds the distillation length fixed while the regression length increases, matching Figure~\ref{fig:compute_split_pareto}.
The dashed line is the warm start reached by regression alone, before any distillation.
Cost excludes the teacher forward on both sides, as in Figure~\ref{fig:compute_split_pareto}; the attacker's cost per distillation step is a factor $1.26$ below the defender's, since a shallow student needs no gradient-checkpoint recompute.}
\label{fig:attack_pareto}
\end{figure}

\paragraph{Setup.}
The attacker runs our pipeline in reverse, as Section~\ref{sec:attacks} describes, copying every non-MLP weight from the locked checkpoint verbatim and relearning only the MLP sub-block from random initialization.
The two phases mirror ours: module-wise regression onto each locked module's residual map, then top-$k$ logit distillation against the locked model's outputs.
We sweep the number of regression steps ($0$, $50\text{k}$, $150\text{k}$, $500\text{k}$) against the number of distillation steps ($0$, $10\text{k}$, $30\text{k}$, $100\text{k}$, $300\text{k}$) on the locked model of Table~\ref{tab:eval-benchmarks}, the only checkpoint and recovery pipeline we study here, where a step count of $0$ skips that phase, each cell a dedicated run with its own cosine schedule at the same per-step token counts as the defender's phases (Appendix~\ref{app:ablation_compute_split}), and evaluate every cell on one fixed held-out sample of $819$k tokens.
None of \method{}'s locking mechanisms discussed in the main text applies to this route.
The locked model enters only as a frozen forward-only teacher.
What the attacker pays is the volume of distillation that recovery needs.

\paragraph{Recovery cost.}
Figure~\ref{fig:attack_pareto} maps the resulting frontier on the same accounting as the defender's Figure~\ref{fig:compute_split_pareto}, so the two horizontal axes mean the same thing.
The frontier is the set of allocations that no other beats on both cost and perplexity.
The cheapest allocation in our grid to reach its lowest perplexity band, $50\text{k}$ regression steps with $300\text{k}$ distillation steps, costs $7.5\times10^{18}$ floating point operations, against $6.3\times10^{18}$ for the locking run it attacks, $500$k Phase-1 steps followed by $100$k Phase-2 steps.
Reaching that quality therefore costs slightly more than locking the model did, and roughly $50$ times the $5000$-step baseline fine-tuning run of Section~\ref{sec:exp_resistance}.
A single-task attacker has no reason to take this route, and it pays only for one who reuses the recovered model across many adaptations or passes it on, the regime of Remark~\ref{rem:amortization}.
Perplexity is still falling at the largest budget we ran, so no allocation shows recovery having converged, and the attacker cannot read the frontier as the price of a finished job.
The allocation above is about a fifth of a Chinchilla compute-optimal pretraining budget at this model size~\citep{hoffmann2022training}.
Being both the cheapest in our grid and short of convergence, it is a lower bound: recovery costs at least that much, and so is not negligible beside pretraining a comparable model outright.
Within the frontier, regression saturates.
Treating cells within $0.1$ perplexity as tied, since evaluations move by that much between neighboring checkpoints, no cell with $500\text{k}$ regression steps is cost-efficient: each is matched by a cheaper allocation that spends less on regression and more on distillation.
\begin{figure}[!t]
\centering
\includegraphics[width=0.72\linewidth]{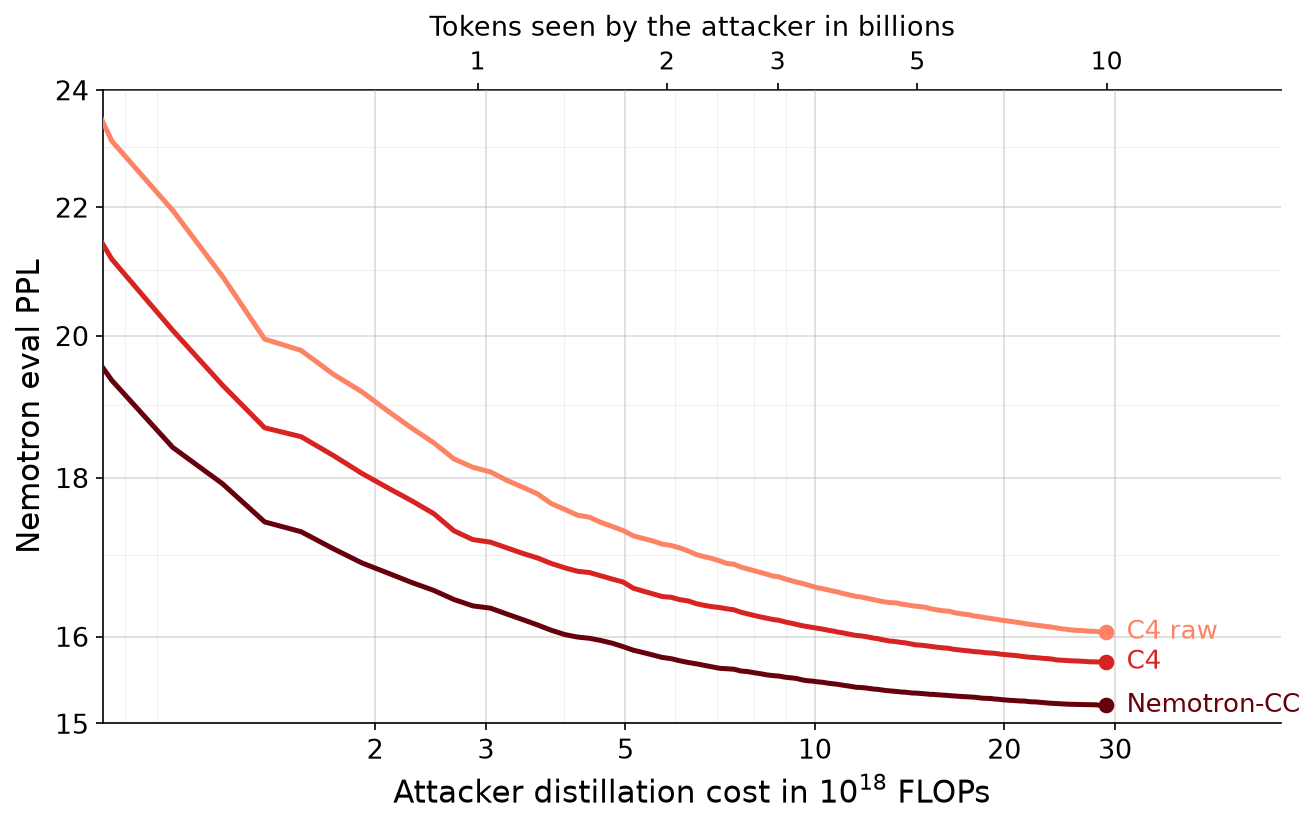}
\caption{Reverse distillation from three corpora against the same locked Qwen3-0.6B, under pure logit distillation with no regression warm start.
Nemotron-CC is the corpus the model was locked on, and C4 in cleaned and raw form stands for commodity web text.
Cost is on the axis of Figure~\ref{fig:attack_pareto} and perplexity on its evaluation sample, so the two figures can be read together.
All three arms stop at $10$B tokens and run at a constant learning rate of $10^{-3}$.}
\label{fig:attack_corpus}
\end{figure}

\paragraph{Corpus dependence.}
In practice the attacker does not hold the corpus the model was locked on.
Figure~\ref{fig:attack_corpus} measures recovery without it.
Three arms distill the same locked model under pure logit distillation, from Nemotron-CC and from C4~\citep{raffel2020exploring} in its cleaned and raw forms, on the cost axis of Figure~\ref{fig:attack_pareto}.
At $2.9\times10^{19}$ floating point operations, four times the cheapest allocation on that frontier to reach its lowest perplexity band, the matched-corpus arm reaches $15.20$ against $15.70$ and $16.05$ for cleaned and raw C4, and all three are still descending.
The gap to the matched corpus is the reading that matters, and it holds at every budget shown.
The two C4 arms also order by curation.
Efficient recovery thus benefits from a corpus close to the one the model was locked on, which the provider holds and the attacker in general does not.
The grid above gave the attacker that corpus throughout, so its cost results do not rest on this asymmetry.

\paragraph{The missing guarantee.}
Beyond cost, the two directions differ in what the regression phase can certify, since depth generally buys expressivity that width cannot cheaply match~\citep{telgarsky2016benefits,eldan2016power}.
\method{}'s Phase~1 regresses a deep \arch{} onto the shallow MLP it replaces, a direction in which the student is if anything the more expressive of the two, so a close fit is likely to exist and the loss measures how near the student has come.
The attacker regresses the other way, where existence looks assured all the same: each \arch{} was fitted to an original MLP, which therefore reproduces it.
That holds of the module Phase~1 produces, but the attacker's target is what Phase~2 leaves behind.
Phase~2 optimizes the composition and lets the modules co-adapt to cancel one another's error, so each locked module drifts from the MLP it replaced while the full function is preserved.
A drifted module may therefore lie outside what a shallow MLP can express; existence would be assured only if the MLP class fully contained the \arch{}'s, which we have no reason to expect and do not establish.
A low regression loss therefore does not certify that the attacker is approaching the original model, and nothing in it says when to stop.
This is an argument about what the attacker can rely on rather than a penalty we measure, and it would bite hardest near a converged solution, which these budgets do not reach.

\paragraph{Partial replacement.}
Everything above concerns full replacement, in which every locked module is relearned.
The attack also generalizes: an attacker could recover only some of the layers and leave the rest locked.
We do not measure this case.
Partial replacement interpolates between two cases we do measure: fine-tuning the locked model directly, with nothing recovered (Section~\ref{sec:exp_resistance}), and full replacement, with everything recovered.
Its cost and the lock it leaves both fall between them, since every layer left in place still forces the backward pass through its own depth-$L$ residual stack, which freezing parameters does not remove, as Section~\ref{sec:attacks} argues for the frozen-\arch{} attack.
Partial replacement therefore trades the cost of recovering a layer against the cost of adapting around the layers still locked, rather than escaping either.

\subsection{Constant learning-rate fine-tuning}
\label{app:attack_constlr}

The main text's fine-tuning sweep, Figure~\ref{fig:finetune_speed_wikitext}, uses a cosine schedule with a short warmup, so the persistent loss gap it reports could in principle be an artifact of annealing: the locked model, whose backward pass we claim to be optimization-unfriendly, might simply be denied the late-stage step sizes it would need to close the gap.
We therefore re-run that sweep with the decay removed and every other knob held identical: the same checkpoints, dataset stream, batch size, number of steps and seeds.
The protocol is that of Figure~\ref{fig:finetune_speed_wikitext}, $5000$ steps at batch $4\times2048$ over three seeds, with the learning rate held at its peak after a $5\%$ warmup instead of decaying, and we sweep $\eta \in \{10^{-5}, 3{\times}10^{-5}, 10^{-4}\}$.

\paragraph{Results.}
Removing the decay does not close the gap.
At the two rates where the locked model trains stably, its best eval loss exceeds the baseline's by $0.16$ at $\eta{=}10^{-5}$ and $0.14$ at $3{\times}10^{-5}$, against $0.17$ and $0.16$ under cosine, so the gap is the same size with or without annealing.
The baseline is not being deflated by the schedule either: at every rate its final loss is also its best and is still falling over the last $500$ steps, so it has not stalled at its peak rate.
What the constant schedule does change is where the locked model becomes unstable, and it moves it down: under cosine the locked model diverges only at $\eta{=}10^{-4}$, whereas here it bottoms out at $2.51$ by step $4100$ at $3{\times}10^{-5}$ and then climbs to $3.91$, while the baseline at that rate keeps descending.
A constant rate therefore narrows the attacker's usable range rather than widening it.
At $\eta{=}10^{-4}$ the two arms end only $0.07$ apart, but neither is usable there: the locked model's loss swings above $6$ through most of the run, and the baseline is itself degraded to $2.51$ from the $2.37$ it reaches at lower rates.
The wall-clock claim of Section~\ref{sec:exp_resistance} thus comes with a quality claim: the best loss the attacker reaches anywhere in this sweep is $2.51$ against the baseline's $2.37$, so the locked model's reachable loss is genuinely worse, not only slower to attain.

\begin{figure}[t]
\centering
\includegraphics[width=0.86\linewidth]{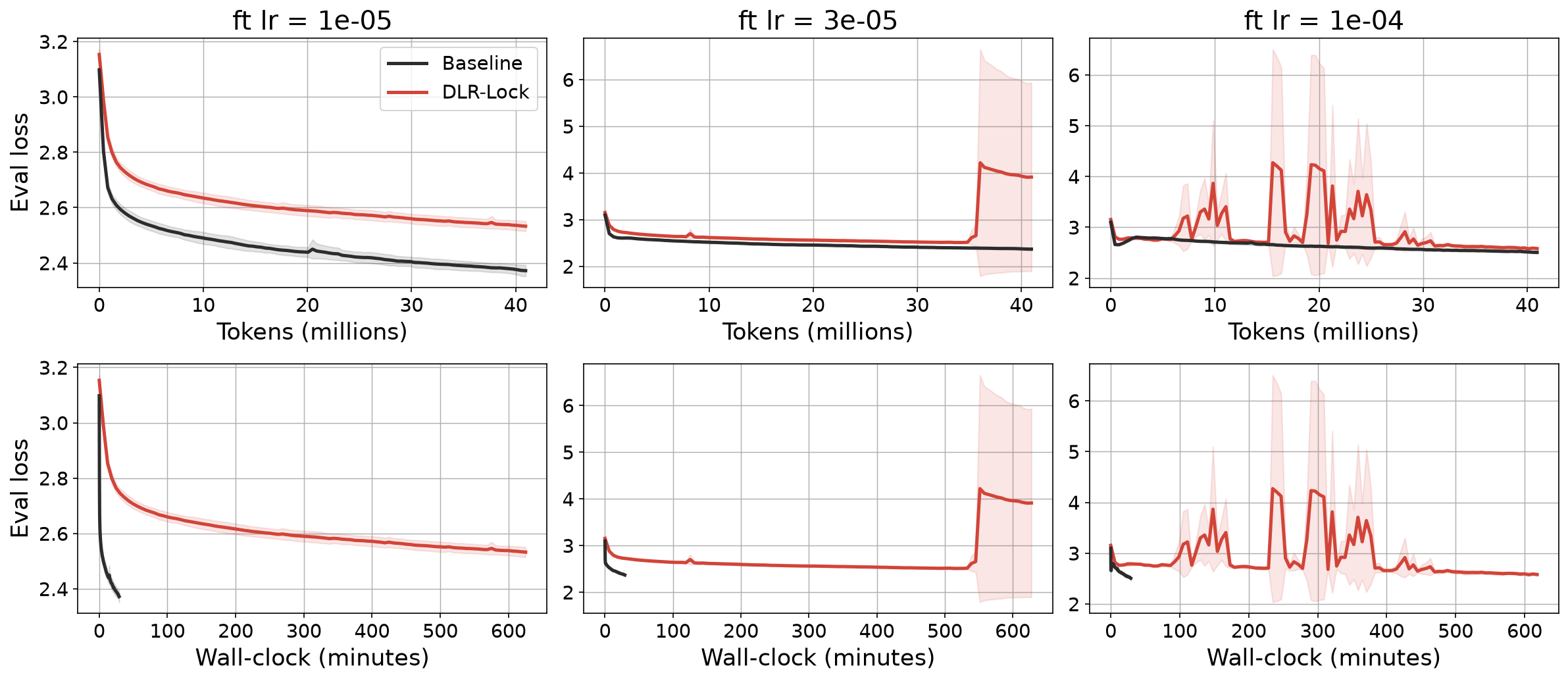}
\caption{Constant learning-rate version of Figure~\ref{fig:finetune_speed_wikitext}: only the schedule changes, from cosine to constant after warmup.
Panels, axes and shading are as in that figure.}
\label{fig:ft_constlr}
\end{figure}

\section{Implementation and timing measurements}
\label{app:measurements}

\subsection{Inference latency}
\label{app:latency}

A natural concern is that replacing each wide SwiGLU MLP with a chain of many
thin low-rank residual layers is prohibitively slow at inference.
The replacement is FLOP-matched to SwiGLU by construction, so it
adds no arithmetic; the cost is instead wall-clock, arising from the launch
overhead of many small kernels and from matrix multiplications too thin to fully
utilize the hardware.
This overhead may look inevitable, but we show that standard compilation and a
few simple custom-kernel strategies largely close the gap.
A more thorough kernel implementation specialized to each configuration is beyond
the scope of this paper, and is itself an interesting and important direction for
future work.
This appendix gives the measurement protocol and our kernel implementation, then
reports the full overhead sweep across model scales: the residual overhead
shrinks with model size, and at the full-model level the locked model is
near-transparent at long sequences.

\subsubsection{Measurement setup and kernel implementation}
\label{app:latency_impl}

All latency measurements are on a single NVIDIA H100 GPU in \texttt{bfloat16};
full-model measurements use FlashAttention~\citep{dao2024flashattention}.
To remove run-to-run variance from dynamic clock scaling, we lock the GPU core
and memory clocks with \texttt{nvidia-smi} for every measurement, baseline and
\arch{} alike.
Each configuration is timed as the minimum over 100 iterations after 100 warmup
iterations, repeated over 10 passes; we report the mean across passes, with
error bars showing the across-pass standard deviation.

\paragraph{Compilation.}
Both the SwiGLU baseline and the \arch{} are compiled with
\texttt{torch.compile}~\citep{10.1145/3620665.3640366} in its
\texttt{reduce-overhead} mode, which traces each model into a graph, fuses
elementwise operations, and replays the forward pass from a captured CUDA graph.
For the \arch{} this is the first of two variants we report; the second
additionally replaces each MLP with a custom kernel (below).
Compilation helps the \arch{} disproportionately: its $L$-layer chain dispatches
many tiny kernels whose dominant cost is the per-launch CPU overhead of issuing
each one, and CUDA graph replay issues the entire captured sequence with a single
host-side launch.
It does not remove the on-device cost of executing those many small,
memory-bound kernels, which is what the custom kernels below target.

\paragraph{Custom kernel implementation.}
The \arch{} replaces one wide SwiGLU MLP with a chain of $L$ thin low-rank
residual layers, each applying two skinny matrix multiplications through a
rank-$r$ bottleneck.
The remaining cost is memory traffic: the compiled graph still writes and
re-reads the hidden state from high-bandwidth memory (HBM) between the many thin
matrix multiplications.
Our kernels remove this by fusing the $L$-layer chain so that the hidden state
stays resident in on-chip memory and each MLP issues only a few launches.
Rather than hand-tune per configuration, we implement a small set of standard
strategies and, for each model scale and batch size, select whichever is fastest
by a short autotuning sweep:
(i)~a single fused kernel (Triton~\citep{tillet2019triton} or CUDA) that marches
through all $L$ layers on chip, with an optional path using warpgroup matrix
instructions (WGMMA) and the Tensor Memory Accelerator (TMA) on Hopper
GPUs~\citep{nvidia2022hopper}; and
(ii)~an orchestrated variant that issues the matrix multiplications through
NVIDIA's cuBLAS library with a fused normalization and activation epilogue and
captures the whole layer chain in a CUDA graph, running a few large library
matrix multiplications in place of the many small compiler-generated kernels.
These are standard fusion and graph-capture techniques rather than new kernel
algorithms, and we make no claim of optimality; our aim is only to show that the
locking architecture's latency overhead is largely an implementation artifact
removable with off-the-shelf tooling, not a fundamental cost.
Our implementations are available at \url{\coderepo}.

\paragraph{Full-model composition.}
To measure the full model, each locked MLP is exposed as a
\texttt{torch.library.custom\_op}~\citep{10.1145/3620665.3640366} and runs
whichever single-MLP strategy above is fastest for that configuration.
Wrapping each MLP as an opaque op lets the outer \texttt{reduce-overhead} graph
capture the entire forward pass, attention and all MLPs, as one CUDA graph with
no per-MLP graph or host-visible copies at the seams.
On Qwen3-8B the compiled model still executes tens of thousands of small kernels
per forward pass ($L$ layers $\times$ $36$ blocks); the custom kernels collapse
each MLP's chain to far fewer, larger kernels, which the single outer graph then
replays.
At large token counts the fastest strategy is a \emph{chunked} orchestrator that
processes the tokens in fixed-size blocks, each kept resident in on-chip cache
across the layer chain.

\subsubsection{Overhead measurements}
\label{app:latency_results}

\paragraph{Single-MLP overhead.}
We first isolate a single MLP, benchmarking one SwiGLU MLP against its \arch{}
replacement at the deepest rank and depth for each scale as we vary the number of
input tokens.
Figure~\ref{fig:mlp_ratio_appendix} reports the \arch{} latency relative to
SwiGLU across all four model scales; it is the full-scale version of the
main-text Figure~\ref{fig:mlp_overhead} (left), which shows only 0.6B and 8B.
Each panel compares the off-the-shelf compiled \arch{} against our custom
kernel.
Compilation alone already removes most of the per-launch overhead
(\S\ref{app:latency_impl}), and the custom kernel closes much of the residual
gap.
For 8B and 32B the ratio further decreases at large batch sizes as GPU
utilization improves; for 0.6B the small rank ($r{=}32$) keeps arithmetic
intensity below the roofline, so the overhead remains memory-bound.

\begin{figure}[t]
  \centering
  \includegraphics[width=0.7\linewidth]{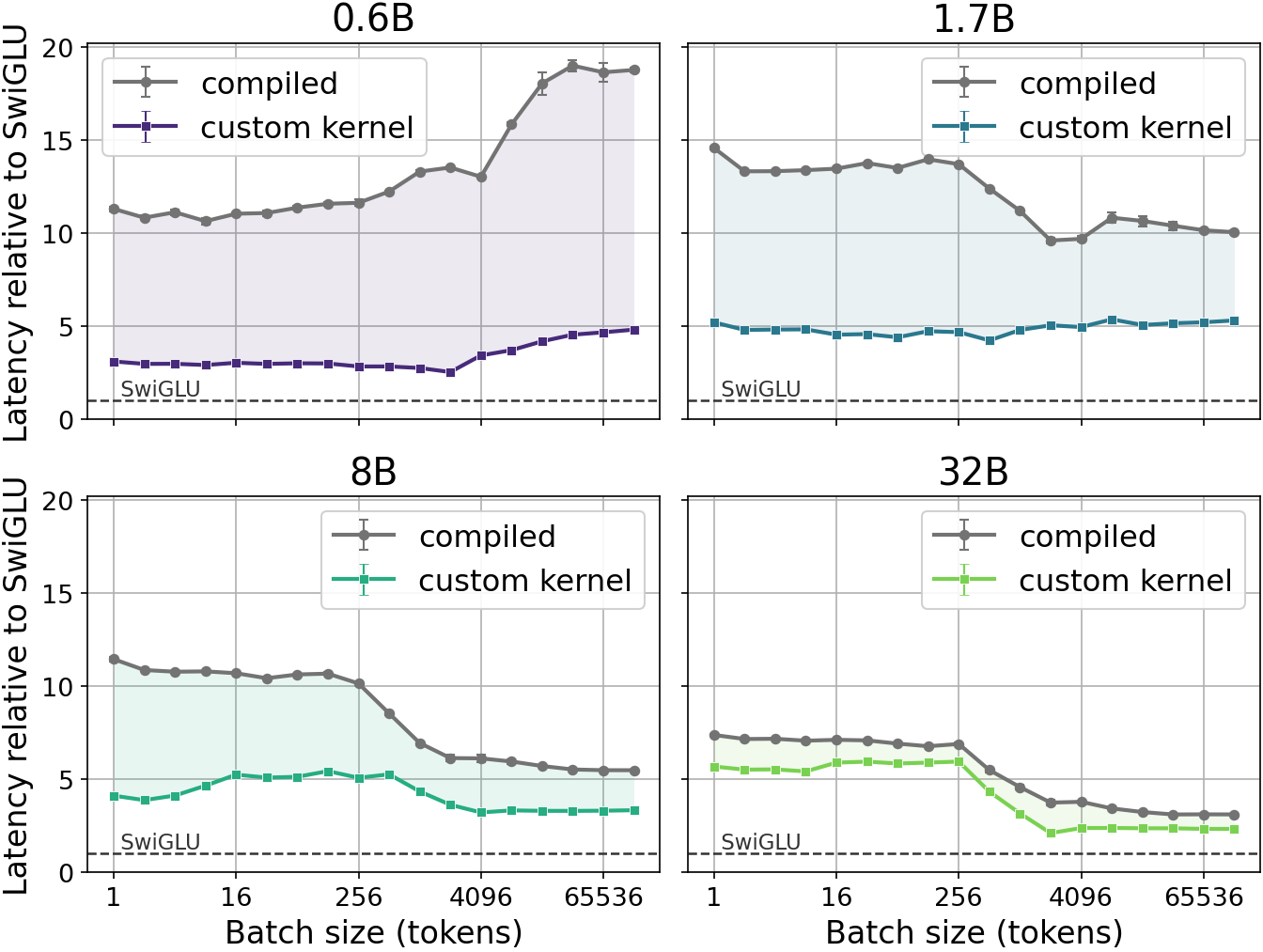}
  \caption{%
    Single-MLP latency of the \arch{} relative to the SwiGLU MLP as a
    function of batch size, across all four model scales.
    The SwiGLU baseline is compiled with \texttt{torch.compile} in
    \texttt{reduce-overhead} mode; relative to it we plot the \arch{} in two
    forms, the same off-the-shelf compilation (gray) and our custom kernel
    (colored).
    The shaded band is the residual latency gap the kernel closes, and the
    dashed line at $1$ is the baseline.
    The $y$-axis is shared across panels, and both the overhead
    and the residual gap shrink with model scale.
    Each panel uses the deepest configuration for that scale
    (0.6B: $r{=}32$, $L{=}141$;
     1.7B: $r{=}64$, $L{=}142$;
     8B: $r{=}128$, $L{=}143$;
     32B: $r{=}256$, $L{=}149$).}
  \label{fig:mlp_ratio_appendix}
\end{figure}

\paragraph{Full-model overhead.}
The results above measure a single MLP in isolation.
In a full transformer, attention, embeddings, and layer normalization are
identical between the baseline and \arch{}, and attention increasingly
dominates at longer sequences, diluting the per-MLP overhead.
Tables~\ref{tab:full_prefill_06b}--\ref{tab:full_prefill_8b} report the
full-model prefill overhead ratio for Qwen3-0.6B and 8B with the custom kernel.
The ratio falls steadily with sequence length: for 8B (batch $1$) it drops from
$3.46\times$ at $\text{seq}{=}1024$ to $1.97\times$ at $\text{seq}{=}4096$ and
$1.21\times$ at $\text{seq}{=}32768$, and 0.6B is near-transparent at long
sequences ($1.04\times$ at $\text{seq}{=}65536$).
The right panel of Figure~\ref{fig:mlp_overhead} shows the same dilution at a
fixed $32\text{k}$-token budget across the rank/depth ladder.
The overhead is largest only in the short-sequence, small-batch corner, where
attention is negligible so there is little to dilute and the latency is bound by
the \arch{}'s sequential-depth chain; this is the single-request regime that is
intrinsically hard for any deep model, and it relaxes as the token count grows.

\begin{table}[t]
  \caption{%
    Full-model prefill overhead ratio (\arch{}\,/\,baseline) for Qwen3-0.6B
    with the custom kernel on a single H100.
    Rows: batch size. Columns: sequence length.
    ``--'' denotes configurations that exceeded GPU memory.}
  \label{tab:full_prefill_06b}
  \centering
  \small
  \setlength{\tabcolsep}{4pt}
  \begin{tabular}{@{}r*{10}{w{c}{1.0cm}}@{}}
    \toprule
    & \multicolumn{10}{c}{Sequence length} \\
    \cmidrule(l){2-11}
    $B$ & 128 & 256 & 512 & 1024 & 2048 & 4096 & 8192 & 16384 & 32768 & 65536 \\
    \midrule
    1     & $5.16$ & $4.57$ & $3.92$ & $2.69$ & $1.73$ & $1.49$ & $1.24$ & $1.13$ & $1.07$ & $1.04$ \\
    2     & $4.72$ & $4.12$ & $2.97$ & $1.98$ & $1.75$ & $1.41$ & $1.24$ & $1.13$ & $1.07$ & $1.04$ \\
    4     & $4.24$ & $3.13$ & $2.19$ & $2.02$ & $1.63$ & $1.41$ & $1.23$ & $1.13$ & $1.07$ & -- \\
    8     & $3.26$ & $2.33$ & $2.25$ & $1.87$ & $1.63$ & $1.40$ & $1.23$ & $1.13$ & --     & -- \\
    16    & $2.41$ & $2.42$ & $2.06$ & $1.87$ & $1.63$ & $1.40$ & $1.24$ & --     & --     & -- \\
    32    & $2.50$ & $2.19$ & $2.05$ & $1.89$ & $1.63$ & $1.41$ & --     & --     & --     & -- \\
    64    & $2.28$ & $2.18$ & $2.08$ & $1.87$ & $1.64$ & --     & --     & --     & --     & -- \\
    128   & $2.26$ & $2.19$ & $2.05$ & $1.88$ & --     & --     & --     & --     & --     & -- \\
    256   & $2.28$ & $2.18$ & $2.06$ & --     & --     & --     & --     & --     & --     & -- \\
    512   & $2.25$ & $2.17$ & --     & --     & --     & --     & --     & --     & --     & -- \\
    \bottomrule
  \end{tabular}
\end{table}

\begin{table}[t]
  \caption{%
    Full-model prefill overhead ratio (\arch{}\,/\,baseline) for Qwen3-8B
    with the custom kernel on a single H100.}
  \label{tab:full_prefill_8b}
  \centering
  \small
  \setlength{\tabcolsep}{4pt}
  \begin{tabular}{@{}r*{10}{w{c}{1.0cm}}@{}}
    \toprule
    & \multicolumn{10}{c}{Sequence length} \\
    \cmidrule(l){2-11}
    $B$ & 128 & 256 & 512 & 1024 & 2048 & 4096 & 8192 & 16384 & 32768 & 65536 \\
    \midrule
    1     & $5.91$ & $5.22$ & $4.77$ & $3.46$ & $2.52$ & $1.97$ & $1.68$ & $1.41$ & $1.21$ & -- \\
    2     & $5.29$ & $4.86$ & $3.62$ & $2.73$ & $2.22$ & $1.97$ & $1.66$ & $1.39$ & $1.21$ & -- \\
    4     & $4.86$ & $3.68$ & $2.88$ & $2.40$ & $2.16$ & $1.91$ & $1.61$ & $1.37$ & --     & -- \\
    8     & $3.70$ & $2.90$ & $2.49$ & $2.37$ & $2.17$ & $1.91$ & $1.57$ & --     & --     & -- \\
    16    & $3.07$ & $2.60$ & $2.52$ & $2.32$ & $2.18$ & $1.89$ & --     & --     & --     & -- \\
    32    & $2.65$ & $2.58$ & $2.47$ & $2.34$ & $2.16$ & --     & --     & --     & --     & -- \\
    64    & $2.62$ & $2.54$ & $2.41$ & $2.29$ & --     & --     & --     & --     & --     & -- \\
    128   & $2.59$ & $2.54$ & $2.40$ & --     & --     & --     & --     & --     & --     & -- \\
    256   & $2.56$ & $2.53$ & --     & --     & --     & --     & --     & --     & --     & -- \\
    512   & $2.56$ & --     & --     & --     & --     & --     & --     & --     & --     & -- \\
    \bottomrule
  \end{tabular}
\end{table}

\subsection{Per-step training cost}
\label{app:step_cost}

Section~\ref{sec:exp_resistance} attributes the attacker's wall-clock penalty to the
backward pass rather than to forward overhead.
This section reports the measurement behind that claim.
At the main-text batch of $4{\times}2048$ tokens on one H100, we time an inference
forward with no gradient ($t_f$) and a full optimizer step ($t_e$), taking the training
overhead as $t_b = t_e - t_f$; primed quantities repeat both under the deployed setting
of Section~\ref{sec:inference}, $t_f'$ with our custom kernel and $t_e'$ under
\texttt{reduce-overhead} compilation of the ordinary module stack.
Each quantity is the mean of at least ten timed iterations after warmup.

\begin{table}[t]
\caption{Per-step timing and memory of Qwen3-0.6B at the batch $4{\times}2048$
fine-tuning shape, on one H100: $t_f$ the no-gradient forward, $t_e$ one full optimizer
step, and $t_b = t_e - t_f$ the training overhead. Deployed primes repeat all three under
our custom kernel for the forward and \texttt{reduce-overhead} compilation of the
ordinary module stack for the full step, since the kernel has no backward. The locked
model runs with activation checkpointing inside each \arch{}, without which it does not
fit at this batch; the baseline contains no \arch{}s, so no checkpointing applies to it.}
\label{tab:step_cost}
\centering
\small
\begin{tabular}{@{}lrrr@{}}
\toprule
& Baseline & +\method{} & Ratio \\
\midrule
\multicolumn{4}{@{}l}{\emph{Eager}} \\
Forward, no gradient, $t_f$ (ms)  & $60.1$  & $694.4$  & $11.6\times$ \\
Training overhead, $t_b$ (ms)     & $160.7$ & $4635.9$ & $28.8\times$ \\
Full step, $t_e$ (ms)             & $220.8$ & $5330.3$ & $24.1\times$ \\
Overhead share, $t_b/t_e$         & $73\%$      & $87\%$      & --- \\
Overhead over forward, $t_b/t_f$  & $2.7\times$ & $6.7\times$ & --- \\
\midrule
\multicolumn{4}{@{}l}{\emph{Deployed}} \\
Forward, $t_f'$ (ms)              & $52.7$  & $104.5$ & $1.98\times$ \\
Training overhead, $t_b'$ (ms)    & $137.2$ & $893.1$ & $6.51\times$ \\
Full step, $t_e'$ (ms)            & $189.8$ & $997.6$ & $5.26\times$ \\
Overhead share, $t_b'/t_e'$       & $72\%$      & $90\%$      & --- \\
Overhead over forward, $t_b'/t_f'$ & $2.6\times$ & $8.5\times$ & --- \\
\midrule
\multicolumn{4}{@{}l}{\emph{Memory, eager}} \\
Peak, checkpointing on (GB)       & $37.6$ & $31.8$ & $0.84\times$ \\
Peak, checkpointing off (GB)      & $37.6$ & OOM    & --- \\
\bottomrule
\end{tabular}
\end{table}

Table~\ref{tab:step_cost} shows that the step is dominated by what happens after the forward, and more so once the model is
locked: $t_b$ is $73\%$ of the baseline's step and $87\%$ of the locked model's, which is
$2.7\times$ and $6.7\times$ the respective forward.
Compiling narrows the full step from $24.1\times$ to $5.26\times$, but the forward
narrows far more, from $11.6\times$ to $1.98\times$, so for the locked model the
training overhead is a \emph{larger} share of the deployed step than of the eager
one, $90\%$ against $87\%$, and a wider multiple of the forward, $8.5\times$ against
$6.7\times$; the baseline barely moves, $73\%$ to $72\%$ and $2.7\times$ to $2.6\times$.
The custom kernel that Section~\ref{sec:inference} uses to make the forward near-free
therefore does not narrow the asymmetry analyzed there; it widens it, since it speeds up
only the part of the step the backward does not share.

The locked model cannot train at this batch without checkpointing: the $\Theta(Ld)$
activation stack exhausts the $80$\,GB device.
This is the amplification of Proposition~\ref{prop:kappa} at its most direct, and it
leaves the attacker two options: shrinking the batch and accumulating gradients over more
micro-steps, which is the case Proposition~\ref{prop:kappa} bounds, or recomputing.
An attacker who recomputes trades the memory for time and, in this configuration, ends up
peaking below the unlocked baseline, $31.8$\,GB against $37.6$, since the $L$ intermediate
activations are discarded and re-derived rather than stored.
That ordering is a property of this configuration rather than a general one: nothing
guarantees that a single level of checkpointing brings a locked model within device
memory.

The optimization instability from architectural mismatch, the second cost of Section~\ref{sec:exp_resistance}, is separate from what this section measures: it slows convergence itself, a cost that compounds on top of the per-step slowdown just measured, not a consequence of it.

\applefootnote{ \textcolor{textgray}{\sffamily Apple and the Apple logo are trademarks of Apple Inc., registered in the U.S. and other countries and regions.}}

\end{document}

%% file: apple_preamble.tex
\usepackage{amsmath}
\usepackage{enumerate}
\usepackage{amsfonts}
\usepackage{amsthm}
\usepackage{diagbox}
\usepackage{colortbl}
\usepackage{amssymb}
\usepackage{xspace}
\usepackage{wrapfig}
\usepackage{adjustbox}
\usepackage{tabularx}
\usepackage{mathtools}
\usepackage{tikz}
\usepackage{enumitem}
\usepackage{silence}
\usepackage{dsfont}
\usepackage{makecell}
\usepackage{xfakebold}
\input{math_commands}

\definecolor{textgray}{HTML}{6E6E73}
\usetikzlibrary{positioning, calc}
\usetikzlibrary{decorations.pathmorphing}

\makeatletter
\patchcmd{\wrong@fontshape}{\@gobbletwo}{}{}{}
\makeatother
\numberwithin{equation}{section}
\makeatletter
\AtBeginDocument{
  \urlstyle{sf}
  
}
\makeatother

\definecolor{light}{RGB}{125, 125, 125}
\crefname{tcb@cnt@pbox}{code}{code}
\Crefname{tcb@cnt@pbox}{Code}{Code}
\crefname{assumption}{assumption}{assumption}
\Crefname{assumption}{Assumption}{Assumptions}

\newtcolorbox[auto counter]{pbox}[2][]{
  colback=white,
  title=Code~\thetcbcounter: #2,
  #1,fonttitle=\sffamily,
  fontupper=\sffamily,
  arc=2pt,
  colframe=bgcolor,
  coltitle=fgcolor,
  colbacktitle=bgcolor,
  toptitle=0.25cm,
  bottomtitle=0.125cm
}

\makeatletter
\newcommand\applefootnote[1]{%
  \begingroup
  \renewcommand\thefootnote{}%
  \renewcommand\@makefntext[1]{\noindent##1}%
  \footnote{#1}%
  \addtocounter{footnote}{-1}%
  \endgroup
}
\makeatother

\definecolor{cverbbg}{gray}{0.90}

%% file: math_commands.tex
\usepackage{amsmath,amsfonts,bm}

\def\eqref#1{equation~\ref{#1}}

\def\1{\bm{1}}

\def\vh{{\bm{h}}}

\def\vv{{\bm{v}}}

\def\vx{{\bm{x}}}

\def\vz{{\bm{z}}}

\def\mA{{\bm{A}}}

\def\mH{{\bm{H}}}
\def\mI{{\bm{I}}}

\def\mM{{\bm{M}}}

\def\mU{{\bm{U}}}
\def\mV{{\bm{V}}}
\def\mW{{\bm{W}}}
\def\mX{{\bm{X}}}

\def\mZ{{\bm{Z}}}

\DeclareMathAlphabet{\mathsfit}{\encodingdefault}{\sfdefault}{m}{sl}
\SetMathAlphabet{\mathsfit}{bold}{\encodingdefault}{\sfdefault}{bx}{n}

\def\gA{{\mathcal{A}}}

\def\gL{{\mathcal{L}}}
\def\gM{{\mathcal{M}}}

\def\gT{{\mathcal{T}}}

\newcommand{\R}{\mathbb{R}}